%% file: main-paper.tex
\documentclass{article}

 \usepackage[preprint]{neurips_2026}

\title{Augmented Lagrangian Method for Last-Iterate Convergence for Constrained MDPs}

\author{%
    Michael Lu \\
    Simon Fraser University\\
    \texttt{michael\_lu\_3@sfu.ca}
    \And
    Max Qiushi Lin \\
    Simon Fraser University\\
    \texttt{maxqslin@gmail.com} \\
    \AND 
    Mo Chen \\
    Simon Fraser University\\
    \texttt{mochen@cs.sfu.ca}
    \And  
    Sharan Vaswani \\
    Simon Fraser University\\
    \texttt{vaswani.sharan@gmail.com} \\
    }

\PassOptionsToPackage{numbers, compress}{natbib}

\usepackage{titletoc} %
\input{header}

\input{symbols}

\begin{document}
\maketitle

\input{Sections/00_abstract}
\input{Sections/01_introduction}

\input{Sections/02_problem_formulation}
\input{Sections/03_framework_no_mu}
\input{Sections/04_tabular_setting_no_mu}
\input{Sections/05_handling_fa_no_mu}

\input{Sections/07_conclusion}

\bibliographystyle{plainnat}
\bibliography{ref}
\newpage
\appendix
\newcommand{\appendixTitle}{%
\vbox{
    \centering
	\hrule height 4pt
	\vskip 0.2in
	{\LARGE \bf Supplementary material}
	\vskip 0.2in
	\hrule height 1pt 
}}
\appendixTitle

\startcontents %
\printcontents{}{1}{} %

\input{Appendix/A1_definitions}
\input{Appendix/A2_ALM_PG_proofs}

\input{Appendix/A3_ALM_proofs}
\input{Appendix/A4_general_utility_proofs}
\input{Appendix/A3_experiments}

\input{Appendix/A4_environment_details}
\input{Appendix/A5_restarted_theorem_lemma}

\end{document}

%% file: header.tex
\usepackage{lmodern}
\usepackage[english]{babel}
\usepackage{latexsym}
\usepackage{amsmath}
\usepackage{mathrsfs}
\usepackage{amssymb}
\usepackage{mathtools}
\usepackage[inline,shortlabels]{enumitem} 
\usepackage{bm}
\usepackage{datetime}
\usepackage[table,xcdraw]{xcolor}
\usepackage{accents}
\usepackage{tikz}
\usepackage{listings}
\usepackage{mdframed}
\usepackage{pgfplots}
\usepackage{pgfplotstable}

\usepackage[linesnumbered,lined,boxed,commentsnumbered,ruled,longend]{algorithm2e}
\usepackage{algorithmic}

\usepackage{dsfont}
\usepackage{color}
\usepackage{colortbl}
\usepackage{pifont}
\usepackage[bf,font=small,singlelinecheck=off]{caption}
\usepackage{microtype} 
\usepackage{float}
\usepackage{xfrac} 
\usepackage{xspace}
\usepackage{booktabs}
\usepackage{blkarray}
\usepackage{subcaption}

\allowdisplaybreaks

\usepackage{hyperref}
\hypersetup{
    unicode=false,          
    pdftoolbar=true,        
    pdfmenubar=true,        
    pdffitwindow=false,     
    pdfstartview={FitH},    
    pdftitle={XXXXX},    
    pdfauthor={XXX},     
    pdfsubject={Bandits, Reinforcement Learning},   
    pdfcreator={Creator},   
    pdfproducer={Producer}, 
    pdfkeywords={bandits} {reinforcement learning} {policy gradient}, 
    pdfnewwindow=true,      
    colorlinks=true,       
    linkcolor=blue,          
    citecolor=blue,        
    filecolor=magenta,      
    urlcolor=cyan           
}
\usepackage{amsthm}
\usepackage{times}
\usepackage{natbib}
\usepackage{nicefrac}
\usepackage{wrapfig}
\usepackage{pgfplots}
\usepackage[capitalize]{cleveref}
\usepackage[nottoc,numbib]{tocbibind}

\usepackage{geometry}
\usepackage[normalem]{ulem}

\usepackage{thmtools, thm-restate}

\declaretheorem{theorem}
\declaretheorem{proposition}
\declaretheorem{assumption}

\newtheorem{lemma}{Lemma}
\usepackage{comment}

\newcommand{\red}[1]{{\color{red} #1 }}

\newcommand{\green}[1]{\textcolor{green}{#1}}

\newtheorem{problem}{Problem}

\crefname{theorem}{Theorem}{Theorems}
\crefname{proposition}{Proposition}{Propositions}
\crefname{assumption}{Assumption}{Assumptions}
\crefname{corollary}{Corollary}{Corollaries}
\crefname{remark}{Remark}{Remarks}
\crefname{definition}{Definition}{Definitions}
\crefname{lemma}{Lemma}{Lemmas}
\crefname{problem}{Problem}{Problems}

%% file: symbols.tex
\def\1{\bm{1}}

\newcommand{\F}{F_{\beta, \lambda}}
\newcommand{\Ft}{F_{\beta, \lambda_t}}
\newcommand{\cmark}{\green{\ding{51}}}%
\newcommand{\xmark}{\red{\ding{56}}}%

\newcommand{\TRPO}{\texttt{TRPO}}

\newcommand{\PPOLag}{\texttt{PPOLag}}

\newcommand{\CPO}{\texttt{CPO}}

\newcommand{\AL}{{\texttt{AL}}}
\newcommand{\RLGU}{{\texttt{RLGU}}}

\newcommand{\opt}{\mathrm{opt}}
\newcommand{\bias}{\mathrm{bias}}

\newcommand\numberthis{\addtocounter{equation}{1}\tag{\theequation}}
\newcommand{\almsaomslack}{\overline{F}_{\beta,\lambda}}
\newcommand{\almsaomslackt}{\overline{F}_{\beta,\lambda_t}}
\newcommand{\saom}{\mu^{\pi}}
\newcommand{\saomt}{\mu^{\pi_t}}
\newcommand{\saomtt}{\mu^{\pi_{t+1}}}

\newcommand{\alr}{\beta}
\newcommand{\al}{\Ls}

\newcommand{\khalf}{k + \nicefrac{1}{2}}
\newcommand{\gupi}{\Gamma(\pi)}

\newcommand{\gupik}{\Gamma(\pi_k)}
\newcommand{\gupikk}{\Gamma(\pi_{k+1})}

\newcommand{\qgu}{Q^{\pi}_{\gupi}}

\newcommand{\qgupik}{Q^{\pi_k}_{\gupik}}
\newcommand{\qgupikk}{Q^{\pi_{k+1}}_{\gupikk}}
\newcommand{\KL}{\mathrm{KL}}

\newcommand{\objsaom}{\dpd{\saom, r}}

\def\pistar{\pi^*}

\def\pitheta{\pi_{\theta}}
\def\pitt{{\pi_{\theta_{t+1}}}}

\def\thtt{{\theta_{t+1}}}

\DeclarePairedDelimiter\parens{\lparen}{\rparen}
\DeclarePairedDelimiter\abs{\lvert}{\rvert}
\DeclarePairedDelimiter\norm{\lVert}{\rVert}
\DeclarePairedDelimiter\braces{\lbrace}{\rbrace} 
\DeclarePairedDelimiter\bracks{\lbrack}{\rbrack} 
\DeclarePairedDelimiter\angles{\langle}{\rangle}

\def\dpd#1{{\angles*{ #1 }}}
\def\eps{{\epsilon}}

\def\supnorm#1{\norm*{ #1 }_\infty}
\def\normsq#1{{\norm*{ #1 }^2_2}}

\DeclareMathAlphabet{\mathsfit}{\encodingdefault}{\sfdefault}{m}{sl}
\SetMathAlphabet{\mathsfit}{bold}{\encodingdefault}{\sfdefault}{bx}{n}

\def\gO{{\mathcal{O}}}

\def\gS{{\mathcal{S}}}

\newcommand{\E}{\mathbb{E}}

\newcommand{\Ls}{\mathcal{L}}
\newcommand{\R}{\mathbb{R}}

\DeclareMathOperator*{\argmax}{arg\,max}
\DeclareMathOperator*{\argmin}{arg\,min}

\def\cA{{\mathcal{A}}}

\def\cS{{\mathcal{S}}}

\def\cT{{\mathcal{T}}}

\def\cP{{\mathcal{P}}}

\def\cF{{\mathcal{F}}}

\def\cK{{\mathcal{K}}}

\newcommand{\xkk}{x_{t+1}}

\newcommand{\gradf}[1]{\nabla f(#1)}

\newcommand{\cL}{\mathcal{L}}
\newcommand{\cZ}{\mathcal{Z}}
\def\cX{{\mathcal{X}}}

\newcommand{\lk}{\lambda_t}
\newcommand{\lkk}{\lambda_{t+1}}
\newcommand{\lstar}{\lambda^*}
\newcommand{\epsk}{\epsilon_t}

\newcommand{\LCal}{\mathcal{L}}

%% file: Sections/00_abstract.tex
\vspace{-2ex}
\begin{abstract}
We study policy optimization for infinite-horizon, discounted constrained Markov decision processes (CMDPs). While existing theoretical guarantees typically hold for the mixture policy, deploying such a policy is computationally and memory intensive. This leads to a practical mismatch where a single (last-iterate) policy must be deployed. Recent theoretical works have thus focused on proving last-iterate convergence, but  are largely limited to the tabular setting or to algorithmic variants that are rarely used in practice. To address this, we use the classic inexact augmented Lagrangian (\texttt{AL}) method from constrained optimization, and propose a general framework with provable last-iterate convergence for CMDPs. We first focus on the tabular setting and propose to solve the \texttt{AL} sub-problem with projected Q-ascent (\texttt{PQA}). Combining the theoretical guarantees of \texttt{PQA} and the standard \texttt{AL} analysis enables us to establish global last-iterate convergence. We generalize these results to handle log-linear policies, and demonstrate that an efficient, projected variant of \texttt{PQA} can achieve last-iterate convergence with comparable guarantees as prior work. Finally, we demonstrate that our framework scales to complex non-linear policies, and evaluate it on continuous control tasks.

\end{abstract}

%% file: Sections/01_introduction.tex
\vspace{-3ex}
\section{Introduction}
\label{sec:introduction}
\vspace{-1ex}
Reinforcement Learning (RL) has found success in a variety of applications such as video games~\citep{mnih2015human}, robotics~\citep{kober2013reinforcement} and fine-tuning large language models~\citep{uc2023survey}. While standard RL algorithms are typically designed to optimize a single performance objective, real-world systems must also operate under operational and safety constraints. For instance, in a wireless communication system, a network controller aims to maximize the bitrate delivered to users while ensuring that total energy consumption does not exceed a given budget~\citep{julian2002qos,buratti2009overview}. A common approach is to model such problems as a constrained Markov decision process (CMDP)~\citep{altman2021constrained}. In this setting, an agent must maximize a cumulative reward (e.g. bitrate), while ensuring that the cumulative constraint reward (e.g. energy consumption) satisfies a prescribed threshold. Unlike standard MDPs, the optimal policies for CMDPs are stochastic, making it challenging to design algorithms that reliably recover a near-optimal policy.

A typical algorithmic template~\citep{jain2022towards,liu2025sample,ding2020natural,bura2022dope} to solve CMDPs is to form the Lagrangian and iteratively perform alternating updates of the policy (primal variable) and the associated Lagrange multipliers (dual variable). While such primal–dual (\texttt{PD}) methods are guaranteed to output a policy that maximizes the reward while satisfying the constraints, such guarantees only hold for the \textit{mixture policy}. In particular, \texttt{PD} methods iteratively generate a sequence of policies, and maintaining a mixture policy requires storing all intermediate policies, making this approach memory intensive. Moreover, in order to deploy a mixture policy, a policy is randomly sampled from the set of intermediate policies, and the theoretical guarantees hold in expectation over this random selection. From a safety perspective, relying on such a mixture policy can be dangerous, because an individual sampled policy may significantly violate the constraints. Consequently, practitioners use \texttt{PD} methods and deploy the final policy, trading theoretical guarantees for ease of implementation. Unfortunately, this practical variant is typically sensitive to hyper-parameters and the final policy can exhibit oscillatory behavior (for example, see~\cref{fig:alm_vs_pd}).

This concern has motivated recent works~\citep{ding2023last,mondal2024last,muller2024truly,moskovitz2023reload,ying2022dual,montenegro2024last} that aim to design iterative approaches that have theoretical guarantees for the final stochastic policy. Such \textit{last-iterate convergence} is desirable from both a theoretical and empirical perspective. All such recent works rely on adding additional regularization to both the primal and dual variables. Despite their theoretical appeal, these last-iterate algorithms are either difficult to implement in practice, or are designed for simplified settings, limiting their applicability. In fact, to the best of our knowledge, there is no last-iterate algorithm that is both principled (has guarantees in simplified settings) and practical (can be efficiently implemented for large problems). 

Motivated by this gap, and in contrast to the previous work on last-iterate convergence, we employ the \textit{augmented Lagrangian ($\AL$) method}. Such methods are a classic approach in the constrained optimization literature, and are widely used in practice (see~\citet{deng2025augmented} for a recent survey). Under suitable assumptions, $\AL$ methods enjoy strong theoretical guarantees, including last-iterate global convergence~\citep{xu2021iteration, liu2019nonergodic}. Empirically, \texttt{AL} methods tend to be stable and more robust when compared to common \texttt{PD} approaches~\citep{birgin2014practical}. For example,~\cref{fig:alm_vs_pd} shows that the final policy produced by our proposed \texttt{AL}-based approach attains convergence to the optimal policy (the optimality gap is small) and satisfies the constraints more reliably when compared to the \texttt{PD} approach in~\citet{ding2020natural}. 
\begin{figure}[ht]
  \begin{center}
    \centerline{\includegraphics[scale=1.2]{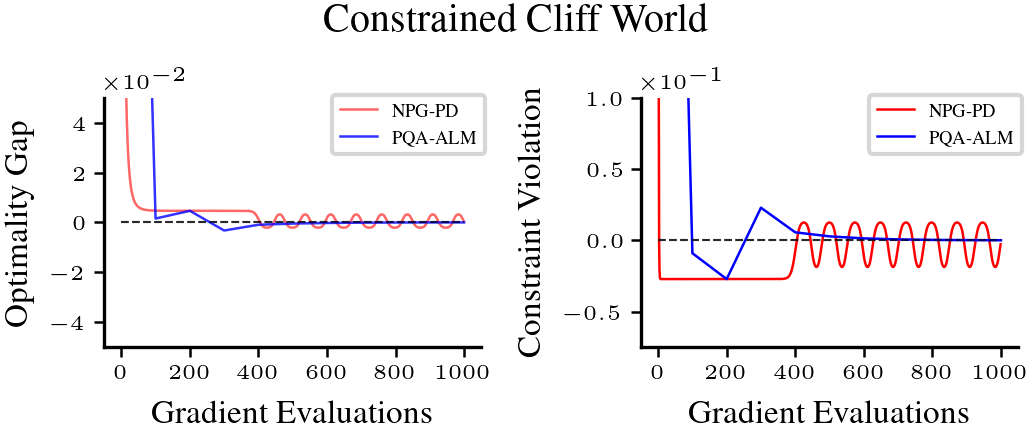}}
    \caption{
    Comparing the proposed $\AL$ method (\texttt{PPG-ALM}) to the primal-dual approach (\texttt{NPG-PD})~\citep{ding2020natural} on a tabular environment.
While the iterates of \texttt{NPG-PD} oscillate, \texttt{PPG-ALM} results in smooth convergence to zero optimality gap and constraint violation. 
    }
    \label{fig:alm_vs_pd}
  \end{center}
\end{figure}
\vspace{-4ex}
We note that such $\AL$ methods have been previously used for solving CMDPs. However, prior work either uses the \texttt{AL} method only as inspiration, without providing theoretical guarantees~\citep{dai2023augmented,li2021augmented}, or is restricted to simplistic settings and relies on computationally expensive procedures that do not scale in practice~\citep{muller2023cancellation}. Consequently, our objective is to develop a \emph{general, practical and theoretically principled \texttt{AL} framework for last-iterate convergence in CMDPs}. To this end, we make the following contributions.

\textbf{Contribution 1}: In~\cref{sec:framework}, we introduce a generic inexact \texttt{AL} framework for solving CMDPs.
Each iteration of the algorithm requires computing an approximate solution to a non-concave \texttt{AL} \textit{subproblem}.
Assuming access to an inexact solver for the subproblem, we prove that this framework requires $\gO(\nicefrac{1}{\varepsilon^2})$ iterations to output an $\varepsilon$-optimal policy (which simultaneously achieves an $\varepsilon$ optimality-gap and constraint violation) for the CMDP. This modular design separates the \texttt{AL} analysis from the choice of the subproblem solver, enabling flexibility across a range of settings.

\textbf{Contribution 2}:
We first consider the \textit{tabular setting} where the number of parameters scales with the size of the state-action space, and propose to solve the \texttt{AL} subproblem using policy optimization.
In~\cref{section:pg}, we show the policy gradient of the \texttt{AL} subproblem involves policy-dependent, non-stationary \textit{pseudo-rewards}. We use projected Q-ascent (\texttt{PQA})~\citep{xiao2022convergence,liu2025convergence} which yields last-iterate convergence guarantees for the \texttt{AL} subproblem. In~\cref{section:putting}, we combine these results with the standard $\AL$ analysis~\citep{liu2019nonergodic}, resulting in last-iterate guarantees for the CMDP. Specifically, for a target $\varepsilon > 0$, the proposed algorithm requires $\gO(\nicefrac{1}{\varepsilon^6)}$ policy gradient evaluations to obtain an $\varepsilon$-optimal stochastic policy.

\textbf{Contribution 3}: 
To extend our framework to practical settings with large state-action spaces, we use function approximation and consider a parametric class of policies. Such policies can be succinctly expressed using fewer parameters compared to the state-action space. For our $\AL$ framework, we solve each subproblem with a restricted policy class. In \cref{sec:fa}, we derive the \texttt{AL} policy gradient and show that it can be estimated as efficiently as in standard RL. As a result, our framework allows common PG methods such as \texttt{PPO}~\citep{schulman2017proximal} or \texttt{SPMA}~\citep{asad2024fast} to be used to solve the \texttt{AL} subproblem. Using the $\AL$ policy gradient, we follow~\citet{vaswani2021general,asad2024fast} and develop a surrogate optimization approach to update the parameterized policy. 

In~\cref{sec:linear_fa} we specialize this approach to log-linear policies~\citep{yuan2022linear,agarwal2021theory} and prove that a projected variant of \texttt{PQA} requires $\gO(\nicefrac{1}{\varepsilon^{6}})$ gradient evaluations to output an $\varepsilon$-optimal final stochastic policy.
Unlike~\citet{ding2023last}, our method does not involve non-standard algorithmic modifications and works for multiple constraints.
Furthermore, unlike~\citet{mondal2024last,montenegro2024last}, it is computationally efficient to implement in practice.

\textbf{Contribution 4}: In~\cref{appendix:large_experiments}, we evaluate our approach on large-scale constrained continuous control tasks from  Safety-Gymnasium~\citep{ji2023safety}.
We instantiate our framework using \texttt{PPO} and \texttt{SPMA}.
Our experiments demonstrate that the proposed approach performs comparably to strong heavily-engineered baselines such as \texttt{PPO-Lag}~\citep{ray2019benchmarking} and \texttt{CPO}~\citep{achiam2017constrained}.

%% file: Sections/02_problem_formulation.tex
\vspace{-2ex}
\section{Problem Formulation \& Background}
\vspace{-1ex}
We consider an infinite-horizon discounted constrained Markov decision process (CMDP)~\citep{altman2021constrained} defined by the tuple  $(\cS, \cA, \cP, \rho, \gamma, r, \{c_i\}_{i=1}^m, \{b_i\}_{i=1}^m)$. Here, $\cS$ is the set of states with $S = \abs{\cS}$, $\cA$ is the set of actions with $A = \abs{\cA}$, $\cP: \cS \times \cA \to \Delta_{\cS}$ is the transition probability function, $\rho \in \Delta_\cS$ is the initial state distribution, $\gamma \in [0, 1)$ is the discount factor and $r: \cS \times \cA \to [0, 1]$ is the reward function. The constraint reward functions are denoted by $\{c_i\}_{i=1}^m$ where $c_i: \cS \times \cA \to [0, 1]$ and $\{b_i\}_{i=1}^m$ is the set of corresponding non-negative thresholds.

For policy $\pi: \cS \to \Delta_{\cA}$, the \textit{state-action value function} $Q^\pi_\diamond: \cS \times \cA \to \R$ is defined as $Q^\pi_\diamond(s, a) \coloneq \E[\sum_{\tau=0}^\infty \gamma^\tau \diamond(s_\tau, a_\tau) | s_0 = s , a_0 = a]$ such that for $\tau \geq 1$, $s_{\tau+1} \sim \cP(\cdot | s_\tau, a_\tau)$, $a_{\tau+1} \sim \pi(\cdot | s_\tau)$ and $\diamond : \cS \times \cA \to \R$ is a bounded (pseudo) reward function equal to either $r$ or  $\{c_i\}_{i = 1}^m$ or their linear combination. The \textit{value function} $V^\pi_\diamond: \cS \to \R$ is defined as $V^\pi_\diamond(s) \coloneq \E_{a \sim {\pi(\cdot | s)}}[Q^\pi_\diamond(s, a)]$ and we use $V^{\pi}_{\diamond}(\rho) := \E_{s_0 \sim \rho}[V^{\pi}_{\diamond}(s_0)]$ to denote the expected value function. The \textit{reward state-action value function} and \textit{reward value function} are denoted as $Q^\pi_r$ and $V^{\pi}_r$. Similarly, for each constraint $i \in [m]$, we denote the \textit{constraint state-action value function} as $Q^\pi_{c_i}$ and the \textit{constraint value function} as $V^{\pi}_{c_i}$.

The objective is to find a policy $\pi$ that maximizes cumulative reward while satisfying the constraints. Formally, we aim to solve the following problem:
\begin{restatable}{problem}{cmdpobjectivepi}\label{problem:cmdp_objective_pi}
$\max_{\pi}  V^{\pi}_r(\rho) \quad \mathrm{ s.t. } \quad V^{\pi}_{c_i}(\rho) \geq {b_i} \quad \forall i \in [m]$ .
\end{restatable}
Unlike MDPs, the optimal policy $\pistar$ of~\cref{problem:cmdp_objective_pi} is stochastic but stationary~\citep{altman2021constrained}. A common approach to solving~\cref{problem:cmdp_objective_pi} is to form its Lagrangian and consider the following problem:
$
\max_{\pi} \min_{\lambda \in \R^m_+} \cL(\pi, \lambda) \coloneq V^{\pi}_r(\rho) + \sum_{i=1}^m \lambda_i \, [V^{\pi}_{c_i}(\rho) - b_i]
$,
where $\lambda$ is the Lagrange multiplier. Primal-dual (\texttt{PD}) methods solve the above max-min problem by iteratively and alternatively updating the policy and the Lagrange multiplier. For example, the policy is updated using natural policy gradient, whereas the Lagrange multiplier is updated using gradient descent~\citep{ding2020natural}. 

For a target $\varepsilon > 0$, prior work~\citep{jain2022towards,ding2020natural,liu2025sample} prove that it is possible to output a mixture policy $\bar{\pi}_T := \frac{1}{T} \, \sum_{t=1}^T \pi_t$ which is an $\varepsilon$-optimal solution to~\cref{problem:cmdp_objective_pi}, 
\begin{align}
    V^{\bar{\pi}_T}_r(\rho) \geq \max_{\pi} V^{\pi}_r(\rho) - \varepsilon,  \text{and} \quad V^{\bar{\pi}_T}_{c_i}(\rho) \geq b_i - \varepsilon, \quad \forall i \in [m].
\label{eq:eps-opt-sol}    
\end{align}
However, maintaining the mixture policy requires storing all the policies $\{\pi_t\}_{t = 1}^{T}$, making this memory intensive. Moreover, in order to deploy a mixture policy, a policy is repeatedly chosen uniformly at random from $\{\pi_1, \pi_2, \ldots, \pi_T\}$, and the guarantees for the mixture policy only hold in expectation over this random selection. This implies that constraints are only approximately satisfied in expectation. In many safety-critical settings, such behavior is undesirable because a randomly sampled policy may significantly violate the constraints. These disadvantages significantly limit the practicality of \texttt{PD} methods, and motivate the development of principled algorithms that do not rely on mixture policies. To that end, prior work~\citep{mondal2024last,ding2023last,muller2024truly,ying2022dual} has developed techniques that ensure \textit{last-iterate} convergence meaning that the resulting guarantees hold for the final stochastic policy. These works employ regularization  such as entropy-regularization on the policy and $\ell_2$ regularization on the dual variables. In contrast, we propose using the \texttt{AL} method for CMDPs, and describe it next. 

%% file: Sections/03_framework_no_mu.tex
\vspace{-3ex}
\section{Augmented Lagrangian Policy Optimization Framework}
\label{sec:framework}
\vspace{-1ex}
The \texttt{AL} method~\citep{bertsekas2014constrained,powell1969method,rockafellar1976augmented} is a classical approach for constrained optimization problems. It adds a quadratic penalty term to the Lagrangian to discourage constraint violations. In contrast to \texttt{PD} methods, the \texttt{AL} method considers the following problem: 
\begin{align}
\max_{\pi}\min_{\lambda} \cL^\beta(\pi, \lambda) \; \text{;} \; \cL^\beta(\pi, \lambda) &\coloneq V^{\pi}_r(\rho) + \frac{\alr}{2} \sum_{i=1}^m \, \parens*{-\min\braces*{V^{\pi}_{c_i}(\rho) - b_i - \frac{\lambda_i}{\beta}, 0}^2 + \frac{\lambda_i^2}{\alr^2}} 
\label{eq:al_max_min}
\end{align}
Here, $0 < \alr < \infty$ is the penalty parameter, and the min operator within the quadratic term corresponds to be the constraint violation for the $i^{\text{th}}$ constraint.
Intuitively, a larger parameter $\alr$ increases the weight of the quadratic  penalty term which forces the policy to satisfy the constraints, improving feasibility.
However, choosing $\alr$ too large leads to an ill-conditioned optimization problem.
As $\alr \to 0$, the \texttt{AL} reduces to the standard Lagrangian.

To find a solution to~\cref{eq:al_max_min},
the inexact \texttt{AL} method iteratively and alternatively updates the policy (primal variable) and the Lagrange multiplier (dual variable).
Specifically at iteration $t \geq 1$,  for a target sub-optimality $\eps_t > 0$ and dual variable $\lambda_t$, the inexact \texttt{AL} method seeks an approximate solution $\pi_{t+1}$ to the \texttt{AL} \textit{subproblem}, such that,  
\begin{align}\label{eq:al_subproblem_oracle}
    \cL^\beta(\pi_{t+1}, \lambda_t) \geq \argmax_{\pi} \, \cL^\beta(\pi, \lambda_t) - \eps_t. 
\end{align}
Once an approximate primal solution is obtained, the dual variable is updated according to: $\forall i \in [m]$
\begin{equation}\label{eq:dual_update}
\lambda_{t+1}[i] = \lambda_t[i] - \frac{\beta}{2} \, \parens*{V^{\pi_{t+1}}_{c_i}(\rho) - [b_i + \xi_i(\pi_{t+1})]}
\end{equation}
where $\xi_i(\pi) \coloneq \max\braces*{V^{\pi}_{c_i}(\rho) - b_i - \frac{\lambda_t[i]}{\beta}, 0}$ acts as a slack variable and prevents overly aggressive updates to the dual variable. For instance, $\xi_i(\pi_{t+1}) > 0 \iff V^{\pi_{t+1}}_{c_i}(\rho) >  b_i + \frac{\lambda_t[i]}{\beta}$. As a result, when constraint $i$ is well satisfied, the corresponding dual variable is only perturbed slightly. This avoids the excessive penalization of already-feasible constraints and prevents the dual variables from oscillating. On the other hand, if the constraint is not satisfied, the resulting update is akin to gradient descent w.r.t. $\lambda$ and enforces feasibility. Compared to \texttt{PD} methods that regularize the primal/dual variables, the $\AL$ method employs a quadratic regularization on the constraint violation. Unlike \texttt{PD} methods, the step-size for the dual update in~\cref{eq:dual_update} is directly determined by the penalty $\beta$.

Unlike the conventional use of the \texttt{AL} method in concave settings, the main difficulty of applying it to CMDPs is that $\cL^\beta(\pi, \lambda)$ is non-concave w.r.t $\pi$~\citep{agarwal2021theory}, making the $\AL$ subproblem challenging to solve accurately. Hence, we first assume access to the solver $\texttt{Oracle-AL}(\cL^\beta, \lambda_t, \eps_t)$ that returns a policy satisfying~\cref{eq:al_subproblem_oracle} for all $t \geq 1$. This abstraction separates the analysis of the \texttt{AL} iterations from the choice of the policy optimization method used to solve the subproblem.
Putting everything together, we describe our generic framework in~\cref{alg:generic_alm_pg}. 
\begin{algorithm}[H]
\begin{algorithmic}[1]
    \STATE \textbf{Input}:
    $\pi_1$ (primal variable), $\lambda_1 = 0$ (dual variable), $\alr > 0$ (penalty parameter), $\eps_t = \gO\left(\frac{1}{t^2}\right)$ (target sub-optimality for $\AL$ subproblem), $T$ (number of iterations)
    \FOR{$t = 1, 2, \dots, T$ }
    \STATE Form $\cL^\beta(\pi, \lambda_t) \coloneq %
    V^{\pi}_r(\rho) + \frac{\alr}{2} \sum_{i=1}^m \, \parens*{-\min\braces*{V^{\pi}_{c_i}(\rho) - b_i - \frac{\lambda_i}{\beta}, 0}^2 + \frac{\lambda_i^2}{\alr^2}}$
    \STATE $\pi_{t+1} = \texttt{Oracle-AL}(\cL^\beta, \lambda_t, \eps_t)$
    \STATE $\lambda_{t+1}[i] = \lambda_t[i] - \frac{\beta}{2} \, \parens*{V^{\pi_{t+1}}_{c_i}(\rho) - [b_i + \xi_i(\pi_{t+1})])}$, $\xi_i(\pi) \coloneq \max\big\{V^{\pi}_{c_i}(\rho) - b_i - \frac{\lambda_t[i]}{\beta}, 0\big\}$
        \ENDFOR
    \STATE \textbf{Return}: $\pi_{T+1}$
\end{algorithmic}
\caption{Generic $\AL$ Policy Optimization Method for~\cref{problem:cmdp_objective_pi}}
\label{alg:generic_alm_pg}
\end{algorithm}
\vspace{-1ex}
Assuming access to $\texttt{Oracle-AL}(\cL^\beta, \lambda, \eps)$, we follow a standard analysis for $\AL$ methods~\citep{liu2019nonergodic}, and characterize the convergence of~\cref{alg:generic_alm_pg} in~\cref{theorem:meta_theorem}. 
\vspace{-1ex}
\begin{restatable}{theorem}{metatheorem}\label{theorem:meta_theorem}
Assuming that for $t \geq 1$, there exists an $\texttt{Oracle-AL}(\cL^\beta, \lambda_t, \eps_t)$ which returns a policy $\pi_{t+1}$ satisfying~\cref{eq:al_subproblem_oracle}, for a given $\varepsilon \in (0, 1)$, and $T = \gO(\nicefrac{1}{\varepsilon^2})$,~\cref{alg:generic_alm_pg} with $\eps_t = \gO\left(\nicefrac{1}{t^2}\right)$, $\beta > 0$, and $\lambda_1 = 0$ returns a stochastic policy $\pi_{T+1}$ that is an $\varepsilon$-approximate solution to~\cref{problem:cmdp_objective_pi}.
\end{restatable}
\vspace{-1ex}
\textit{Proof Sketch:} We first recast the CMDP into an equivalent linear program (LP). 
For this, we define $\Pr^\pi[s_\tau = s, a_\tau = a | s_0]$ as the probability of visiting state $s$ and taking action $a$ at step $\tau$ when starting at the initial state $s_0$ and the unnormalized \textit{state-action occupancy measure} $\mu^\pi \in \frac{1}{1 - \gamma} \, (\Delta_{A})^{S} \coloneq \sum_{s_0} \rho(s_0) \, \sum_{\tau =0}^\infty \gamma^\tau \, \Pr^\pi[s_\tau = s, a_\tau = a | s_0]$. Note for any $\saom$, there exists a unique one-to-one mapping between $\saom$ and $\pi$ given by:
   $\pi(a | s) = \nicefrac{\mu^{\pi}(s, a)}{\sum_{a' \in \cA}\mu^{\pi}(s, a')}$, $\forall (s, a) \in \cS\times\cA$.
Consequently, the objective can be expressed directly in terms of the occupancy measure $\saom$, i.e.,  $V^{\pi}_r(\rho) = \dpd{\saom, r}$ and $V^{\pi}_{c_i}(\rho) = \dpd{\saom, c_i}$ for each constraint $i \in [m]$.
As a result, we can rewrite~\cref{problem:cmdp_objective_pi} as a LP in terms of $\mu^\pi$~\citep{altman2021constrained} as follows:
\begin{restatable}{problem}{cmdpobjectivesaom}
\label{problem:saom_objective}
$\max_{\saom \in \cK} \objsaom \quad \mathrm{ s.t. } \quad \dpd{\saom , c_i} \geq b_i\, \quad \forall i \in [m]$, \; where, \\ $\cK \coloneq \braces*{\saom \in \frac{1}{1 - \gamma} \, (\Delta_{A})^{S} \mid \forall s' \in \cS, \sum_{a' \in \cA} \saom(s', a') = \rho(s') + \gamma \sum_{s \in \cS} \sum_{a \in \cA} \cP(s' | s, a) \, \saom(s, a) }$
\end{restatable}
represents the set of all feasible state-action occupancy measures $\saom$ satisfying the Bellman flow constraints. Under this representation, both the objective and the constraints are linear w.r.t $\saom$. We next convert the inequality constraints in~\cref{problem:saom_objective} into equalities constraints by introducing non-slack variables $z \in \R^m_{+}$:
\begin{restatable}{problem}{cmdpobjectivesaomslack}
\label{problem:saom_objective_slack}
  $\max_{\saom \in \cK, z \in \R^m_{+}} \objsaom \quad \mathrm{ s.t. } \quad \dpd{\saom , c_i} - z_i = b_i\, \quad \forall i \in [m]$
\end{restatable}
\noindent By~\cref{lemma:saom_obj_equal_to_pi_obj,lemma:saom_slack_obj_equal_to_saom_obj}, for a target $\varepsilon > 0$, obtaining an $\varepsilon$-approximate solution to the LP in~\cref{problem:saom_objective_slack} in terms of $\saom$ induces a policy $\pi$ that is an
$\varepsilon$-approximate solution in~\cref{problem:cmdp_objective_pi}.
Hence, the LP formulation serves as an analysis device for the policy iterates of~\cref{alg:generic_alm_pg}; the algorithm does not explicit construct $\saom$. 

We now consider the \texttt{AL} method used to solve~\cref{problem:saom_objective_slack}.
The \texttt{AL} for~\cref{problem:saom_objective_slack} is defined as: \begin{equation}\label{eq:al_saom_slack}
\textstyle    \almsaomslack(\saom, z) \coloneqq \dpd{\saom, r} + \sum_{i=1}^m \lambda_i \parens*{\dpd{\saom, c_i} - b_i - z_i} - \frac{\beta}{2} \sum_{i=1}^m (\dpd{\saom, c_i} - b_i - z_i)^2.
\end{equation}
Note the \texttt{AL} subproblem is concave over $(\saom, z)$. By the one-to-one mapping between $\mu$ and $\pi$, and~\cref{lemma:aug_lag_eq_and_ineq_same,lemma:solving_slack_same_as_inequality}, the policy returned by $\texttt{Oracle-AL}(\cL^\beta, \lambda_t, \eps_t)$ corresponds to a pair $\left(\saomtt, \bar{\xi}(\saomtt)\right)$ satisfying
\vspace{-0.5ex}
\begin{align*}
\textstyle \almsaomslack(\saomtt, \bar{\xi}(\saomtt)) \ge \max_{\saom \in \cK, \, z \in \R^m_+} \almsaomslack(\saom, z) - \eps_t,
\end{align*}
where $\bar{\xi}_i(\saom) := \max\parens*{\dpd{\saom, c_i} - b_i - \nicefrac{\lambda_t[i]}{\beta}, 0}$ for each constraint $i \in [m]$.
Consequently,~\cref{alg:generic_alm_pg} can be analyzed using standard inexact \texttt{AL} analysis for concave problems with linear constraints.
Following~\citet{liu2019nonergodic}, using $\eps_t = \gO(\nicefrac{1}{t^2})$, the inexact \texttt{AL} method  results in the following convergence (\cref{lemma:inexact_al_bound_cmdp}):
\begin{align*}
V^*_r(\rho) - V^{\pi_{T+1}}_r(\rho) &\leq \gO\parens*{\nicefrac{1}{\sqrt{T}}} \quad \text{and} \quad b_i - V^{\pi_{T+1}}_{c_i}(\rho) \leq \gO\parens*{\nicefrac{1}{\sqrt{T}}} \quad \forall i \in [m].
\end{align*}
Setting $T = \gO(\nicefrac{1}{\varepsilon^2})$ ensures both the objective sub-optimality and constraint violation are $\mathcal{O}(\varepsilon)$. Thus, $\pi_{T+1}$ is an $\varepsilon$-approximate solution, completing the proof. \qed

\looseness=-1
Note the lifting argument used to transfer guarantees from $\saom$ to $\pi$ is not tied to this framework and can be extended beyond this setting. Crucially, the \texttt{AL} formulation ensures that the final policy is a valid stochastic policy and does not rely on a mixture of past policies. For $\varepsilon > 0$, the above meta-theorem (proved in~\cref{appendix:framework_proofs}) shows that the inexact \texttt{AL} method attains global last-iterate $\gO(\nicefrac{1}{\varepsilon^2})$ convergence if each \texttt{AL} subproblem can be solved to the required accuracy.  This result isolates the core requirement of the \texttt{AL} framework, and the remaining challenge is to construct an efficient oracle to solve the \texttt{AL} subproblem. We first consider the tabular setting in the following section.

%% file: Sections/04_tabular_setting_no_mu.tex
\section{Instantiating the \texttt{AL} Oracle: Tabular Setting}
\label{section:tabular}
We first derive the \texttt{AL} policy gradient (PG), and subsequently use it to solve the \texttt{AL} subproblem.

\subsection{The \texttt{AL} Policy Gradient}\label{section:pg}
We first consider the \textit{direct representation} in the tabular setting, in which the policy is parameterized  by its state–action probabilities, i.e., $\pi(\cdot | s) \in \Delta_A$ for all $s \in \cS$. In the standard unconstrained RL setting, a recent line of theoretical works~\citep{mei2020global,lu2024towards,agarwal2021theory,bhandari2021linear,lan2023policy,shi2023near,zhan2023policy} establishes global convergence guarantees for PG methods under various structural assumptions (e.g., smoothness or gradient domination).
However, these analyses do not directly apply in this setting due to the quadratic penalty term in the \texttt{AL}. To characterize this challenge precisely, in~\cref{proposition:alm_gu_pg} we first calculate the PG for the \texttt{AL} subproblem using the chain rule. For fixed dual variable $\lambda \in \R^m$,
\begin{align}
\nabla \cL^{\beta}(\pi, \lambda) & =  \sum_{s} \frac{d^{\pi}(s)}{1 - \gamma} \sum_{a} \, \qgu(s,a)  \, \text{;} \, \Gamma(\pi) := r - \beta \, \sum_{i=1}^m c_i \, \min\braces*{V^{\pi}_{c_i}(\rho) - b_i - \frac{\lambda_i}{\beta}, 0}. \label{eq:alm_gamma_pi}
\end{align}
where $d^{\pi}(s) \coloneqq \sum_{a' \in \cA} \saom(s, a')$.
Intuitively, the above equation states the policy aims to maximize a pseudo-reward $\Gamma(\pi)$ that depends on its current level of constraint satisfaction. 
For example, consider the single constraint case ($m=1$).
When the constraints are sufficiently satisfied i.e., $V^{\pi}_{c_1}(\rho) \geq b_1 + \frac{\lambda}{\alr}$, we have $[\Gamma(\pi)](s, a) = r(s, a)$. In this case, the PG solely considers maximizing the reward. On the other hand, when the constraint is violated, the policy must maximize a combination of the reward and cost, given by $[\Gamma(\pi)](s, a) = r(s, a) + c_1(s, a) \, \parens*{b_1 + \nicefrac{\lambda}{\alr} - V^{\pi}_{c_1}(\rho)} \, \alr$. Here, the weight on the cost term scales proportionally with the magnitude of the constraint violation $b_1 + \nicefrac{\lambda}{\alr} - V^{\pi}_{c_1}(\rho)$ and $\beta$. 

Unlike the standard PG, the key difference is that the pseudo-reward $\Gamma(\pi)$ depends on the current policy. As a result, the pseudo-reward is non-stationary across iterations, and the objective no longer satisfies the usual structural properties exploited in standard PG analyses. For example, classical tools such as the Bellman equation cannot be directly applied, since they rely on a fixed reward function. Consequently, extending global convergence guarantees of standard PG to the \texttt{AL} subproblem requires arguments that explicitly account for the non-stationary reward.

For convenience, we focus on the following abstract problem, where for a fixed $\lambda \in \R^m$, $\max_{\pi} G(\pi) \coloneq \cL^\beta(\pi, \lambda)$. We solve the problem using projected Q-ascent (\texttt{PQA})~\citep{xiao2022convergence,liu2025convergence} which closely resembles the projected policy gradient (\texttt{PPG})~\citep{agarwal2021theory} used to attain global convergence for general objectives with non-stationary rewards~\citep{zhang2020variational,barakat2024sample}. 
For iteration $k \geq 0$, the $\texttt{PQA}$ update is given as: 
\begin{equation}\label{eq:statewise_pqa_update}
    \forall s \in \cS \quad \text{,} \quad \pi_{k+1}(\cdot \mid s) = \text{Proj}_{\Delta_A}\bracks*{\pi_k(\cdot | s)  + \eta \, \qgupik(s, \cdot)}.
\end{equation}

To state our results, we first introduce and justify the following assumptions.
\begin{assumption}[Exploration]\label{assumption:exploration}
For all $s \in \cS$, $\rho(s) \geq \rho_{\min} > 0$.
\end{assumption}
\begin{assumption}[Bounded Pseudo-Rewards]\label{assumption:bounded_gen_util_reward}
For all $\pi$, $\gupi \in [0, U]$ for some constant $U \geq 0$.
\end{assumption}
\begin{assumption}[Smoothness]\label{assumption:gu_smooth}
The objective $G: \Pi \to \R$ is $L$-smooth with respect to the Euclidean norm, i.e., 
$\abs{G(\pi) - G(\pi') - \dpd{\nabla_\pi G(\pi'), \pi - \pi'}} \leq \frac{L}{2} \, \normsq{\pi - \pi'}$ for any pair of policies $\pi, \pi'$. 
\end{assumption}
\Cref{assumption:exploration} is standard in PG literature~\citep{asad2024fast,agarwal2021theory,xiao2022convergence,lu2024towards,mei2020global}, and helps isolate the optimization aspect from exploration. \Cref{assumption:bounded_gen_util_reward} ensures that the objective and its gradients remain well-defined and finite. Meanwhile,~\Cref{assumption:gu_smooth} is commonly adopted in PG literature~\citep{hazan2019provably,zhang2020variational,zhang2021convergence,barakat2023reinforcement,ying2022dual,barakat2024sample} and enables monotonic improvement of $G(\pi)$ under appropriate step-size choices. By adapting the proof technique in~\citet{zhang2020variational}, we establish the following result in~\cref{appendix:pqa_proofs}.
\begin{theorem}
Under~\cref{assumption:bounded_gen_util_reward,assumption:exploration,assumption:gu_smooth}, for $K \geq 1$ iterations, using \texttt{PQA} update in~\cref{eq:statewise_pqa_update} with $\eta = \frac{\rho_{\min}}{L}$ results in the following convergence: $\max_\pi G(\pi) - G(\pi_{K}) =  \gO\left( \frac{L}{\rho^2_{\min} \, (1 - \gamma)^3} \, \frac{1}{K}  \right)$. 
\end{theorem}
Thus, \texttt{PQA} achieves an $\gO(\nicefrac{1}{K})$ last-iterate convergence rate, and can be used as an efficient solver for the  \texttt{AL} subproblem. 
\vspace{-2ex}
\subsection{Putting Everything Together}
\label{section:putting}
\vspace{-1ex}
We now instantiate~\cref{alg:generic_alm_pg} with the \texttt{PQA} update for solving the \texttt{AL} subproblem (\texttt{PQA-ALM}). In order to do so, in~\cref{appendix:al_gu_lemmas}, we prove the \texttt{AL} sub-problem satisfies~\cref{assumption:bounded_gen_util_reward,assumption:gu_smooth}. We characterize the convergence of the resulting algorithm in the following corollary.
Note, to measure the complexity of solving the CMDP, we count the total number of PG evaluations.
For our \texttt{AL} framework, this corresponds to summing all inner iterations over the $T$ outer iterations. 
This enables a fair comparison to other methods, such as \texttt{PD} methods who have a single loop.
\begin{restatable}{corollary}{coralmpqa}\label{corollary:alm_pqa}
Under~\cref{assumption:exploration}, for a given $\varepsilon \in (0, 1)$, and $T = \gO(\nicefrac{1}{\varepsilon^2})$,~\cref{alg:generic_alm_pg} with $\eps_t = \gO\left(\nicefrac{1}{t^2}\right)$, $\beta > 0$, $\lambda_1 = 0$ 
and instantiating \texttt{Oracle-AL} using the \texttt{PQA} update in~\cref{eq:statewise_pqa_update} with $\eta = \frac{\rho_{\min}}{L_t}$ for $K_t = \frac{32 \, L_t \, \parens*{1 + \nicefrac{1}{(1 - \gamma) \, \rho_{\min}}}}{\eps_t}$ iterations returns a stochastic policy $\pi_{T+1}$ that is an $\varepsilon$-approximate solution to~\cref{problem:cmdp_objective_pi} after $\gO(\nicefrac{1}{\varepsilon^6})$ gradient evaluations.
\end{restatable}
\textbf{Comparison to prior work:}
For general nonconcave-convex problems, without exploiting additional structure, this rate matches the expected complexity proved by~\citet{lin2020gradient}. In the tabular CMDP setting~\citet{ying2022dual} establish last-iterate convergence using a dual-based algorithm with a nested loop structure, achieving an $\gO(\nicefrac{1}{\varepsilon^2})$ rate in terms of the total number of gradient evaluations.
\citet{gladin2023algorithm} extend their result and attain an $\gO(\log(\nicefrac{1}{\varepsilon}))$ last-iterate rate. However, their analysis and algorithm only apply in the tabular setting, limiting their applicability to large-scale problems. Similarily,~\citet{ding2023last} use an  optimistic primal–dual scheme (\texttt{OPG}),  to obtain a $\gO(\log(\nicefrac{1}{\varepsilon}))$ last-iterate rate that depends on problem-specific constants that can be unbounded.
\citet{moskovitz2023reload} propose a related optimistic primal–dual method over state–action occupancy measures (\texttt{ReLOAD}). However, their theoretical guarantees lacks an explicit rate and the method is computationally more expensive than solving the LP directly.

In the alternative finite-horizon setting,~\citet{muller2024truly} prove last-iterate convergence for regularized \texttt{NPG} \texttt{PD} method (\texttt{NPG-NRL}) with an overall complexity of $\mathcal \gO(\nicefrac{1}{\varepsilon^{10}})$.
However their rate also includes unknown polynomial dependence on the number of actions.
A closely line of work by ~\citet{muller2023cancellation} also advocates using \texttt{AL} method in the finite-horizon setting. However, their approach optimizes the \texttt{AL} subproblem directly using dynamic programming. This resulting in substantial computational overhead and prevents the method from extending to practical settings.
\begin{figure}[!h]
\begin{center}
\includegraphics[scale=1.2]{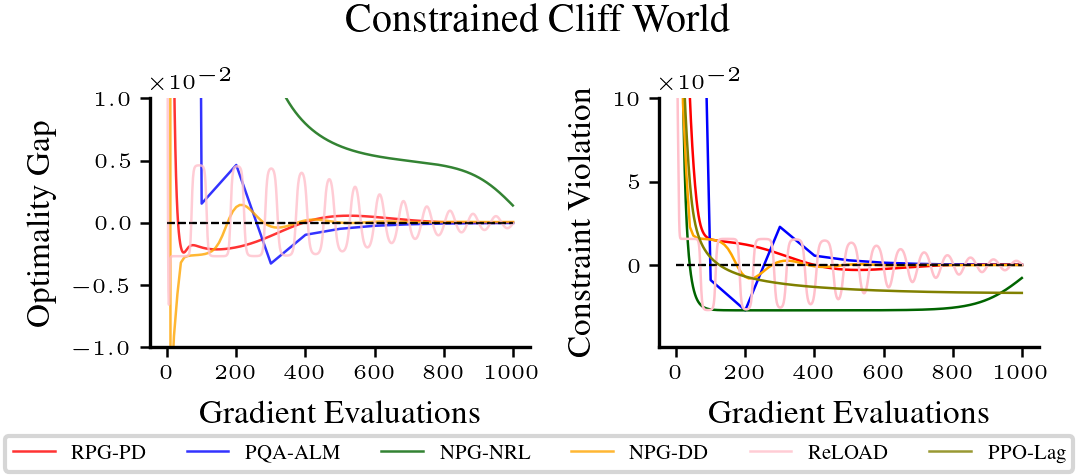}
\caption{Comparing \texttt{PQA-ALM}  against related prior methods. For all methods hyper-parameters were selected via grid search. We report the configuration with the smallest optimality gap among those satisfying $b - V^{\pi_{T+1}}_{c}(\rho) \leq \varepsilon = 0.001$. We observe that the \texttt{AL} methods, similar to \texttt{RPG-PD}~\citep{ding2023last} and \texttt{NPG-DD}~\citep{ying2022dual}, converge toward zero optimality gap and satisfy the constraint without noticeable oscillations. \texttt{ReLOAD} also eventually convergences but oscillates rapidly in the beginning. In contrast, \texttt{NPG-NRL}~\citep{muller2024truly} is unable to converge to the optimal policy.}
\label{fig:alm_vs_all}
\end{center}
\end{figure}

When the state-action space is large, it is common to employ function approximation (FA) to generalize across related states. In our framework, extending to this more practical setting requires solving the \texttt{AL} subproblem under FA, and we focus on this in the next section. 

%% file: Sections/05_handling_fa_no_mu.tex
\vspace{-2ex}
\section{Handling Function Approximation}
\label{sec:fa}
\vspace{-1ex}
We first describe a framework to handle general FA, and then specialize it to the class of log-linear policies for which we derive theoretical convergence guarantees. 

In particular, we consider the general parameterized policy class: $\Pi_{\theta} \coloneq \braces*{\pi \mid \exists \, \theta \text{ s.t. } \pi = \pitheta}$, where $\pitheta$ is the policy with learnable parameters $\theta$, and the model parameterizing the policy is implicit in the definition. For instance, this model could be chosen to be a neural network. Given $\Pi_\theta$, the only modification to our framework is that we now solve the \texttt{AL} subproblem within this parameterized class of policies, i.e., $\max_{\pi \in \Pi_{\theta}} \cL^{\beta}(\pi, \lambda)$. 

In order to do so, we note that the PG for the \texttt{AL} sub-problem can be calculated similar to that in~\cref{section:tabular}, i.e. if $Q^{\pi}_{\Gamma(\pi)}$ is the state-action value function for the pseudo-reward $\Gamma(\pi)$, $\nabla_\theta \al^\beta(\pi(\theta)) = \sum_{s} \frac{d^{\pi}(s)}{1 - \gamma} \sum_{a} \, \qgu(s,a) \frac{d \pitheta(a | s)}{d \theta}$, where, $\Gamma(\pi) = r - \beta \, \sum_{i=1}^m c_i \, \min\braces*{V^{\pi}_{c_i}(\rho) - b_i - \frac{\lambda_i}{\beta}, 0}$. Note the pseudo-reward $\Gamma(\pi)$ decomposes linearly into the original reward and constraint components. Consequently, its state–action value function $Q^{\pi}_{\Gamma(\pi)}$ admits to the following decomposition: $Q^{\pi}_{\Gamma(\pi)} = Q^{\pi}_r - \beta \, \sum_{i=1}^m Q^{\pi}_{c_i} \min\braces*{V^{\pi}_{c_i}(\rho) - b_i - \frac{\lambda_i}{\beta}, 0}$. As a result, $Q^{\pi}_{\Gamma(\pi)}$ can be estimated using approximations of $Q^{\pi}_r$, $Q^{\pi}_{c_i}$, and $V^{\pi}_{c_i}(\rho)$. Hence, the same estimation tools employed in standard RL algorithms can be used to implement this gradient. This flexibility allows common PG methods, such as \texttt{PPO}, \texttt{TRPO}~\citep{schulman2015trust}, or \texttt{SPMA} to be used to solve the \texttt{AL} subproblem.

Following~\cref{section:tabular}, we use policy optimization to solve the $\AL$ problem.  Inspired by \citet{vaswani2021general, asad2024fast}, we adopt a surrogate optimization approach for updating the parameterized policy. In particular, if $K$ is the total number of policy optimization iterations, at iteration $k \in [K]$, we first construct an intermediate policy using a generic PG method: $\pi_{k + \nicefrac{1}{2}} = \texttt{PGMethod}(\pi_k,  Q^{\pi_k}_{\Gamma(\pi_k)})$. We refer to the policy $\pi_{k + \nicefrac{1}{2}}$ as the idealized policy update if there were no restrictions on the policy class. Note that we can use a generic policy optimization method such as \texttt{PQA}, \texttt{NPG} or \texttt{SPMA} to produce this intermediate policy. In the presence of FA, $\pi_{k + \nicefrac{1}{2}}$ may not lie in the set of feasible policies. As a result, we project $\pi_{k + \nicefrac{1}{2}}$ onto the set of feasible policies using the forward KL divergence and solve the following optimization problem:
$\min_{\pi \in \Pi_\theta} \sum_{s \in \cS} d^{\pi_k}(s) \, \KL(\pi_{k + \nicefrac{1}{2}}(\cdot | s) \| \pi(\cdot | s))$.

We use the forward KL divergence due to its mode-covering behavior which encourages exploration resulting in better practical performance. Note that the subsequent framework does not depend on this choice, and we could use the reverse KL or any other $f$-divergence. Solving the above optimization problem exactly is typically intractable since $\pitheta$ is typically non-convex.
Following~\citep{vaswani2021general,asad2024fast,asad2025revisiting,tomar2020mirror,lavington2023target}, we transform this projection problem into an unconstrained optimization and form the following surrogate as a function of $\theta$,
\begin{equation}\label{eq:surrogate}
\ell_k(\theta) = \sum_{s \in \cS} d^{\pi_k}(s) \, \E_{a \sim \pi_{k + \nicefrac{1}{2}}(\cdot | s)} [-\log \pi_\theta(a | s)] \; \text{;} \; \theta_{k+1} \approx \argmin_{\theta} \ell_k(\theta) \; \text{;} \; \pi_{k+1} = \pi_{\theta_{k+1}} \, .
\end{equation}
In practice, this minimization is carried out approximately using a finite number of gradient steps. 

\textbf{Remark}. For large problems where $d^{\pi_k}(s)$ cannot be computed exactly, the surrogate $\ell_k(\theta)$ can be efficiently approximated using sampling. In particular, following~\citet{vaswani2021general,asad2024fast}, we can generate trajectories by rolling out policy $\pi_k$. For each encountered state, we compute $\pi_{k+\nicefrac{1}{2}}(\cdot | s)$ (using~\cref{eq:statewise_pqa_update}) and the empirical loss $\E_{a \sim \pi_{k+\nicefrac{1}{2}}(\cdot  |s)}[-\log \pitheta(a | s)]$ on the fly. 

This framework supports off-policy updates~\citep{vaswani2021general} similar to \texttt{TRPO} and \texttt{PPO}, allowing us to reuse the data gathered by policy $\pi_k$ and  perform multiple parameter updates. Next, we specialize this general framework to log-linear  policies~\citep{agarwal2021theory,yuan2022linear}. 
\vspace{-1ex}
\subsection{Log-Linear Policies}\label{sec:linear_fa}
\vspace{-1ex}
Given features $\varphi \in \R^{SA \times d}$ and parameters $\theta \in \R^d$ with $d \ll SA$, a log-linear policy is parameterized as: $\pi_\theta(a | s) = \frac{\exp(\dpd{\varphi(s, a) , \theta})}{\sum_{a'} \exp(\dpd{\varphi(s, a') , \theta})}$. We instantiate the intermediate policy using the \texttt{PQA} update, i.e. $\pi_{k+\nicefrac{1}{2}}(\cdot | s) = \text{Proj}_{\Delta_A}[\pi_k(\cdot | s)  + \eta \, \qgupik(s, \cdot)]$. For log-linear policies, the surrogate in~\cref{eq:surrogate} corresponds to a multi-class logistic regression classification with soft labels  enabling efficient optimization even in large state–action spaces~\citep{asad2024fast}. Minimizing this surrogate carries out the forward-KL projection, yielding the projected \texttt{PQA} update (\texttt{PPQA}). The corresponding pseudo-code (\cref{alg:projected_pqa}) is provided in \cref{appendix:gu_proofs}.

Following~\citet{agarwal2021theory,asad2024fast}, in order to control the projection error induced by approximate minimization of $\ell_k(\theta)$, we make the following assumptions:
\begin{restatable}[Bias]{assumption}{pqaBias}\label{assumption:pqa_bias}
For all $k \geq 1$, $\min_{\theta} \ell_k(\theta) \leq \eps_{\bias}$
\end{restatable}
\vspace{-1ex}
\begin{restatable}[Optimization Error]{assumption}{pqaOptErr}\label{assumption:pqa_optimizaiton_error}
Suppose $\theta_{k+1}$ is obtained by approximately minimizing $\ell_k(\theta)$, then, $\abs{\ell_k(\theta_{k+1}) - \min_\theta \ell_k(\theta)} \leq \eps_{\opt}$.
\end{restatable}
In~\cref{assumption:pqa_bias}, $\eps_{\bias}$ is determined by the expressivity of the given features. For example, in the case when $d=SA$, and the features are ``one-hot'' vectors, $\eps_{\bias} = 0$. 
\cref{assumption:pqa_optimizaiton_error} accounts for the optimization error arising from minimizing the surrogate using a finite number of gradient steps. 
This assumption is common in the analysis of PG methods with function approximation~\citep{asad2024fast,agarwal2021theory,yuan2022general,ding2023last}. 
In particular, if $n$ steps of gradient descent are used to minimize an $L$-smooth function $\ell_k(\theta)$, we have $\eps_{\opt} = \gO(\nicefrac{1}{n})$.

Putting everything together, we instantiate \texttt{Oracle-AL} in~\cref{alg:generic_alm_pg} using the \texttt{PPQA} update, and refer to the resulting algorithm as \texttt{PPQA-ALM}. We characterize the convergence rate of this algorithm in the following corollary. 
\begin{restatable}{corollary}{coralmppqa}\label{corollary:alm_ppqa}
In~\cref{alg:generic_alm_pg}, consider instantiating \texttt{Oracle-AL} with the \texttt{PPQA} update in~\cref{alg:projected_pqa} with the appropriate step-size and number of iterations. Under~\cref{assumption:exploration,assumption:pqa_optimizaiton_error,assumption:pqa_bias} such that $\eps_{\bias} \leq \gO\parens*{\nicefrac{1}{\varepsilon^{16}}}$, for a given $\varepsilon \in (0, 1)$,~\cref{alg:generic_alm_pg} with $T = \gO(\nicefrac{1}{\varepsilon^2})$, $\eps_t = \gO(\nicefrac{1}{t^2})$, $\beta > 0$, $\lambda_1 = 0$ returns a stochastic policy $\pi_{T+1}$ that is an $\varepsilon$-approximate solution to~\cref{problem:cmdp_objective_pi} after $\gO(\nicefrac{1}{\varepsilon^{6}})$ gradient evaluations.
\end{restatable}
Provided that $\eps_{\bias}$ is sufficiently small,
the above corollary proves that \texttt{PPQA-ALM} attains last-iterate convergence at an $\gO(\nicefrac{1}{\varepsilon^{6}})$ rate. 

\begin{figure}[!h]
  \begin{center}
    \centerline{\includegraphics[scale=1.2]{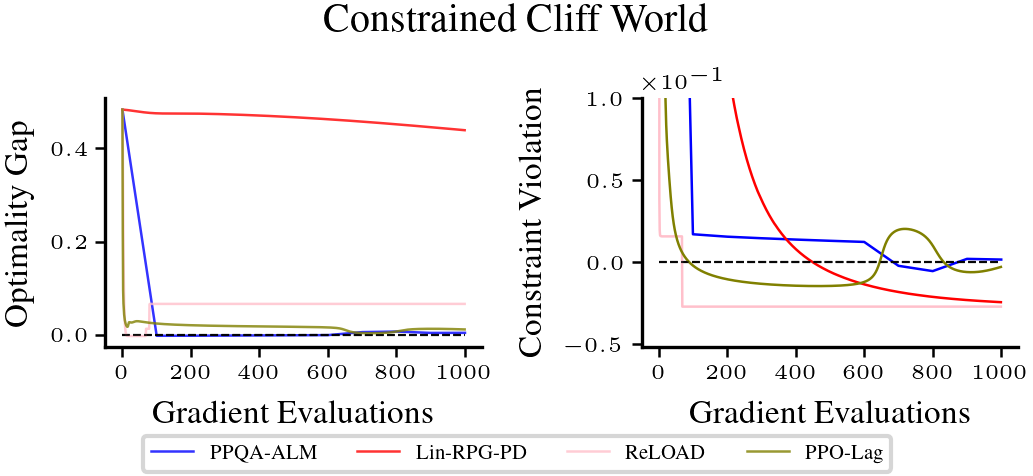}}
    \caption{
    Comparing \texttt{PPQA-ALM} to \texttt{Lin-RPG-PD}, \texttt{ReLoad} and \texttt{PPO-Lag} with log-linear policies. \texttt{Lin-RPG-PD} has comparable theoretical guarantees in this setting, whereas \texttt{ReLOAD} and \texttt{PPO-Lag} are included as a competitive empirical benchmark. Following~\citet{asad2024fast} we use tile-coded features for construct $\varphi$ with $d = 60 < SA$ (see~\cref{appendix:linear_fa_experiments} for additional details).  We find that \texttt{PPQA-ALM} reliably converges to the optimal policy, whereas both \texttt{Lin-RPG-PD} and \texttt{ReLOAD} perform poorly. While $\texttt{PPO-Lag}$ achieves comparable performance, the method lacks any theoretical guarantees.
    }
    \label{fig:alm_vs_all_linear}
  \end{center}
\end{figure}
\vspace{-4ex}
\looseness=-1
\textbf{Comparison to prior work:}~\citet{ding2023last} analyze a regularized \texttt{NPG} primal-dual method  for log-linear policies (\texttt{Lin-RPG-PD})
and establish a last-iterate $\gO(\nicefrac{1}{\varepsilon^6})$ rate. 
However, their analysis relies on several restrictive structural assumptions.
First, the policy is required to be projected on to a restricted simplex (a non-standard impractical modification). As further noted by~\citet{muller2024truly}, their guarantee only applies to the single constraint case. Second, to handle function approximation, their analysis relies on the compatible function approximation assumption~\citep{agarwal2021theory}. 
Under this assumption, the natural policy gradient direction is obtained by solving a regression problem at every iteration. As a result, a critic is re-estimated after each policy update and is tightly coupled to the policy parameterization. In contrast, our method does not require such structure. The surrogate depends only on value estimates, allowing the use of practical critic estimation methods and supports off-policy updates. More recently, by additionally imposing a Fisher non-degeneracy assumption~\citep{zhan2023policy}, \citet{montenegro2024last} establish a last-iterate rate of $\tilde{\gO}(\nicefrac{1}{\varepsilon^6})$, which is subsequently improved to $\gO(\nicefrac{1}{\varepsilon^4})$ by \citet{mondal2024sample}. However, this assumption, fails to hold for common policy parameterizations such as log-linear policies. \textit{In conclusion, our method provides the first last-iterate convergence for CMDPs under log-linear policies without further restrictive assumptions about the policy class.}

%% file: Sections/07_conclusion.tex
\vspace{-2ex}
\section{Conclusion}
\vspace{-1ex}
\looseness=-1
We developed a policy optimization framework with the \texttt{AL} method for CMDPs that provides last-iterate guarantees. In the tabular setting, we prove the framework attains an $\gO(\nicefrac{1}{\varepsilon^6})$ last-iterate rate and under standard assumptions, a comparable rate under linear function approximation. The resulting approach is both principled and practical, demonstrating competitive performance in tabular environments as well as in large-scale continuous-control tasks. The main limitation of our result is the strong dependence on $\rho_{\min}$ in the tabular setting and $\eps_{\bias}$ for log-linear policies. Future directions include weakening these dependencies and incorporating actor–critic or off-policy methods into the theoretical framework.

%% file: Appendix/A1_definitions.tex
\section{Definitions} \label{appendix:definitions}
\textbf{[Smoothness]} A differentiable function $f$ is $L$-smooth if for all $v$ and $w$
\begin{equation*}
	\abs{f(v) - f(w) - \dpd{\gradf{w}, v - w}} \leq \frac{L}{2} \normsq{v - w}.
\end{equation*}
\textbf{[Lipschitz continuity]} A function $f$ is B-Lipschitz if for all $v$ and $w$
\begin{equation*}
\abs{f(v) - f(w)} \leq B \, \norm{v - w}_2.
\end{equation*}
\textbf{[Convexity]} A function $f$ is convex if for all $v$ and $w$
\begin{equation*}
f(v) - f(w) \geq \dpd{\nabla f(w), v - w}.
\end{equation*}

%% file: Appendix/A2_ALM_PG_proofs.tex
\section{Proofs for the Augmented Lagrangian Policy Gradient Method}\label{appendix:framework_proofs}
We first state a meta-theorem describing the sub-optimality guarantees of~\cref{alg:generic_alm_pg}, assuming access to an oracle \texttt{Oracle-AL} that solves the augmented Lagrangian (\texttt{AL}) subproblem.
\metatheorem*
\begin{proof}
To prove this result, we reduce finding an approximate solution to~\cref{problem:cmdp_objective_pi} to an equivalent linear program (LP)  (refer to~\cref{appendix:reduction_from_inequality_to_equality} for additional details) . 
Using the standard occupancy measure formulation of CMDPs~\citep{altman2021constrained}, we have
\begin{equation}\label{eq:saom_constrained_objective_slack}
  \max_{\substack{\saom \in \cK, \\ z \in \R^m_{+}}} \objsaom \quad \text{ s.t. } \quad \dpd{\saom , c_i} - z_i = b_i\, \quad \forall i \in [m]
\end{equation}
where $z \in \R^m_{+}$ are non-negative slack variables.
and $$\cK := \left\{ \saom \in \frac{1}{1 - \gamma} \, (\Delta_A)^S \mid \forall s' \in \cS, \sum_{a' \in \cA} \saom(s', a') = \gamma \sum_{s \in \cS} \sum_{a \in \cA} \cP(s' | s, a) \, \saom(s, a) + \rho(s')\right\}$$ represents the set of all feasible state-action occupancy measures satisfying the Bellman flow constraints. 
According to~\cref{lemma:saom_obj_equal_to_pi_obj,lemma:saom_slack_obj_equal_to_saom_obj} , finding an $\varepsilon$-approximate solution to the LP in~\cref{problem:saom_objective_slack} 
in terms of $\saom$ induces a policy $\pi$ that is an
$\varepsilon$-approximate solution in~\cref{problem:cmdp_objective_pi}.

Next, we relate the oracle guarantee in~\cref{alg:generic_alm_pg} to the \texttt{AL} method applied to~\cref{problem:saom_objective_slack}.
For a fixed outer iteration $t \in [T]$ and target $\eps_t > 0$, by the oracle guarantee:
\begin{equation*}
   \cL^\beta(\pi_{t+1}, \lambda_t) \geq  \max_{\pi \in \Pi} \cL^\beta(\pi, \lambda_t) - \eps_t.
\end{equation*}
Using the one-to-one mapping between $\saom$ and $\pi$, given by $\pi(a | s) = \frac{\saom(s, a)}{\sum_{a'} \saom(s, a')}$,
the state-action occupancy $\mu^{\pi_{t+1}}$ satisfies:
\begin{equation*}
   \Ft(\mu^{\pi_{t+1}}) \geq  \max_{\saom \in \cK} \Ft(\mu^{\pi_{t+1}}) - \eps_t,
\end{equation*}
where
\begin{equation*}
    \F(\saom) :=  \dpd{\saom, r} + \sum_{i=1}^m \frac{\beta}{2} \bracks*{-\parens*{\min\braces*{\dpd{\saom, c_i} - b_i - \frac{\lambda_i}{\beta}, 0}}^2  + \frac{\lambda_i^2}{\beta^2}}. 
\end{equation*}
As a result, by~\cref{lemma:solving_slack_same_as_inequality,lemma:inexact_al_bound_cmdp},~\cref{alg:generic_alm_pg} with $T > 0$, $\eps_t = \frac{\sigma}{t^2}$, for a $\sigma > 0$, $\beta > 0$, $\lambda_1 = 0$ satisfies that for all $i \in [M]$,
\begin{align*}
V^*_r(\rho) - V^{\pi_{T+1}}_r(\rho) = \dpd{\mu^* - \mu^{\pi_{T+1}}, r}  &\leq \frac{2 \, B}{\sqrt{\beta}} \left(\frac{\sqrt{C}}{\sqrt{T}} + \frac{\sqrt{2 \, \sigma}}{T}\right) + \frac{\sigma}{T^2}, \\
b_i - V^{\pi_{T+1}}_{c_i}(\rho) = b_i - \dpd{\mu^{\pi_{T+1}}, c_i} &\leq \frac{2}{\sqrt{\beta}} \left(\frac{\sqrt{C}}{\sqrt{T}} + \frac{\sqrt{2 \, \sigma}}{T}\right).
\end{align*}
where the constants $B$, $C$ defined in~\cref{lemma:inexact_al_bound_cmdp}.
Hence, for any $\varepsilon \in (0, 1)$, setting $T = \gO(\nicefrac{1}{\varepsilon^2})$ implies that for all $i \in [M]$,
\begin{align*}
V^*_r(\rho) - V^{\pi_{T+1}}_r(\rho) &\leq \gO(\varepsilon), \\
b_i - V^{\pi_{T+1}}_{c_i}(\rho) &\leq \gO(\varepsilon).
\end{align*}
\end{proof}

\subsection{Tabular Setting}
\begin{lemma}\label{lemma:pqa_aug_lag}
Under~\cref{assumption:bounded_gen_util_reward,assumption:exploration}, 
fixed outer iteration $t \geq 1$, dual variable $\lambda_t$, given $\eps_t \in
(0, 1)$ and let $\cL(\pi, \lambda_t)$ denote the \texttt{AL} defined
in~\cref{eq:al_max_min} at iteration t. Using the \texttt{PQA}
update in~\cref{eq:pqa_update_appendix} with step-size $\eta =
\frac{\rho_{\min}}{L_t}$ 
results in a policy that satisfies~\cref{eq:al_subproblem_oracle} after $K_t = \frac{32 \, L_t \, \parens*{1 + \nicefrac{1}{(1 - \gamma) \, \rho_{\min}}}}{(1 - \gamma)^2 \, \rho_{\min} \, \eps_t}$ iterations where by~\cref{assumption:exploration}, $ \rho_{\min} > 0$.
\end{lemma}
\begin{proof}
By~\cref{proposition:aug_lag_is_gu_function}, the \texttt{AL} $\cL_t$ is a concave general utility and by~\cref{lemma:alm_saom_smooth} is $\cL_t$ is $L_t$-smooth.

Under~\cref{assumption:exploration}, by~\cref{theorem:pqa_last_iterate} the \texttt{PQA} update in~\cref{eq:pqa_update_appendix} using a step-size $\eta = \frac{\rho_{\min}}{L_t}$ with the general utility $G(\pi) = \cL_t(\pi, \lambda_t)$ results in the following convergence:
\begin{equation*}
\max_{\pi \in \Pi} \cL_t(\pi, \lambda_t) - \cL_t(\pi_{K_{t}}, \lambda_t)  \leq \frac{32 \, L_t \, \parens*{1 + \nicefrac{1}{(1 - \gamma) \, \rho_{\min}}}}{(1 - \gamma)^2 \, \rho_{\min} \, K_t}
\end{equation*}
Hence, using $K_t = \frac{32 \, L_t \, \parens*{1 + \nicefrac{1}{(1 - \gamma) \, \rho_{\min}}}}{(1 - \gamma)^2 \, \rho_{\min} \, \eps_t}$ iterations results in the policy $\pi_{K_t}$ satisfying the oracle condition in~\cref{eq:al_subproblem_oracle}.
\end{proof}

\coralmpqa*
\begin{proof}
\Cref{assumption:bounded_gen_util_reward} is satisfied by~\cref{lemma:U_alm}.
Then, under~\cref{assumption:exploration}, using~\cref{lemma:pqa_aug_lag} the PQA update in~\cref{eq:pqa_update_appendix} with step-size $\eta = \frac{\rho_{\min}}{L_t}$ 
results in a policy that satisfies the oracle condition in~\cref{eq:al_subproblem_oracle} after $K_t = \frac{32 \, L_t \, \parens*{1 + \nicefrac{1}{(1 - \gamma) \, \rho_{\min}}}}{(1 -\gamma)^2 \, \rho_{\min} \, \eps_t}$ iterations.

As a result, by~\cref{theorem:meta_theorem}, for a given $\varepsilon \in (0, 1)$, and $T = \gO(\nicefrac{1}{\varepsilon^2})$,~\cref{alg:generic_alm_pg} with $\eps_t = \frac{\sigma}{t^2}$, for any $\sigma > 0$, $\beta > 0$, $\lambda_1 = 0$ returns policy $\pi_{T+1}$ that satisfies 
\begin{align*}
   &V^*_r(\rho) - V^{\pi_{T+1}}_r(\rho)  \leq \gO(\varepsilon) \\
   \text{and}  \quad &b_i - V^{\pi_{T+1}}_{c_i}(\rho)  \leq  \gO(\varepsilon) \quad \forall i \in [m].
\end{align*} 
To characterize the total number of gradient evaluations, each primal update requires that
\begin{equation*}
    K_t =\gO\parens*{\frac{1}{\eps_t}}  \quad \text{ gradient evaluations with } \quad \eps_t = \frac{\sigma}{t^2}.
\end{equation*}
Hence, summing over all $t \in [T]$ yields that 
\begin{equation*}
\sum_{t=1}^T K_t = \sum_{t=1}^T \gO\parens*{t^2} = \gO(T^3) = \gO(\nicefrac{1}{\eps^{6}}).
\end{equation*}
\end{proof}

\subsection{Linear Function Approximation Setting}\label{appendix:framework_proofs_ppqa}
\begin{lemma}\label{lemma:projected_pqa_aug_lag}
Under~\cref{assumption:bounded_gen_util_reward,assumption:exploration}, 
fixed outer iteration $t \geq 1$, dual variable $\lambda_t$, target $\eps_t \in (0, 1)$ and let $\cL_t$ denote the \texttt{AL} defined in~\cref{eq:al_max_min} at iteration $t$.
Furthermore,
under~\cref{assumption:pqa_optimizaiton_error,assumption:pqa_bias} where $\eps_{\bias} + \eps_{\opt} \leq \gO\parens*{\frac{1}{K_t^4}}$, 
using the \texttt{PPQA} update in~\cref{alg:projected_pqa} 
with step-size
$\eta = \frac{\rho_{\min}}{L_t}$
results in a policy that satisfies~\cref{eq:al_subproblem_oracle} after
$K_t = \gO(\nicefrac{1}{\eps_t})$ iterations.
by~\cref{assumption:exploration}, $ \rho_{\min} > 0$.
\end{lemma}
\begin{proof}
By~\cref{proposition:aug_lag_is_gu_function}, the \texttt{AL} $\cL_t$ is a concave general utility and by~\cref{lemma:alm_saom_smooth} is $\cL_t$ is $L_t$-smooth.

Under~\cref{assumption:exploration,}, 
according to~\cref{theorem:projected_pqa_last_iterate},
the \texttt{PPQA} update in~\cref{alg:projected_pqa} with $\eta = \frac{\rho_{\min}}{L_t}$,  with the general utility $G(\pi) = \cL_t(\pi, \lambda_t)$ results in the following convergence:
\begin{equation*}
\max_{\pi}\cL_t(\pi, \lambda_t) - \cL_t(\pi_{K_t}) \leq \underbrace{\frac{4 \, C}{K_t}}_{\text{policy optimization error}} + 
\underbrace{2 \, C \, \sqrt{\frac{\sqrt{2} \, U \, \sqrt{\eps_{\mathrm{opt}} + \eps_{\mathrm{bias}}}}{2 \, C \, (1 - \gamma)^3 \, \rho_{\min}}}}_{\text{projection error}}
\end{equation*} 
where $C \coloneq \frac{16}{(1 - \gamma)^2 \, \rho_{\min}} \, \parens*{1 + \frac{1}{(1 - \gamma) \, \rho_{\min}}}$.
In particular, if the projection errors are small enough such that $\eps_{\bias} + \eps_{\opt} \leq \gO\parens*{\frac{1}{K_t^4}}$, then
\begin{equation*}
\max_{\pi}\cL_t(\pi, \lambda_t) - \cL_t(\pi_{K_t+1}) \leq \gO\parens*{\frac{1}{K_t}}.
\end{equation*}
Hence, using $K_t = \gO\parens*{\frac{1}{\eps_t}}$ gradient evaluations results a policy $\pi_{K_t}$ that satisfies~\cref{eq:al_subproblem_oracle}.
\end{proof}
\coralmppqa*
\begin{proof}
\Cref{assumption:bounded_gen_util_reward} is satisfied by~\cref{lemma:U_alm} with 
$U = 1 + \alr \,\norm{b}_1   + \sqrt{m} \, \sqrt{\normsq{M} + \frac{\beta \sigma \pi^2}{3}} + \norm{M}_1$
where $M := (\min_{i \in [m]} \zeta_i (1 - \gamma))^{-1}$ and $\zeta_i = \max_{\pi} V^{\pi}_{c_i}(\rho) - b > 0$.

Moreover, under~\cref{assumption:pqa_optimizaiton_error,assumption:pqa_bias} where $\eps_{\bias} + \eps_{\opt} \leq \gO\parens*{\frac{1}{T^8}}$ we have \begin{align*}
   \eps_{\bias} + \eps_{\opt} &\leq \gO\parens*{\frac{1}{T^8}}  \\
      &\leq \gO\parens*{\frac{1}{t^8}}  \tag{For any $t \leq T$} \\
      &= \gO\parens*{\frac{\sigma^4}{t^8}} \tag{For any $\sigma > 0$} \\
      &= \gO\parens*{\eps_t^4} \tag{Since $\eps_t = \frac{\sigma}{t^2}$} \\
\end{align*}
Hence, all $t \in [1, T]$, with $K_t = \gO\parens*{\nicefrac{1}{\eps_t}}$,
$\eps_{\bias} + \eps_{\opt} \leq \gO\parens*{\eps_t^4} = \gO\parens*{\nicefrac{1}{K_t^4}}$.
As a result, under~\cref{assumption:exploration}, by~\cref{lemma:projected_pqa_aug_lag} using the \texttt{PPQA} update in~\cref{alg:projected_pqa} 
with step-size
$\eta = \frac{\rho_{\min}}{L_t}$, results in a policy that satisfies~\cref{eq:al_subproblem_oracle} after
$K_t = \gO(\nicefrac{1}{\eps_t})$ gradient evaluations.

By~\cref{theorem:meta_theorem}, for a given $\varepsilon \in (0, 1)$, and $T = \gO(\nicefrac{1}{\varepsilon^2})$,~\cref{alg:generic_alm_pg} with $\eps_t = \frac{\sigma}{t^2}$, for any $\sigma > 0$, $\beta > 0$, $\lambda_1 = 0$ returns policy $\pi_{T+1}$ which satisfies that
\begin{align*}
   &V^*_r(\rho) - V^{\pi_{T+1}}_r(\rho)  \leq \gO(\varepsilon) \\
   \text{and}  \quad &b_i - V^{\pi_{T+1}}_{c_i}(\rho)  \leq  \gO(\varepsilon) \quad \forall i \in [m].
\end{align*} 
Hence, summing over all $t \in [T]$, the total number of gradient evaluations required is
\begin{equation*}
\sum_{t=1}^T K_t = \sum_{t=1}^T \gO\parens*{t^2} = \gO(T^{3}) = \gO(\nicefrac{1}{\eps^{6}}).
\end{equation*}
\end{proof}

\section{Proof of Technical Tools}\label{appendix:technical_tools}
\subsection{Reduction from Inequality to Equality Constraints}\label{appendix:reduction_from_inequality_to_equality}
The objective is to solve the following CMDP
\cmdpobjectivepi*
\noindent Using the standard linear programming (LP) formulation of CMDPs~\citep{altman2021constrained}, solving~\cref{problem:cmdp_objective_pi} with respect to a policy $\pi$ is equivalent to solving the following LP
\cmdpobjectivesaom*
\noindent where $\cK := \left\{ \saom \in \frac{1}{1 - \gamma} \, \Delta_A^S \mid \forall s' \in \cS, \sum_{a' \in \cA} \saom(s', a') = \gamma \sum_{s \in \cS} \sum_{a \in \cA} \cP(s' | s, a) \, \saom(s, a) + \rho(s')\right\}$ represents the set of all feasible state-action occupancy measures satisfying the Bellman flow constraints. 
Under this representation, both the objective and constraints are now linear with respect to the decision variable $\saom$.
Furthermore, for any $\saom \in \cK$, the unique one-to-one mapping between $\saom$ and $\pi$ is given by
\begin{equation}\label{eq:one_one_saom_pi}
   \pi(a | s) = \frac{\mu^{\pi}(s, a)}{\sum_{a' \in \cA}\mu^{\pi}(s, a')} \quad \forall (s, a) \in \cS\times\cA.
\end{equation}
To handle inequality constraints, we introduce slack variables and consider the equivalent equality-constrained LP
\cmdpobjectivesaomslack*
\noindent For a target $\varepsilon > 0$, we relate obtaining 
$\varepsilon$-approximate solutions of~\cref{problem:saom_objective_slack} to~\cref{problem:cmdp_objective_pi} with the following Lemmas.
\begin{lemma}\label{lemma:saom_slack_obj_equal_to_saom_obj}
For a given $\varepsilon \in (0, 1)$, 
let $(\mu^{\pi}, z) \in \cK \times \R^m_+$ be an $\varepsilon$-approximate solution to~\cref{problem:saom_objective_slack}, i.e.,
\begin{align*}
  \dpd{\mu^{\pi}, r} \geq \dpd{\mu^*, r} - \varepsilon \quad \text{and} \quad \dpd{\mu^{\pi}, c_i} - z_i = b_i - \varepsilon\, \quad \forall i \in [m]
\end{align*}
then $\mu^{\pi}$ is an $\varepsilon$-approximate solution to~\cref{problem:saom_objective}, i.e.,
\begin{align*}
  \dpd{\mu^{\pi}, r} \geq \dpd{\mu^*, r} - \varepsilon \quad \text{and} \quad \dpd{\mu^{\pi}, c_i} \geq b_i - \varepsilon\, \quad \forall i \in [m].
\end{align*}
where $\mu^*$ is the optimal state-action occupancy measure of~\cref{problem:saom_objective}.
\end{lemma}
\begin{proof}
For a given $\varepsilon \in (0, 1)$, 
let $(\mu^{\pi}, z)$ be an $\varepsilon$-approximate solution to~\cref{problem:saom_objective_slack}.
Since the slack variables $z_i \geq 0$ for each constraint $i \in [m]$    
\begin{align*}
\dpd{\mu^{\pi}, c_i} =  z_i + b_i - \varepsilon \geq b_i - \varepsilon.
\end{align*}
Hence, $\saom$ is an $\varepsilon$-approximate solution to to~\cref{problem:saom_objective}.
\end{proof}
\begin{lemma}\label{lemma:saom_obj_equal_to_pi_obj}
For a given $\varepsilon \in (0, 1)$, 
let $\mu^{\pi} \in \cK$ be an $\varepsilon$-approximate solution to~\cref{problem:saom_objective}, 
Then the policy $\pi$ induced by $\saom$ is an  $\varepsilon$-solution to~\cref{problem:cmdp_objective_pi}, i.e.,
\begin{align*}
  V^{\pi}_r(\rho) \geq V^*_r(\rho) - \varepsilon \quad \text{and} \quad V^{\pi}_{c_i}(\rho) \geq b_i - \varepsilon\, \quad \forall i \in [m]
\end{align*}
where $\pistar$ is the optimal policy of~\cref{problem:cmdp_objective_pi}.
\end{lemma}
\begin{proof}
For a given $\varepsilon \in (0, 1)$, 
let $\mu^{\pi}$ be an $\varepsilon$-approximate solution to~\cref{problem:saom_objective}.
Using the one-to-one mapping between $\saom$ and $\pi$ in~\cref{eq:one_one_saom_pi}, the policy $\pi$ induced by $\saom$ satisfies:
\begin{align*}
    \dpd{\saom, r} = V^{\pi}_r(\rho) \quad \text{and} \quad
    \dpd{\saom, c_i} = V^{\pi}_{c_i}(\rho) \quad \forall i \in [m].
\end{align*}
Hence,
\begin{align*}
  V^{\pi}_r(\rho) \geq V^*_r(\rho) - \varepsilon \quad \text{and} \quad V^{\pi}_{c_i}(\rho) \geq b_i - \varepsilon\, \quad \forall i \in [m].
\end{align*}
and is an $\varepsilon$-approximate policy to~\cref{problem:cmdp_objective_pi}.
\end{proof}
As a result, by chaining~\cref{lemma:saom_obj_equal_to_pi_obj,lemma:saom_slack_obj_equal_to_saom_obj} together, finding an $\varepsilon$-approximation solution to the LP in~\cref{problem:saom_objective_slack} 
in terms of $\saom$ induces a policy $\pi$ that is an
$\varepsilon$-approximation solution in~\cref{problem:cmdp_objective_pi}.

Next, we present the \texttt{AL} method applied to~\cref{problem:saom_objective_slack}.

\subsection{The Augmented Lagrangian Method on the Linear Program View of CMDPs}
The augmented Lagrangian (\texttt{AL}) for~\cref{problem:saom_objective_slack} is defined as: \begin{equation}\label{eq:al_saom_slack}
    \almsaomslack(\saom, z) \coloneqq \dpd{\saom, r} + \sum_{i=1}^m \lambda_i \parens*{\dpd{\saom, c_i} - b_i - z_i} - \frac{\beta}{2} \sum_{i=1}^m (\dpd{\saom, c_i} - b_i - z_i)^2.
\end{equation}
where $\lambda \in \R^m$ is the Lagrange multiplier associated with the equality constraint and $\beta > 0$ is the penalty parameter.
The \texttt{AL} method for solving~\cref{problem:saom_objective_slack} 
proceeds by alternating between approximately maximizing the \texttt{AL} and updating the dual variables.
At iteration $t \in [T]$, given an target $\eps_t \in (0, 1)$ and Lagrange multiplier $\lambda_t \in \R^m$, the primal update computes a pair $(\mu^{\pi_{t+1}}, z_{t+1})$ satisfying
\begin{equation}\label{eq:almsaomslack_subproblem}
   \almsaomslackt(\mu^{\pi_{t+1}}, z_{t+1})  
   \geq \max_{(\saom, z) \in \cK \times
   \R^m} 
   \almsaomslackt(\saom, z) - \eps_t.
\end{equation}
The dual variables are then updated as, for each constraint $i \in [m]$:
\begin{equation*}
    \lambda_{t+1}[i] = \lambda_t[i] - \frac{\beta}{2} \parens*{\dpd{\mu^{\pi_{t+1}}, c_i} - (b_i + z_{t+1}[i])}.
\end{equation*}
The \texttt{AL} method for~\cref{problem:saom_objective_slack} is summarized below.

\begin{algorithm}[H]
\begin{algorithmic}[1]
    \STATE \textbf{Input}:
    $\mu^{\pi_1}$ (primal variable), $\lambda_1 = 0$ (dual variable), $\alr > 0$ (constant constraint penalty), $\eps_t = \frac{\sigma}{t^2}$ (subproblem target for any $\sigma > 0$), $T$ (number of iterations)
    \FOR{$t = 1, 2, \dots, T$ }
    \STATE Formulate \texttt{AL}: $\almsaomslack(\saom, z) \coloneq \dpd{\saom, r} + \sum_{i=1}^m \lambda_i \parens*{\dpd{\saom, c_i} - b_i - z_i} - \frac{\beta}{2} \sum_{i=1}^m (\dpd{\saom, c_i} - b_i - z_i)^2$.
    \STATE Find the approximate solution pair $(\mu^{\pi_{t+1}}, z_{t+1}) \in \cK \times \R^m_+$ such that
       $$\almsaomslackt(\mu^{\pi_{t+1}}, z_{t+1})   \geq \max_{(\saom, z) \in \cK \times
       \R^m} 
       \almsaomslackt(\saom, z) - \eps_t. $$
    \STATE  $\lambda_{t+1}[i] = \lambda_t[i] - \frac{\beta}{2} \parens*{\dpd{\mu^{\pi_{t+1}}, c_i} - (b_i + z_{t+1}[i])}$ for each constraint $i \in [m]$.
    \ENDFOR
    \STATE Return $\mu^{\pi_{T+1}}$
\end{algorithmic}
\caption{The \texttt{AL} Method for~\cref{problem:saom_objective_slack}}
\label{alg:alm_lp_slack}
\end{algorithm}

\begin{lemma}\label{lemma:inexact_al_bound_cmdp}
Suppose for each iteration $t \geq 1$, there exists an oracle that returns an pair $(\mu^{\pi_{t+1}}, z_{t+1}) \in \cK \times \R^m$ such that
\begin{equation*}
   \almsaomslackt(\mu^{\pi_{t+1}}, z_{t+1})  
   \geq \max_{(\saom, z) \in \cK \times
   \cZ} 
   \almsaomslackt(\saom, z) - \eps_t.
\end{equation*}
Then~\cref{alg:alm_lp_slack} with $\eps_t = \frac{\sigma}{t^2}$, for any $\sigma > 0$, $\beta > 0$, $\lambda_1 = 0$ converges as: for all $t \geq 1$
\begin{align*}
\dpd{\mu^* - \mu^{\pi_{t+1}}, r}  &\leq \frac{2 \, B}{\sqrt{\beta}} \left(\frac{\sqrt{C}}{\sqrt{t}} + \frac{\sqrt{2 \, \sigma}}{t}\right) + \frac{\sigma}{t^2} , \\
b_i - \dpd{\mu^{\pi_{t+1}}, c_i} &\leq \frac{2}{\sqrt{\beta}} \left(\frac{\sqrt{C}}{\sqrt{t}} + \frac{\sqrt{2 \, \sigma}}{t}\right) \quad \forall i \in [m]
\end{align*}
where $B := M + \sqrt{\sigma \, \frac{\beta \, \pi^2}{3}}$, 
$M := (\min_{i \in [m]} \max_{\saom} (\dpd{\saom, c_i} - b) \, (1 - \gamma))^{-1}$,
$\kappa := \frac{\beta}{8 B^2}$, $\omega := 3$, $C := \frac{4}{\kappa}\left(1+\sqrt{2\,\omega\sigma\kappa(\omega\sigma\kappa+1)}\right) + 4\max\left\{\max_{\saom, z} \almsaomslack(\pi, z, \lambda^*) - \max_{\saom, z}\almsaomslack(\pi, z, \lambda_1),\frac{4}{\kappa}\right\}$.
\end{lemma}
\begin{proof}
First note by~\cref{lemma:dual_variable_ub}, both~\cref{problem:saom_objective} and its slack-variable formulation~\cref{problem:saom_objective_slack} satisfies strong duality, i..e, for $\lambda \in \R^m_+$
\begin{equation*}
   \max_{\saom \in \cK} \min_{\lambda \in \R^m_+ } \dpd{\saom, r} + \sum_{i=1}^m \lambda_i, (\dpd{\saom, c_i} - b_i) = 
   \min_{\lambda \in \R^m_+ }\max_{\saom \in \cK} \dpd{\saom, r} + \sum_{i=1}^m \lambda_i, (\dpd{\saom, c_i} - b_i)
\end{equation*}
and the optimal Lagrange multiplier is bounded:
\begin{equation} \label{eq:lambda_star_ub}
    \norm{\lambda^*} \leq \underbrace{\frac{1}{\min_{i \in [m]} \zeta_i \, (1 - \gamma)}}_{:= M}
\end{equation}
where $\zeta_i := \max_{\pi} V^{\pi}_{c_i}(\rho) - b_i > 0$.

Next, we rewrite~\cref{problem:saom_objective_slack} in standard form in order to match the analysis for the \texttt{AL} analysis for convex, linearly constrained problems presented in~\cref{appendix:convex_proofs}.

Define the change of variables
\begin{align*}
   x &:= \begin{bmatrix}
      \saom  \\
      z
   \end{bmatrix}  \quad \text{ ; } \quad f(x) := -\dpd{\saom, r}  \quad \text{ ; } \quad A := \begin{bmatrix} c_1 & -1 \\ c_2 & -1\\ \vdots  & \vdots \\ c_m  & -1 \end{bmatrix}  \quad \text{ ; }  \quad \cX = \Lambda \times \cZ.
\end{align*}
where for any $\sigma > 0$ and $\beta > 0$, $\cZ := \braces*{ z \in \R^m_+ \mid 0 \leq z_i \leq \bar{z}_i}$ is domain of the non-negative slack variables with $\bar{z}_i := \frac{1}{1 - \gamma} + b_i + \frac{\sqrt{\frac{\sigma \, \beta \, \pi^2}{3}} + M}{\beta}$ for each constraint $i \in [m]$. Under this change of variables,~\cref{problem:saom_objective_slack} coincides with setting described
in~\cref{appendix:convex_proofs}. 
Using~\cref{thm:convex_linear_equality_constraints_alm_rate} and the above mapping, if $\eps_t = \frac{\sigma}{t^2}$ for any constant $\sigma > 0$, $\beta > 0$ and $\lambda_1 = 0$,~\cref{alg:alm_lp_slack} achieves the following convergence:
\begin{align}
\dpd{\mu^* - \mu^{\pi_{T+1}}, r}  &\leq \frac{2 \, B}{\sqrt{\beta}} \left(\frac{\sqrt{C}}{\sqrt{T}} + \frac{\sqrt{2 \, \sigma}}{T}\right) + \frac{\sigma}{T^2} , \\
\norm*{\sum_{i=1}^m \dpd{\mu^{\pi_{T+1}}, c_i} - [z_{T+1}]_i - b_i}  &\leq \frac{2}{\sqrt{\beta}} \left(\frac{\sqrt{C}}{\sqrt{T}} + \frac{\sqrt{2 \, \sigma}}{T}\right) \label{eq:saom_bound_vector}
\end{align}
where $B := 2 \, M + \sqrt{\sigma \, \beta \, \frac{\pi^2}{3}}$, $\kappa := \frac{\beta}{8 B^2}$, $\omega := 3$,  \\ and $C := \frac{4}{\kappa}\left(1+\sqrt{2\,\omega\sigma\kappa(\omega\sigma\kappa+1)}\right) + 4\max\left\{\max_{\saom, z} \bar{F}_{\beta, \lambda^*}(\pi, z) - \max_{\saom, z}\bar{F}_{\beta, \lambda_1}(\pi, z),\frac{4}{\kappa}\right\}$.

To relate~\cref{eq:saom_bound_vector} to an individual constraint $i \in [m]$,
\begin{align*}
   b_i - \dpd{\mu^{\pi_{T+1}}, c_i} &= - \parens*{\dpd{\mu^{\pi_{T+1}}, c_i} - b_i} \\
   &= - \parens*{\dpd{\mu^{\pi_{T+1}}, c_i} - b_i - [z_{T+1}]_i + [z_{T+1}]_i}\tag{Add/Subtract $[z_{T+1}]_i$}\\
   &\leq - \parens*{\dpd{\mu^{\pi_{t+1}}, c_i} - b_i - [z_{T+1}]_i} \tag{Since $[z_{T+1}]_i \geq 0$} \\
   &\leq \abs*{\dpd{\mu^{\pi_{T+1}}, c_i} - b_i - [z_{T+1}]_i} \\
   &\leq \max_{i \in [m]} \abs*{\dpd{\mu^{\pi_{T+1}}, c_i} - b_i - [z_{T+1}]_i} \\
   &\leq \norm*{\sum_{i=1}^m \dpd{\mu^{\pi_{t+1}}, c_i} - [z_{t+1}]_i - b_i} \tag{Since $\supnorm{\cdot} \leq \norm{\cdot}$} \\
\implies b_i - \dpd{\mu^{\pi_{T+1}}, c_i} &\leq \frac{2}{\sqrt{\beta}} \left(\frac{\sqrt{C}}{\sqrt{T}} + \frac{\sqrt{2 \, \sigma}}{T}\right) \quad \forall i \in [m] \tag{Using~\cref{eq:saom_bound_vector}}
\end{align*}
Finally, we note that $z_t \in \cZ$ for all $t \geq 1$.
According to~\cref{lemma:aug_lag_eq_and_ineq_same}, 
for all $t \geq 1$, the optimal slack variable is given by
$\xi_i(\mu^{\pi_{t}}) = \max\braces*{\dpd{\mu^{\pi_{t}}, c_i} - b_i - \frac{[\lambda_t]_i}{\beta}, 0}$. Hence,
\begin{align*}
 z_{t} &\leq \xi_i(\mu^{\pi_{t}}) \\
     &= \max\braces*{\dpd{\mu^{\pi_{t}}, c_i} - b_i - \frac{[\lambda_t]_i}{\beta}, 0}    \\
     &\leq \abs*{\dpd{\mu^{\pi_{t}}, c_i} - b_i - \frac{[\lambda_t]_i}{\beta}} \\
    &\leq \dpd{\mu^{\pi_{t+1}}, c_i} + b_i + \frac{\abs{[\lambda_t]_i}}{\beta} \tag{Using triangle inequality}  \\
    &\leq \frac{1}{1 - \gamma} + b_i + \frac{\abs{[\lambda_t]_i}}{\beta}  \tag{Since $c_i \in [0, 1]$}\\
    &\leq \frac{1}{1 - \gamma} + b_i + \frac{\max_{i \in [m]}\abs{[\lambda_t]_i}}{\beta} \\
    &\leq \frac{1}{1 - \gamma} + b_i + \frac{\norm{\lambda_t}}{\beta} \tag{Since $\supnorm{\cdot} \leq \norm{\cdot}$} \\
    &\leq \frac{1}{1 - \gamma} + b_i + \frac{\norm{\lambda_t - \lambda^*} + \norm{\lambda^*}}{\beta} \tag{Add/Subtract $\lambda^*_i$ and using triangle inequality} \\
    &\leq \frac{1}{1 - \gamma} + b_i + \frac{\sqrt{\frac{\sigma \, \beta \, \pi^2}{3}} + \norm{\lambda^*}}{\beta}  \tag{By~\cref{eq:lk_minus_lstar_bound}, $\norm{\lambda_t - \lambda^*} \leq \sqrt{\frac{\sigma \, \beta \, \pi^2}{3}}$} \\
    &\leq \frac{1}{1 - \gamma} + b_i + \frac{\sqrt{\frac{\sigma \, \beta \, \pi^2}{3}} + M}{\beta}.  \tag{By~\cref{eq:lambda_star_ub}, $\norm{\lambda^*} \leq M$} 
\end{align*}

\end{proof}

\subsection{The Augmented Lagrangian is a General Utility}
Following~\citet{bertsekas2014constrained}, we show that \texttt{AL} the  in~\cref{eq:al_saom_slack} can first be maximized over the slack variables first. 
This equivalence allows the \texttt{AL} subproblem only in terms of $\saom$, implying that it is a concave general utility function. 
The following two lemmas formalize this reduction and its implications for approximate solutions.
\begin{lemma}\label{lemma:aug_lag_eq_and_ineq_same}
For any $\lambda \in \R^m$, $\beta > 0$ and $\saom \in \cK$,
\begin{align*}
   &\max_{z \in \R^m} \almsaomslack(\saom, z)
   = \F(\saom) := \dpd{\saom, r} + \sum_{i=1}^m \frac{\beta}{2} \bracks*{-\parens*{\min\braces*{\dpd{\saom, c_i} - b_i - \frac{\lambda_i}{\beta}, 0}}^2  + \frac{\lambda_i^2}{\beta^2}}
\end{align*}
Furthermore,
\begin{equation*}
   \almsaomslack(\saom, \xi(\saom))  = \F(\saom).
\end{equation*}
where for any constraint $i \in [m]$
\begin{equation*}
   \xi_i(\saom) = \max\braces*{\dpd{\saom, c_i} - b_i -\frac{\lambda_i}{\beta}, 0}.
\end{equation*}
\end{lemma}
\begin{proof}
This proof follows from~\citet[Chapter 3]{bertsekas2014constrained} and is added for completeness.

Fix $\lambda \in \R^m$ and $\saom \in \cK$. 
Since the \texttt{AL} is separable in the slack variables $z_i$, we may maximize each coordinate independently:
\begin{align*}
\max_{z \in \R^m_+} \almsaomslack(\saom, z) &=
\dpd{\saom, r} + \sum_{i=1}^m \lambda_i \parens*{\dpd{\saom, c_i} - b_i - z_i} - \frac{\beta}{2} \sum_{i=1}^m (\dpd{\saom, c_i} - b_i - z_i)^2 \\
&= \dpd{\saom, r} + \sum_{i=1}^m \max_{z_i \in \R^+} \braces*{\lambda_i \, (u_i - z_i) - \frac{\beta}{2} (u_i - z_i)^2} \tag{Let $u_i := \dpd{\saom, c_i} - b_i$}
\end{align*}
For a fixed constraint $i \in [m]$, let
\begin{equation*}
 g_i(z) := \lambda_i \, (u_i - z_i) - \frac{\beta}{2} \, (u_i - z_i)^2.
\end{equation*}
Since $g$ is concave, we can take its derivative and find the global maximum,
\begin{align*}
    z_i^* &= \argmax_{z_i \geq 0} g_i(z) = \max\braces*{u_i - \frac{\lambda_i}{\beta}, 0}. \\ 
    \implies g_i(z_i^*) &= \lambda_i \, (u_i - z_i^*) - \frac{\beta}{2} \parens*{u_i - z_i^*}^2 
\end{align*}
Furthermore, for a fixed constraint $i \in [m]$,
\begin{align*}
   g_i(z^*_i) &= \lambda_i \, \parens*{u_i - \max\braces*{u_i - \frac{\lambda_i}{\beta}, 0}} - \frac{\beta}{2} \parens*{u_i - \max\braces*{u_i - \frac{\lambda_i}{\beta}, 0}}^2 \\
   &= \lambda_i \, \parens*{u_i + \min\braces*{- u_i + \frac{\lambda_i}{\beta}, 0}} - \frac{\beta}{2} \parens*{u_i + \min\braces*{- u_i + \frac{\lambda_i}{\beta}, 0}}^2 \tag{$\max\braces*{a, 0} = -\min\braces*{-a, 0}$}\\
   &= \lambda_i \parens*{\min\braces*{u_i,  \frac{\lambda_i}{\beta}}} - \frac{\beta}{2} \parens*{\min\braces*{u_i, \frac{\lambda_i}{\beta}}}^2 \\
   &= \lambda_i \parens*{\min\braces*{u_i - \frac{\lambda_i}{\beta}, 0} + \frac{\lambda_i}{\beta}} - \frac{\beta}{2} \parens*{\min\braces*{u_i - \frac{\lambda_i}{\beta}, 0} + \frac{\lambda_i}{\beta}}^2 \tag{Add/Subtract $\frac{\lambda_i}{\beta}$} \\
   &= - \frac{\beta}{2} \parens*{\min\braces*{u_i - \frac{\lambda_i}{\beta}, 0} + \frac{\lambda_i}{\beta}}^2 +  \lambda_i \parens*{\min\braces*{u_i - \frac{\lambda_i}{\beta}, 0} + \frac{\lambda_i}{\beta}}  \\
   &= \frac{\beta}{2} \bracks*{- \parens*{\min\braces*{u_i - \frac{\lambda_i}{\beta}, 0} + \frac{\lambda_i}{\beta}}^2 + \frac{2 \, \lambda_i}{\beta} \min\braces*{u_i - \frac{\lambda_i}{\beta}, 0} + \frac{2 \, \lambda_i^2}{\beta^2}} \\
   &= \frac{\beta}{2} \bracks*{-\parens*{\min\braces*{u_i - \frac{\lambda_i}{\beta}, 0}}^2  + \frac{\lambda_i^2}{\beta^2}}  \tag{Expanding the square}
\end{align*}
\begin{equation}\label{eq:bertsekas_chapter_3_algebra}
   \implies \lambda_i \, (\dpd{\saom, c_i} - b_i - z_i^*) + \frac{\beta}{2} \parens*{\dpd{\saom, c_i} - b_i - z_i^*}^2 
   = \frac{\beta}{2} \bracks*{-\parens*{\min\braces*{\dpd{\saom, c_i} - b_i -  \frac{\lambda_i}{\beta}, 0}}^2  + \frac{\lambda_i^2}{\beta^2}}
\end{equation}
Since $z^*$ is the optimal slack variable for an arbitrary $\lambda$ and $x$,
\begin{align*}
   &\max_{z \in \R^m} \almsaomslack(\saom, z) = \almsaomslack(\saom, z^*) = 
   \dpd{\saom, r} + \sum_{i=1}^m \lambda_i \, (\dpd{\saom, c_i} - b_i - z_i^*) + \frac{\beta}{2} \parens*{\dpd{\saom, c_i} - b_i - z_i^*}^2 
   \numberthis \label{eq:aug_lag_slack_algebra}\\
   &= \dpd{\saom, r} + \sum_{i=1}^m \frac{\beta}{2} \bracks*{-\parens*{\min\braces*{\dpd{\saom, c_i} - b_i - \frac{\lambda_i}{\beta}, 0}}^2  + \frac{\lambda_i^2}{\beta^2}}  \tag{Using~\cref{eq:bertsekas_chapter_3_algebra}} \\
   &= \F(\saom).
\end{align*}
Finally, defining $\xi(\saom) \in \R^m$ such that for each $i \in [m]$
\begin{equation*}
    \xi_i(\saom) = \max\braces*{\dpd{\saom, c_i} - b_i - \frac{\lambda_i}{\beta}, 0},
\end{equation*}
and substituting into $\almsaomslack(\saom, \xi(\saom))$ reproduces the same expression, implying
\begin{align*}
\almsaomslack(\saom, \xi(\saom)) = \F(\saom).
\end{align*}
\end{proof}

\begin{lemma}\label{lemma:solving_slack_same_as_inequality}
For a target $\eps \in (0, 1)$, any $\lambda \in \R^m$ and $\beta > 0$, suppose there exists $\mu^{\pi'} \in \cK$ such that
\begin{equation*}
    \F(\mu^{\pi'}) \geq \max_{\saom \in \cK} \F(\saom) - \eps.
\end{equation*}
Then the pair $(\mu^{\pi'}, \xi(\mu^{\pi'}))$ satisfies
\begin{equation*}
    \almsaomslack(\mu^{\pi'}, \xi(\mu^{\pi'})) \geq \max_{(\saom, z) \in \cK \times \R^m_{+}}\almsaomslack(\saom, z) - \eps.
\end{equation*}
where $\xi(\mu^{\pi'}) \in \R^m$ such that for each constraint $i \in [m]$, $\xi_i(\mu^{\pi'}) = \max\braces*{\dpd{\mu^{\pi'}, c_i} - b_i -\frac{\lambda_i}{\beta}, 0}$
\end{lemma}
\begin{proof}
For a target $\eps \in (0, 1)$, any $\lambda \in \R^m$ and $\beta > 0$, suppose there exists $\mu^{\pi'} \in \cK$ such that
\begin{align*}
    \F(\mu^{\pi'}) &\geq \max_{\saom \in \cK} \F(\saom) - \eps \\
    &= \max_{(\saom, z) \in \cK \times \R^m} \almsaomslack(\saom, z) - \eps \tag{Using~\cref{lemma:aug_lag_eq_and_ineq_same}}
\end{align*}
Furthermore by~\cref{lemma:aug_lag_eq_and_ineq_same}, 
\begin{equation*}
   \F(\mu^{\pi'}) = \almsaomslack(\mu^{\pi'}, \xi(\mu^{\pi'}))
\end{equation*}
where
\begin{equation*}
   \xi_i(\mu^{\pi'}) = \max\braces*{\dpd{\mu^{\pi'}, c_i} - b_i -\frac{\lambda_i}{\beta}, 0}.
\end{equation*}
Hence,
\begin{equation*}
    \almsaomslack(\mu^{\pi'}, \xi(\mu^{\pi'})) \geq \max_{(\saom, z) \in \cK \times \R^m_{+}}\almsaomslack(\saom, z) - \eps.
\end{equation*}
\end{proof}

%% file: Appendix/A3_ALM_proofs.tex
\section{Proofs and Technical Details for the Augmented Lagrangian Method}\label{appendix:convex_proofs}
\subsection{Equality Constraint Problems}
To simplify our analysis and align with classical optimization literature, we consider problems in standard form: 
\begin{problem}[Equality Constraint Problem]\label{problem:convex_equality}
\begin{align}
\begin{array}{cl}
\displaystyle \min_{x\in \cX} & \displaystyle f(x) \quad \mathrm{subject\ to} \quad Ax=b,\\
\end{array}
\label{eq:objective}
\end{align}
\end{problem}
where $f$ is a convex function, $\cX$ is a closed and bounded set, $A\in \mathbb{R}^{m\times n}$, $b\in\mathbb{R}^{m}$.
Let $x^*$ denote an optimal solution to~\cref{problem:convex_equality}.
The augmented Lagrangian (\texttt{AL}) function of~\cref{problem:convex_equality} is defined as
\begin{align}\label{eq:al_convex_equality}
\cL^\beta(x,\lambda) := f(x)+\langle \lambda, Ax-b \rangle + \frac{\beta}{2} \|Ax-b\|^2,
\end{align}
where $\lambda\in\mathbb{R}^m$ is the Lagrange multiplier associated with the linear constraint, and $\beta>0$ is the penalty parameter. 
The \texttt{AL} dual of~\cref{problem:convex_equality} is
\begin{equation}
\max_{\lambda\in\mathbb{R}^{m}}\ d(\lambda)\,, \quad \text{where}  \, \quad d(\lambda):=\min_{x\in \cX} \cL^\beta(x,\lambda).
\end{equation}
Given an Lagrange multiplier $\lk$ at iteration $t$, the \texttt{AL} method for solving~\cref{problem:convex_equality} updates the primal and dual variables via
\begin{align}
x(\lk)=\displaystyle\arg\min_{x \in \cX} \cL^\beta(x,\lk) \quad \text{;} \quad \lkk= \lk + \frac{\beta}{2}\left(Ax(\lk)-b\right).
\label{eq:update}
\end{align}
Since computing the primal update can be computationally infeasible, we consider the inexact variant where finding an approximate solution up to a target $\eps_t > 0$ is sufficient (shown in~\cref{alg:cgal}).
\begin{algorithm}[H]
\begin{algorithmic}[1]
    \STATE \textbf{Input}:
    $x_1 \in \cX$ (primal variable) $\lambda_1 \in \R^m$ (dual variable), $\beta > 0$ (constraint penalty), and the non-negative sequence $\braces*{\eps_t}$.
    \FOR{$t = 1, 2, \dots, T$ }
        \STATE Formulate augmented Lagrangian $\cL^\beta(x, \lambda_t) := f(x) + \dpd{\lambda_t, Ax - b} + \frac{\beta}{2} \, \normsq{Ax - b}$
    \STATE Find the approximate solution $x_{t+1}$ of the \texttt{AL} subproblem such that \[\cL^\beta(x_{t+1}, \lambda_t) - \min_{x \in \cX} \cL^\beta(x, \lambda_t) \leq \eps_t \]
        \STATE $\lambda_{t+1} = \lambda_t + \frac{\beta}{2} \, (Ax_{t+1} - b)$
        \ENDFOR
    \STATE Return $x_{T + 1}$
\end{algorithmic}
\caption{Inexact Augmented Lagrangian Method for Equality Constraints (\cref{problem:convex_equality})}
\label{alg:cgal}
\end{algorithm}
Before proceeding with the main theorem, we require a couple of technical lemmas.
The following notation and relations will be helpful in the proofs:
\begin{align}
x(\lk):=&\arg\min_{x \in \cX} \cL^\beta(x, \lk),~t=1,2,\ldots,\label{xlambda}\\
\nabla d(\lk):=&Ax(\lk)-b,~t=1,2,\ldots,\nonumber\\
\bar d(\lk):=& \cL^{\beta}(\xkk,\lk),~t=1,2,\ldots,\label{bard}\\
\nabla \bar d(\lk):=&A\xkk-b,~t=1,2,\ldots,\label{bardgradient}
\end{align}
As stated in~\cref{alg:cgal}, at iteration $t$, we ensure that the following error condition is satisfied 
\begin{align}
\cL^\beta(\xkk, \lk) - \cL^\beta(x(\lk), \lk) = \bar{d}(\lk)- d(\lk) \leq \epsk.
\label{eq:error-metric}    
\end{align}
\begin{lemma}
In~\cref{alg:cgal}, for all iterations $t \geq 1$,
\begin{align}
\normsq{\nabla d(\lk)-\nabla\bar d(\lk)} \leq \frac{2 \epsk}{\beta}.
\end{align}
\label{lemma:diff-grad}    
\end{lemma}
\begin{proof}
\begin{align}
\normsq{\nabla d(\lk)-\nabla\bar d(\lk)} &= \normsq{A x(\lk) - b - (A \xkk - b)} = \normsq{A x(\lk) - A \xkk}
\label{eq:diff-grad-1}
\end{align}
Define 
\begin{align*}
u^* = A x(\lk) \quad \text{;} \quad \hat u &= A \xkk \quad \text{;} \quad h(u) = \langle \lk,\,u - b\rangle + \frac{\beta}{2} \normsq{u - b} \\
\implies \cL^\beta(x, \lk) & = f(x) + h(Ax)
\end{align*}
Since $h$ is $\beta$-strongly convex in $u$, we have
\begin{align*}
h(\hat u) \;\ge\; h(u^*) +\langle\nabla h(u^*),\,\hat u - u^*\rangle +\frac{\beta}{2} \normsq{\hat u - u^*}.    
\end{align*}
By convexity of $f$, for any $x, y$
\[
f(y)\;\ge\;f(x)+\langle\nabla f(x),\, y - x \rangle.
\]
And the first‐order optimality condition at $x(\lk)$ gives, for all $x \in \cX$, 
\begin{align}
& \langle\nabla \cL^\beta(x(\lk), \lk),\, x - x(\lk) \rangle \geq 0 \implies \langle\nabla f(x(\lk)) + A^T\nabla h(u^*),\,x - x(\lk)\rangle \geq 0 \nonumber \\
\implies & \langle\nabla \cL^\beta(x(\lk), \lk),\, \xkk - x(\lk) \rangle \geq 0 \implies \langle\nabla f(x(\lk)) + A^T\nabla h(u^*),\, \xkk - x(\lk)\rangle \geq 0 \tag{Setting $x = \xkk$} \\
\intertext{Using the definition of $u^*$ and $\hat{u}$}
\implies & \langle\nabla f(x(\lk)),\, \xkk - x(\lk)\rangle + \langle\nabla h(u^*),\, \hat{u} - u^*\rangle \geq 0 \label{eq:opt-condition}
\end{align}
\begin{align*}
\cL^\beta(\xkk, \lk)-\cL^\beta(x(\lk), \lk) &= \left[f(\xkk)-f(x(\lk))\right] + \left[h(\hat u)-h(u^*)\right]\\
& \geq \langle\nabla f(x(\lk)),\, \xkk - x(\lk)\rangle + \left[\langle\nabla h(u^*),\,\hat u - u^*\rangle + \frac{\beta}{2} \normsq{\hat u - u^*} \right] \tag{By convexity of $f$ and strong-convexity of $h$} \\
&\ge \underbrace{\langle\nabla f(x(\lk)),\,\xkk - x(\lk)\rangle +\langle\nabla h(u^*),\,\hat u - u^*\rangle}_{\geq 0 \text{ from~\cref{eq:opt-condition}}}
+ \frac{\beta}{2} \normsq{\hat{u} - u^*} \\
&\geq \frac{\beta}{2} \normsq{A\xkk - A x(\lk)} \tag{By definition of $u^*$ and $\hat{u}$}
\end{align*}
Hence, since $\cL^\beta(\xkk, \lk)-\cL^\beta(x(\lk), \lk) \leq \epsk$, we conclude
\begin{align*}
\normsq{A\xkk - A x(\lk)} \leq \frac{2\epsk}{\beta}.    
\end{align*}
Combining with~\cref{eq:diff-grad-1} finishes the proof. 
\end{proof}

\begin{lemma}
In~\cref{alg:cgal}, for all $t \geq 1$, the following two inequalities hold:
\begin{align}
  d(\lambda) \leq \bar d(\lambda') + \left\langle\nabla\bar d(\lambda'),\lambda-\lambda'\right\rangle,~\forall~\lambda,\lambda',
  \label{eq:d-error-ub}
\end{align}
and
\begin{align}
d(\lkk)\geq\bar d(\lk)+\frac{\beta}{4} \left\|\nabla \bar d(\lk)\right\|^2 - 2 \, \epsk.
\label{eq:d-error-lb}
\end{align}
\label{lemma:d-error}
\end{lemma}
\begin{proof} 
We first show \cref{eq:d-error-ub}. By the definitions of $d(\lambda)$ and $\cL^\beta(x, \lambda)$, we have
\begin{align}
d(\lambda)=\min_{x\in \cX}\left\{\cL^\beta(x,\lambda)\right\}\leq \cL^\beta(x_{\lambda'},\lambda)=\cL^\beta(x_{\lambda'},\lambda') + \langle\lambda-\lambda',Ax_{\lambda'}-b\rangle, \tag{By definition of $\LCal_\beta(x, \lambda)$ and $\LCal_\beta(x, \lambda')$}
\end{align}
where $x_{\lambda'}$ is defined such that $\bar d(\lambda')=\cL^\beta(x_{\lambda'},\lambda')$. This, together with the definition of $\nabla \bar d(\lambda') = A x_\lambda' - b$ yields \cref{eq:d-error-ub}.

\vspace{2ex}

We now prove \cref{eq:d-error-lb}. According to~\cref{lemma:alm_dual_smooth}, $d$ is $\frac{1}{\beta}$-smooth,
\begin{align*}
   d(\lkk) &\geq d(\lk) + \langle \nabla  d(\lk), \lkk - \lk \rangle - \frac{1}{2 \, \beta} \, \normsq{\lkk - \lk} \\
   &= d(\lk) + \langle \nabla  d(\lk), \frac{\beta}{2} \, \nabla \bar d(\lk) \rangle - \frac{\beta}{8} \, \normsq{\nabla \bar d(\lk)} \tag{Using the dual update $\lkk = \lk + \frac{\beta}{2} \, \nabla \bar d(\lk)$} \\
   &= \bar d(\lk) + \underbrace{d(\lk) - \bar d(\lk)}_{\geq -\eps_t \text{ from~\cref{eq:error-metric}}} + \langle \nabla  d(\lk), \frac{\beta}{2} \, \nabla \bar d(\lk) \rangle - \frac{\beta}{8} \, \normsq{\nabla \bar d(\lk)} \tag{Add/Subtract $\bar{d}(\lk)$} \\
   &\geq \bar d(\lk) - \eps_t + \langle \nabla  d(\lk) + \nabla \bar d(\lk) - \nabla \bar d(\lk), \frac{\beta}{2} \, \nabla \bar d(\lk) \rangle - \frac{\beta}{8} \, \normsq{\nabla \bar d(\lk)} \tag{Add/Subtract $\langle \nabla \bar{d}(\lk), \nabla \frac{\beta}{2} \, \bar{d}(\lk) \rangle$} \\
   &= \bar d(\lk) - \eps_t  + \frac{3\,\beta}{8} \, \normsq{\nabla \bar d(\lk)} - \left[\langle - \frac{\beta}{2} \, \nabla \bar d(\lk), \nabla  d(\lk) - \nabla \bar d(\lk) \rangle \right] \\
   &\geq \bar d(\lk) - \eps_t  + \frac{3 \, \beta}{8} \, \normsq{\nabla \bar d(\lk)} - \frac{\beta}{8} \, \normsq{\nabla \bar d(\lk)} - \frac{\beta}{2} \, \normsq{\nabla  d(\lk) - \nabla \bar d(\lk)}   \tag{Using Young's Inequality, $\langle a, b \rangle \leq \frac{\normsq{a}}{2\,\beta}  + \frac{\beta \, \normsq{b}}{2}$} \\
\implies d(\lkk) &\geq \bar d(\lk) - 2 \, \eps_k  + \frac{\beta}{4} \, \normsq{\nabla \bar d(\lk)}. \tag{Using~\cref{lemma:diff-grad}}
\end{align*}
\end{proof}

\begin{lemma}
For all $t \geq 1$, 
\begin{align}
\norm{\lk - \lstar} \leq \sqrt{\normsq{\lambda_1- \lstar} + 2 \, \beta \, \sum_{i = 1}^{t-1} \epsilon_i }
\end{align}
\label{lemma:bounded-dual-iterates}    
\end{lemma}
\begin{proof}
\begin{align*}
\normsq{\lkk - \lstar} &= \normsq{\lk - \frac{\beta}{2} \nabla \bar{d}(\lk) - \lstar} = \normsq{\lk - \lstar} + \beta \langle \nabla \bar{d}(\lk), \lk - \lstar \rangle + \frac{\beta^2}{4} \normsq{\nabla \bar{d}(\lk)} \tag{Using the dual update $\lkk = \lk + \frac{\beta}{2} \, \nabla \bar d(\lk)$}\\
& \leq \normsq{\lk - \lstar} + \beta \, [\bar{d}(\lk)- d(\lstar)] + \frac{\beta^2}{4} \normsq{\nabla \bar{d}(\lk)} \tag{Using \cref{eq:d-error-ub} in~\cref{lemma:d-error} with $\lambda = \lambda^*$ and $\mu = \lk$} \\
& = \normsq{\lk - \lstar} + \underbrace{\beta \, [d(\lkk)- d(\lstar)]}_{\leq 0 \text{ since $\lambda^* := \argmax d(\lambda)$}} + \beta \, [\bar{d}(\lk) - d(\lkk)] + \frac{\beta^2}{4} \, \normsq{\nabla \bar{d}(\lk)} \tag{Add/Subtract $d(\lkk)$} \\
& \leq \normsq{\lk - \lstar} + \beta \, [\bar{d}(\lk) - d(\lkk)] + \beta [d(\lkk) - \bar d(\lk) +  2 \, \eps_k] \tag{Using~\cref{eq:d-error-lb} in~\cref{lemma:d-error} to simplify the last term} \\
& \leq  \normsq{\lk - \lstar} + 2 \, \beta \, \epsk
\intertext{Hence, by recursing, for all $t \geq 1$,}
\implies \normsq{\lk - \lstar} & \leq \normsq{\lambda_1- \lstar} + 2 \, \beta \, \sum_{i = 1}^{t-1} \epsilon_i 
\end{align*}    
Using that $\sqrt{a+b} \leq \sqrt{a} + \sqrt{b}$ completes the proof. 
\end{proof}

\begin{theorem}\label{thm:convex_linear_equality_constraints_alm_rate}
If $\epsk = \frac{\sigma}{t^{2}}$ with any constant $\sigma>0$ and $\lambda_1 = 0$,~\cref{alg:cgal}  achieves the following convergence: for all $t \geq 1$,
\begin{align}
d(\lambda^*) - d(\lk) & \leq \frac{C}{t} \\ 
\norm{Ax_{t+1}-b}  &\leq \frac{2}{\sqrt{\beta}} \left(\frac{\sqrt{C}}{\sqrt{t}} + \frac{\sqrt{2 \, \sigma}}{t}\right)  \\ 
f(x_{t+1}) - f(x^*)  &\leq \frac{2(B + \norm{\lambda^*})}{\sqrt{\beta}} \left(\frac{\sqrt{C}}{\sqrt{t}} + \frac{\sqrt{2 \, \sigma}}{t}\right) + \frac{\sigma}{t^2} ,
\end{align}
where $B := \norm{\lambda^*} + \sqrt{\sigma \, \beta \, \frac{\pi^2}{3}}$, $\kappa := \frac{\beta}{8 B^2}$, $\omega := 3$,  and $C := \frac{4}{\kappa}\left(1+\sqrt{2\,\omega\sigma\kappa(\omega\sigma\kappa+1)}\right) + 4\max\left\{d(\lambda^*) - d(\lambda_1),\frac{4}{\kappa}\right\}$.
\label{thm:convergence-rate}   
\end{theorem}
\begin{proof}
Using~\cref{eq:d-error-lb} in~\cref{lemma:d-error}, 
\begin{align*}
d(\lkk) &\geq \bar d(\lk) + \frac{\beta}{4} \, \normsq{\nabla \bar d(\lk)} - 2 \eps_k  \\
\intertext{Since by definition, $\bar{d}(\lk) \geq d(\lk)$,}
& \geq d(\lk) + \frac{\beta}{4} \normsq{\nabla \bar{d}(\lk)} - 2 \epsk \\
\intertext{Since $\normsq{\nabla d(\lk)} = \normsq{\nabla d(\lk) - \nabla \bar{d}(\lk) + \nabla \bar{d}(\lk)} \leq 2 \, \normsq{\nabla d(\lk) - \nabla\bar{d}(\lk)} + 2 \, \normsq{\nabla\bar{d}(\lk)}$. Hence, $\normsq{\nabla\bar{d}(\lk)} \geq \frac{1}{2} \normsq{\nabla d(\lk)} - \normsq{\nabla d(\lk) - \nabla \bar{d}(\lk)}$. Hence,}
& \geq d(\lk)+ \frac{\beta}{8} \normsq{\nabla d(\lk)} - \frac{\beta}{4} \, \normsq{\nabla d(\lk) - \nabla \bar{d}(\lk)} - 2 \epsk \\
\implies d(\lambda^*) - d(\lkk) & \leq d(\lambda^*) - d(\lk) - \frac{\beta}{8} \, \normsq{\nabla d(\lk)} + \frac{\beta}{4} \, \normsq{\nabla d(\lk) - \nabla \bar{d}(\lk)} + 2\epsk \\
& \leq \underbrace{d(\lambda^*) - d(\lk)}_{:= \delta_t} - \frac{\beta}{8} \, \normsq{\nabla d(\lk)} + 3 \epsk \tag{Using~\cref{lemma:diff-grad}} \\
\implies \delta_{t+1} & \leq \delta_t - \frac{\beta}{8} \, \normsq{\nabla d(\lk)} + 3 \epsk
\end{align*}
Furthermore,
\begin{align*}
    \norm{\lk - \lstar} &\leq \sqrt{\normsq{\lambda_1- \lstar} + 2 \, \beta \, \sum_{i = 1}^{t-1} \epsilon_i } \tag{Using~\cref{lemma:bounded-dual-iterates}} \\
    &\leq \left[\norm{\lambda^*} + \sqrt{2 \, \sigma \, \beta \, \sum_{i = 1}^{t-1} \frac{1}{i^2}} \right] \tag{Using that $\lambda_1 = 0$,  $\sqrt{a+b} \leq \sqrt{a} + \sqrt{b}$ and $\epsk = \frac{\sigma}{t^2}$} \\
   &\leq \underbrace{\left[\norm{\lambda^*} + \sqrt{\frac{\sigma \, \beta \, \pi^2}{3}} \right]}_{:= B} \tag{$\sum_{i = 1}^{t-1} \frac{1}{i^2} \leq \sum_{i = 1}^{\infty} \frac{1}{i^2} = \frac{\pi^2}{6}$} \\ 
   \implies \norm{\lk - \lstar} &\leq B. \numberthis \label{eq:lk_minus_lstar_bound}
\end{align*}
Since $d(\lambda)$ is concave in $\lambda$, 
\begin{align*}
d(\lstar) - d(\lk) & \leq \langle \nabla d(\lk), \lstar - \lk \rangle \leq \norm{\nabla d(\lk)} \, \norm{\lk - \lstar} \tag{Cauchy Schwarz} \\
& \leq \norm{\nabla d(\lk)} \, B \tag{Using~\cref{eq:lk_minus_lstar_bound}} \\
\implies \norm{\nabla d(\lk)} & \geq \frac{d(\lstar) - d(\lk)}{B} = \frac{\delta_t}{B}
\end{align*}
Combining the above relations, 
\begin{align*}
\delta_{t+1} & \leq \delta_t - \frac{\beta}{8 \, B^2} \, \delta_t^2 + 3 \epsk     
\end{align*}
Using~\cref{lemma:delta-bound}, we can now conclude that for $\eps_t = \frac{\sigma}{t^2}$, $\kappa = \frac{\beta}{8 B^2}$ and $\omega = 3$, if $C := \frac{4}{\kappa}\left(1+\sqrt{2\,\omega\sigma\kappa(\omega\sigma\kappa+1)}\right) + 4\max\left\{\delta_1,\frac{4}{\kappa}\right\}$, for all $t \geq 1$,
\begin{align*}
\delta_t \leq \frac{C}{t}    
\end{align*}
Now we will use the above result to bound the constraint violation and sub-optimality. Using the above relations, we know that, 
\begin{align*}
d(\lkk) & \geq d(\lk)+ \frac{\beta}{4} \normsq{\nabla \bar{d}(\lk)} - 2 \epsk \\
\end{align*}
Since $\nabla \bar{d}(\lk) = A \xkk - b$, 
\begin{align*}
\left\|Ax_{t+1}-b\right\|^2 &\leq  \frac{4}{{\beta}}\left(d(\lambda_{t+1})- d(\lambda_t)+ 2 \, \epsk \right)  \\
&\leq \ \frac{4}{{\beta}}\left(d(\lambda^*)-d(\lk)+ 2 \, \epsk\right) \\ 
&= \frac{4}{{\beta}}\left(\delta_t+ 2 \, \epsk \right) \\ 
&\leq \frac{4}{\beta} \left(\frac{C}{t} + \frac{2 \, \sigma}{t^2}\right) \\
\implies \norm{Ax_{t+1}-b} & \leq \frac{2}{\sqrt{\beta}} \left(\frac{\sqrt{C}}{\sqrt{t}} + \frac{\sqrt{2 \, \sigma}}{t}\right) 
\end{align*}
We will now bound the sub-optimality. 
\begin{align*}
\cL^\beta(\xkk,\lk) &= \cL^\beta(\xkk,\lk) - \cL^\beta(x(\lk),\lk) + \cL^\beta(x(\lk),\lk) \leq \cL^\beta(x(\lk),\lk) + \epsk \\
& = d(\lk) + \epsk \tag{By definition of $d(\lk)$} \\
& \leq d(\lambda^*) + \epsk \tag{Since $\lambda^* = \argmax_\lambda d(\lambda)$} \\
& = f(x^*) + \epsk \tag{Strong duality}
\end{align*}
\begin{align*}
\implies f(\xkk) + \langle \lk, A \xkk - b \rangle + \frac{\beta}{2} \normsq{A \xkk - b} & \leq f(x^*) + \epsk \tag{By definition of $\cL^\beta(\xkk, \lk)$}
\end{align*}
\begin{align*}
\implies f(\xkk) & \leq f(x^*) + \norm{\lk} \, \norm{A \xkk - b} + \epsk \tag{Cauchy Schwarz} \\
& \leq f(x^*) + (\norm{\lk - \lambda^*} + \norm{\lambda^*}) \, \norm{A \xkk - b} + \epsk \tag{Triangle inequality} \\
& \leq f(x^*) + (B + \norm{\lambda^*}) \, \norm{A \xkk - b} + \epsk \tag{Using~\cref{eq:lk_minus_lstar_bound}} \\
\implies f(\xkk) - f(x^*) & \leq \frac{2 \, (B + \norm{\lambda^*})}{\sqrt{\beta}} \left(\frac{\sqrt{C}}{\sqrt{t}} + \frac{\sqrt{2 \sigma}}{t}\right) + \frac{\sigma}{t^2} \tag{Using the bound on the constraint violation}
\end{align*}
\end{proof}

\subsection{Additional Lemmas}
\begin{lemma}
For $t \geq 1$, if $\delta_t \geq 0$, a constants $\kappa \geq 0$, $\omega \geq 0$ and $\epsk = \frac{\sigma}{t^{2}}$ for $\sigma > 0$, if 
\begin{align*}
\delta_{t+1} & \leq \delta_t - \kappa \, \delta_t^2 + \omega \epsk \,,    
\end{align*}
then, for all $t \geq 1$, 
\begin{align}
\delta_{t} \leq \frac{C}{t} \text{ where } C := \frac{4}{\kappa}\left(1+\sqrt{2\,\omega\sigma\kappa(\omega\sigma\kappa+1)}\right)
+ 4\max\left\{\delta_1,\frac{4}{\kappa}\right\}.
\label{deltaeta-global}
\end{align}
\label{lemma:delta-bound}
\end{lemma}
\begin{proof}
We will follow the proof in~\citet[Theorem 4.3]{liu2019nonergodic} and do the proof by induction. Clearly, the inequality \eqref{deltaeta-global} holds for $t=1$. Next, we assume that \eqref{deltaeta-global} holds for some $t\geq 1$, and show it is also true for $t+1$. We use the contrapositive argument. Assume that \eqref{deltaeta-global} does not hold for $t+1$, i.e.,
\begin{align}
\delta_{t+1} > \frac{C}{t+1},
\label{eq:nothold}
\end{align}
In this case, we define $z^* :=\frac{C}{4}-\frac{1}{\kappa}>0$ and will subsequently prove that
\begin{align*}
&\frac{\kappa}{2} \left(\delta_{t+1}-\frac{z^*}{(t+1)}\right)^2 + \left(\delta_{t+1}-\frac{z^*}{(t+1)}\right) \leq \left(\delta_{t}-\frac{z^*}{t}\right) \\ 
&\Leftrightarrow P(z^*)  := \frac{\kappa}{2} \left(\delta_{t+1}-\frac{z^*}{(t+1)}\right)^2 + \left(\delta_{t+1}-\frac{z^*}{(t+1)}\right)- \left(\delta_{t}-\frac{z^*}{t}\right)\leq0.
\end{align*}
If the $P(z^*) \leq 0$, then, it follows from Lemma \ref{prop:1} that
\begin{align}
\label{eq:final}
\delta_{t+1}-\frac{z^*}{(t+1)} \leq \frac{\max\left\{\delta_1,\frac{4}{\kappa}\right\}}{t+1} \implies \delta_{t+1}\leq \frac{z^* + \max\left\{\delta_1,\frac{4}{\kappa}\right\}}{(t+1)} < \frac{C}{(t+1)}.
\end{align}
Clearly, \eqref{eq:final} contradicts \eqref{eq:nothold}, which implies that \eqref{deltaeta-global} is true and completes the inductive step. 

Now, we will prove that if~\cref{eq:nothold} holds, then, $P(z^*) \leq 0$. For this, we consider the following quadratic function $Q(z)$ with respect to $z:$
\begin{eqnarray*}\label{eq:quad}
Q(z) = \frac{\kappa}{2}z^2 - \left(\frac{\kappa C}{4}-1\right)z + \omega\left(\omega \kappa \sigma +1\right)\sigma.
\end{eqnarray*}
It can be verified that for the minimizer of $Q(z)$ is $z^*$, and that for $C$ defined in the theorem statement, 
\begin{equation}\label{Qz*}
    Q(z^*)\leq 0.
\end{equation}
Moreover, for any $t\geq 1$, we know that, 
\begin{align*}
\delta_{t+1} & \leq \delta_t - \kappa \, \delta_t^2 + \omega \epsk \implies \delta_{t+1} \leq \delta_t + \omega  \epsk \Longrightarrow 
\frac{1}{2}\delta_{t+1}^2 \leq \delta_t^2 + \omega^2 \epsk^2\Longrightarrow-\delta_t^2 \leq -\frac{1}{2}\delta_{t+1}^2 + \omega^2 \sigma\epsk. \tag{Since $\epsk \leq \frac{\sigma}{t^2}$}    
\end{align*}

Combining the last inequality in the above with the condition in the statement yields
\begin{align}
\label{eq:ineqnoise2}
\delta_{t+1} & \leq \delta_t - \kappa \, \delta_t^2 + \omega \epsk \implies \delta_{t+1} \leq \delta_t + \kappa \, \left[-\frac{1}{2}\delta_{t+1}^2 + \omega^2 \sigma\epsk \right] + \omega \epsk \\
\implies \frac{\kappa}{2}\delta_{t+1}^2 + \delta_{t+1} & \leq \delta_t  + \omega\left(\omega\kappa\sigma + 1\right) \epsk,
\end{align}
which further implies 
\begin{align*}
P(z^*)&=\frac{\kappa}{2} \delta_{t+1}^2 + \delta_{t+1} - \delta_t - \frac{\kappa\delta_{t+1}z^*}{(t+1)} + \frac{\kappa(z^*)^2}{2(t+1)^{2}} -\frac{z^*}{(t+1)} + \frac{z^*}{t}\nonumber\\
& \leq \omega\left(\omega\kappa\sigma + 1\right) \epsk- \frac{\kappa\delta_{t+1}z^*}{(t+1)} + \frac{\kappa(z^*)^2}{2(t+1)^{2}} -\frac{z^*}{(t+1)} + \frac{z^*}{t}. %
\end{align*}
Since $\epsk\leq \frac{\sigma}{t^{2}}$ and under the assumption that $\delta_{t+1} > \frac{C}{(t+1)}$, using the facts $(t+1)^{2}\leq 4t^{2}$ for all $t\geq 1$, we get
\begin{align*}
&\leq \frac{\omega \left(\omega \, \kappa \, \sigma + 1\right) \sigma}{t^{2}} - \frac{\kappa C}{4t^{2}}z^*
+ \frac{\kappa}{2t^{2}}(z^*)^2 + \frac{1}{t^{2}}z^*\nonumber\\
&=\frac{Q(z^*)}{t^{2}},
\end{align*}
which, together with \eqref{Qz*} yields $P(z^*)\leq 0.$ This completes the proof.
\end{proof}

\begin{lemma}[Lemma 4.2 in~\citet{liu2019nonergodic}]
\label{prop:1}
Suppose the nonnegative sequence $\left\{\delta_t\right\}$ satisfies
\begin{eqnarray}\label{eq:ineqnonoise2}
\frac{E}{2} \delta_{t+1}^2 + \delta_{t+1} \leq \delta_t,~t=1,2,\ldots,
\end{eqnarray}where $E>0$ is a constant.
Then, we have
\begin{eqnarray}\label{eq:complexity1}
\delta_{t} \leq \frac{\max\left\{\delta_1,\frac{4}{E}\right\}}{t},~t=1,2,\ldots.
\end{eqnarray}
\end{lemma}
\begin{proof}
{Again we prove this by induction.} Clearly, the inequality \eqref{eq:complexity1} is true for $t=1$. Next, assuming \eqref{eq:complexity1} is true for some $t\geq 1$, we show it is also true for $t+1$. In fact, we have
\begin{align*}
\delta_{t+1} \leq \frac{-1+\sqrt{1+2E\delta_t}}{E} \leq &\  \frac{-1+\sqrt{1+2E\frac{\max\left\{\delta_1,\frac{4}{E}\right\}}{t}}}{E}\\=&\ \frac{2\max\left\{\delta_1,\frac{4}{E}\right\}}{t+\sqrt{t^2+2E\max\left\{\delta_1,\frac{4}{E}\right\}t}}\\
\leq&\  \frac{\max\left\{\delta_1,\frac{4}{E}\right\}}{t+1},
\end{align*}
where the first inequality is due to the inequality \eqref{eq:ineqnonoise2}, and the second inequality is due to the assumption that \eqref{eq:complexity1} holds for $t$.
\end{proof}

\begin{lemma}[Theorem 1 in~\citet{li2025smoothness}]\label{lemma:alm_dual_smooth}
Assuming $f$ is closed proper convex, the dual function $d(\lambda) = \max_{x \in \cX} \cL^\beta(x, \lambda)$ with penalty parameter $\beta > 0$ is $\frac{1}{\beta}$-smooth.
\end{lemma}

%% file: Appendix/A4_general_utility_proofs.tex
\section{Proofs for Projected Q-Ascent on General Utilities}\label{appendix:gu_proofs}
In the Reinforcement Learning with General Utilities (\texttt{$\RLGU$}) formulation, we consider the optimization problem
\begin{equation}\label{eq:gu_objective}
    \textstyle \max_{\pi} G(\pi) \coloneq \cF(\saom) \,,
\end{equation}
where the general utility $\cF: \cK \to \R$ is a concave function over state-action occupancy measures $\saom$. 
Although $\cF(\saom)$ is concave, optimizing over occupancy measures is impractical when the state or action space is large.
Hence, we aim to optimize the induced objective $G(\pi)$ in the policy space. 
Following prior work~\citep{zhang2020variational,kumar2024policy,barakat2023reinforcement,barakat2024sample}, we will apply PG methods to optimize $G(\pi)$.
\subsection{Tabular Setting}\label{appendix:pqa_proofs}
We first consider the tabular setting with direct policy parameterization where $\pi(\cdot |s ) \in \Delta_A$ for all $s \in \cS$. 
The policy gradient theorem for general utilities (\cref{theorem:gen_utility_pg}) yields $[\nabla_\pi G(\pi)](s, a) = \frac{d^{\pi}(s) \, \qgu(s, a)}{1 - \gamma}$. Therefore, the projected Q-Ascent (\texttt{PQA})~\citep{xiao2022convergence,liu2025convergence} update 
for the step-size $\eta > 0$ can be written as: for any $k \geq 1$,
\begin{align}\label{eq:pqa_update_appendix}
    \pi_{k+1} &= \argmax_{\pi \in \Pi} \bracks*{\sum_{s \in \cS} \, \frac{d^{\pi_k}(s)}{1 - \gamma} \, \dpd{\qgupik(s, \cdot), \pi(\cdot | s) - \pi_k(\cdot | s)} - \frac{1}{2 \, \eta} \, \sum_{s \in \cS} \frac{d^{\pi_k}(s)}{1 - \gamma} \, \normsq{\pi(\cdot | s) - \pi_k(\cdot | s)}} \,.
\end{align}

\begin{theorem}\label{theorem:pqa_last_iterate}
Suppose~\cref{assumption:bounded_gen_util_reward,assumption:exploration,assumption:gu_smooth} holds. 
Let $\pistar$ denote the optimal policy of the general utility objective in~\cref{eq:gu_objective}. Then, the \texttt{PQA} update in~\cref{eq:pqa_update_appendix} with $\eta = \frac{\rho_{\min}}{ L }$ yields that
\begin{equation}
G(\pi^*) - G(\pi_{K})  \leq \frac{32 \, L \, \parens*{1 + \nicefrac{1}{(1 - \gamma) \, \rho_{\min}}}}{(1 - \gamma)^2 \, \rho_{\min} \, K} \,.
\end{equation}
\end{theorem}

\begin{proof}
    Under~\cref{assumption:gu_smooth}, the general utility objective $G$ is $L$-smooth. i.e., for policies $\pi, \pi' \in \Pi$,
    \begin{align*}
        \abs*{G(\pi) - G(\pi') - \langle \nabla_\pi G(\pi'), \pi - \pi' \rangle} \leq \frac{L}{2} \norm{\pi - \pi'}^2_2 \,.
    \end{align*}
    This further implies the following inequalities: 
    \begin{align}
        &G(\pi) \geq G(\pi') + \langle \nabla_\pi G(\pi'), \pi - \pi' \rangle - \frac{L}{2} \, \norm{\pi - \pi'}^2_2 \label{eq:ascent_property_1}  \, \\
        &G(\pi') + \langle \nabla_\pi G(\pi'), \pi - \pi' \rangle \geq G(\pi) - \frac{L}{2} \, \norm{\pi - \pi'}^2_2 \,. \label{eq:ascent_property_2}
    \end{align}
    Under~\cref{assumption:exploration}, for all $s \in \cS$, $\rho(s) = \rho_{\min}$. Let $\eta = \frac{\rho_{\min}}{L}$ and under this choice, $\frac{L}{2} = \frac{\rho_{\min}}{2 \, \eta}$. 
    Using~\cref{eq:ascent_property_1} with $\pi = \pi_{k+1}$ and $\pi' = \pi_k$, we have that
        \begin{align}
        G(\pi_{k+1}) &\geq G(\pi_k) + \langle \nabla_\pi G(\pi_k), \pi_{k+1} - \pi_k \rangle - \frac{\rho_{\min}}{2 \, \eta} \, \norm{\pi_{k+1} - \pi_k}^2_2 \notag \\
        &= G(\pi_k) + \langle \nabla_\pi G(\pi_k), \pi_{k+1} - \pi_k \rangle - \frac{\rho_{\min}}{2 \, \eta} \, \sum_{s \in \cS} \frac{d^{\pi_k}(s)}{d^{\pi_k}(s)} \, \norm{\pi_{k+1}(\cdot | s) - \pi_k(\cdot | s)}^2_2 \notag \\
        &\geq  G(\pi_k) + \langle \nabla_\pi G(\pi_k), \pi_{k+1} - \pi_k \rangle - \sum_{s \in \cS} \frac{ d^{\pi_k}(s)}{2 \, (1 - \gamma) \eta} \,  \, \norm{\pi_{k+1}(\cdot | s) - \pi_k(\cdot | s)}^2_2 
        \tag{Under~\cref{assumption:exploration}, $\frac{1}{d^{\pi_k}(s)} \leq \frac{1}{(1 - \gamma) \rho_{\min}}$} \\
        &= \max_{\pi \in \Pi} \braces*{G(\pi_k) + \langle \nabla_\pi G(\pi_k), \pi - \pi_k \rangle - \sum_{s \in \cS} \frac{d^{\pi_k}(s)}{2 \, (1 - \gamma) \, \eta} \, \norm{\pi (\cdot | s) - \pi_k (\cdot | s)}^2_2} \tag{Using the \texttt{PQA} update in~\cref{eq:pqa_update_appendix}} \\
        &= \max_{\pi \in \Pi} \braces*{G(\pi_k) + \langle \nabla_\pi G(\pi_k), \pi - \pi_k \rangle - \frac{L}{2}\sum_{s \in \cS} \frac{d^{\pi_k}(s)}{(1 - \gamma) \, \rho_{\min} } \, \norm{\pi(\cdot | s) - \pi_k (\cdot | s)}^2_2 \tag{$\eta = \frac{\rho_{\min}}{L}$}} \\
        &\geq \max_{\pi \in \Pi} \braces*{G(\pi_k) + \langle \nabla_\pi G(\pi_k), \pi - \pi_k \rangle - \frac{L}{2 \, (1 - \gamma) \, \rho_{\min}}  \, \norm{\pi - \pi_k }^2_2} \tag{Since $d^{\pi_k}(s) \leq 1$}  \\
        &\geq \max_{\pi \in \Pi} \big\{G(\pi) - \underbrace{\parens*{\frac{L}{2} + \frac{L}{2 \, (1 - \gamma) \, \rho_{\min}}}}_{:= C_1} \norm{\pi - \pi_k}^2_2\big\} \tag{Using~\cref{eq:ascent_property_2} with $\pi' = \pi_k$}
    \end{align}
    \begin{align}
\implies G(\pi_{k+1})  &\geq \max_{\pi \in \Pi} \braces*{G(\pi) - C_1 \norm{\pi - \pi_k}^2_2} \label{eq:pqa_descent_1}
    \end{align}
    Following~\citet{zhang2020variational}, we consider the line segment between the state-action occupancies of $\pi_k$ and $\pistar \coloneq \argmax_{\pi} G(\pi)$ : for $\alpha_k \in [0, 1]$, define
    \begin{equation*}
        \mu_{\alpha_k} := (1 - \alpha_k) \saomt + \alpha_k \mu^{\pistar} \, .
    \end{equation*}
    Since the occupancy set $\cK$ is convex, $\mu_{\alpha_k} \in \cK$ for any $\alpha_k \in [0, 1]$. Furthermore, using the unique one-to-one mapping between $\pi$ and $\saom$, define $\pi_{\alpha_k} \in \Pi$ such that for all $(s, a) \in \cS \times \cA$,
    \begin{equation*}
        \pi_{\alpha_k}(s, a) := \frac{\mu_{\alpha_k}(s, a)}{\sum_{a'}\mu_{\alpha_k}(s, a')}.
    \end{equation*}
    Hence, this defines a valid curve of policies such that $\{\pi_{\alpha_k} \}_{\alpha_k \in [0,1]} \subseteq \Pi$. As a result,
    Using this curve of policies with~\cref{eq:pqa_descent_1},
    \begin{align*}
\implies G(\pi_{k+1})  &\geq \max_{\pi \in \Pi} G(\pi) - C_1 \norm{\pi - \pi_k}^2_2  \\      
    &\geq \max_{\alpha_k \in [0, 1]} G(\pi_{\alpha_k}) - C_1 \, \norm{\pi_{\alpha_k} - \pi_k}^2_2.         \\
    &= \max_{\alpha_k \in [0, 1]} \cF(\pi_{\alpha_k}) - C_1 \, \norm{\pi_{\alpha_k} - \pi_k}^2_2   \\      
    &\geq \max_{\alpha_k \in [0, 1]} \alpha_k \, \cF(\mu^{\pi^*}) + (1 - \alpha_k) \, \cF(\mu^{\pi_k})) - C_1 \, \norm{\pi_{\alpha_k} - \pi_k}^2_2 \tag{Since $\cF$ is concave} \\
    &= \max_{\alpha_k \in [0, 1]} \alpha_k \, G(\pi^*) + (1 - \alpha_k) \, G(\pi_k) - C_1 \, \norm{\pi_{\alpha_k} - \pi_k}^2_2
    \intertext{Multiplying both sides by $-1$ and adding $G(\pistar)$,}
    \implies G(\pistar) - G(\pi_{k+1}) &\leq G(\pistar) - \max_{\alpha_k \in [0, 1]} \braces*{\alpha_k \, G(\pi^*) + (1 - \alpha_k) \, G(\pi_k) - C_1 \normsq{\pi_{\alpha_k} - \pi_k}}  \\
    &= G(\pistar) + \min_{\alpha_k \in [0, 1]} \braces*{-\alpha_k \, G(\pi^*) - (1 - \alpha_k) \, G(\pi_k) + C_1 \, \normsq{\pi_{\alpha_k} - \pi_k}}  \\
    &=\min_{\alpha_k \in [0, 1]} \braces*{(1 -\alpha_k) \, (G(\pi^*) -  G(\pi_k)) + C_1  \, \normsq{\pi_{\alpha_k} - \pi_k}}  \numberthis \label{eq:pqa_descent_2}
    \end{align*}
    Additionally, according to~\citet[Proposition H.1]{zhang2020variational}, the mapping $h(\saom) \to \pi$ is $\frac{2}{\rho_{\min}}$-Lipschitz, where
    \begin{equation*}
        [h(\saom)](s, a) = \frac{\saom(s, a)}{\sum_{a' \in \cA} \saom(s, a')} = \pi(a  |s).
    \end{equation*}
    Hence,
    \begin{align*}
        \norm{\pi_{\alpha_k} - \pi_k}_2^2 &= \norm*{ h \parens*{\alpha \, \mu^{\pi^*} + (1 - \alpha) \, \mu^{\pi_k}} - h \parens*{ \mu^{\pi_k} } }_2^2 \\
        &\leq \frac{4}{\rho_{\min}^2} \norm{\alpha_k \, \mu^{\pistar} + (1 - \alpha_k) \saomt - \saomt}^2_1 \\
        &= \frac{4}{\rho_{\min}^2} \norm{\alpha_k (\mu^{\pistar} - \saomt)}^2_1 \\
        &= \frac{4 \, \alpha_k^2}{\rho_{\min}^2} \norm{\mu^{\pi^*} - \mu^{\pi_k}}^2_1 \\
 \implies \norm{\pi_{\alpha_k} - \pi_k}_2^2   &\leq \frac{16 \, \alpha_k^2}{\rho_{\min}^2 \, (1 - \gamma)^2} \tag{For $\mu, \mu' \in \cK$, $\norm{\mu - \mu'}_1 \leq \frac{2}{1 - \gamma}$}
    \end{align*}
    Using the above inequality with~\cref{eq:pqa_descent_2},
    \begin{align*}
        \implies G(\pi^*) - G(\pi_{k+1}) &\leq \min_{\alpha_k \in [0, 1]} \braces*{ (1 - \alpha_k) \, \parens*{G(\pi^*) - G(\pi_k)} + \frac{16 \, C_1 \, \alpha_k^2}{\rho_{\min}^2 \, (1 - \gamma)^2}} \\
        &= \min_{\alpha_k \in [0, 1]} \left\{ (1 - \alpha_k) \, (\underbrace{G(\pi^*) - G(\pi_k)}_{\delta_k}) + \alpha_k^2 \, C_2  \right\}\tag{$C_2 \coloneq \frac{16 \, C_1}{\rho_{\min}^2 \, (1 - \gamma)^2}$} \\
        \implies \delta_{k+1} &\leq \min_{\alpha_k \in [0, 1]} \braces*{(1 - \alpha_k) \, \delta_k + \alpha_k^2 \, C_2} \numberthis \label{eq:pqa_descent_3}\,.
    \end{align*}
    The minimizer of $(1 - \alpha_k) \, \delta_k + \alpha_k^2 \, C_2$ over $\alpha_k \in [0, 1]$ is $\alpha_k^* = \min\braces*{1, \frac{\delta_k}{2 \, C_2}}$. Define for all $k \geq 1$,
    \begin{equation*}
        \alpha_k := \alpha_k^* = \min\braces*{1, \frac{\delta_k}{2C_2}} \, .
    \end{equation*}
    Since $\alpha_k \in [0, 1]$,  we will prove for all $k \geq2$,
    \begin{equation*}
        \alpha_{k+1}  \leq \parens*{1 - \frac{\alpha_k}{2}} \, \alpha_{t} 
    \end{equation*}
    Note for $k = 1$, since $\alpha_1 \leq 1$, from~\cref{eq:pqa_descent_3} we have that
    \begin{align*}
        \delta_{2} &\leq (1 - 1) \delta_1 + 1^2 \, C_2 = C_2 , \\
        \implies \alpha_2  = \frac{\delta_2}{2C_2} &\leq  \frac{1}{2} < 1 \tag{Dividing by $2 C_2$}
    \end{align*}
    Furthermore, for $k \geq2$, since $\alpha_k = \frac{\delta_k}{2 C_2} < 1$, using~\cref{eq:pqa_descent_3},
    \begin{align*}
        \delta_{k+1} &\leq \parens*{1 - \frac{\delta_k}{2C_2}} \delta_k + \frac{\delta_k^2 }{4 \, C^2} \, C_2 = \parens*{1 - \frac{\delta_k}{4 \, C_2}} \, \delta_k\,, \\
        \implies  \alpha_{k+1} & \leq \parens*{1 - \frac{\alpha_k}{2}} \, \alpha_{t} \tag{Dividing by $2 C_2$}.
    \end{align*}
    Hence, $\alpha_{k+1} \leq \alpha_k < 1$ for all $k \geq2$.
    Using the above inequality,
    \begin{align*}
        \frac{\alpha_{k+1}}{2} &\leq \parens*{1 - \frac{\alpha_k}{2}} \, \frac{\alpha_k}{2} \\
        \implies \frac{2}{\alpha_{k+1}} &\geq \frac{1}{\parens*{1 - \frac{\alpha_k}{2}}  \frac{\alpha_k}{2}} = \frac{2}{\alpha_k} + \frac{2}{2 - \alpha_k}  \tag{$\frac{1}{x(1 - x)} = \frac{1}{x} + \frac{1}{1 - x}$} \\
        &\geq \frac{2}{\alpha_k} + 1 \,. \tag{$\alpha_k < 1$} \\
        \implies \frac{2}{\alpha_{k+1}} &\geq \frac{2}{\alpha_k} + 1 \,.\numberthis \label{eq:pqa_descent_4}
    \end{align*}
    Using~\cref{eq:pqa_descent_4} and recursing from $k=2$ to $K-1$, we have that
    \begin{align*}
        \frac{2}{\alpha_{K}} &\geq \frac{2}{\alpha_2} + (K -2) \\
        &\geq K \tag{$\alpha_2 \leq 1$} \\
        \implies \frac{4 \, C_2}{\delta_{K}} &\geq K  \tag{For all $k \geq2, \alpha_k = \frac{\delta_k}{2 C_2}$} \\
        \implies \delta_{K} &\leq \frac{4 \, C_2}{K}
    \end{align*}
    Making the constants explicit, 
    \begin{equation*}
        G(\pi^*) - G(\pi_{K})  \leq \frac{32 \, L \, \parens*{1 + \nicefrac{1}{(1 - \gamma) \, \rho_{\min}}}}{(1 - \gamma)^2 \, \rho_{\min} \, K} \,.
    \end{equation*}
\end{proof}

\subsection{Linear Functional Approximation Setting}\label{sec:projected_pqa_algorithm}
To handle large state/action spaces, we consider linear function approximation using log-linear policies. Given features $\varphi \in \R^{SA \times d}$, the policy is parameterized as
\begin{equation*}
    \pi_\theta(a | s) = \frac{\exp(\dpd{\varphi(s, a) , \theta})}{\sum_{a'} \exp(\dpd{\varphi(s, a') , \theta})} \,,
\end{equation*}
where $\theta \in \R^d$ are the learnable parameters with $d \ll S \, A$.
Since scaling the features can always be absorbed into a rescaling of $\theta$, we may, without loss of generality, normalize them so that $\max_{s, a} \abs{\varphi(s, a)} = 1$.
As a result, we consider the following policy class
\begin{equation*}
\Pi_{\theta} := \{\pi | \exists \theta \text{ s.t } \pi = \pitheta\}
\end{equation*}
where the model parameterizing the policy is implicit in the definition.

In this setting, to update the policy, we consider the following projected \texttt{PQA} (\texttt{PPQA}) update using the forward KL-divergence. 
For all $k \geq1$, using a fixed step-size $\eta > 0$, the \texttt{PPQA} consists of two steps:
\begin{subequations}
\label{eq:proj_ppg_update}
\begin{align}
    \pi_{\khalf} &= \argmax_{\pi \in (\Delta_A)^S} \bracks*{\sum_{s \in \cS} \frac{d^{\pi_k}(s)}{1 - \gamma} \, \dpd{\qgupik(s, \cdot), \pi(\cdot | s) - \pi_k(\cdot | s)} - \frac{1}{2 \, \eta} \, \sum_{s \in \cS} \frac{d^{\pi_k}(s)}{1 - \gamma} \normsq{\pi(\cdot | s) - \pi_k(\cdot | s)} }\,, \label{eq:proj_pqa_update} \\
    \pi_{k+1} &= \argmin_{\pitheta \in \Pi} \sum_{s \in \cS} d^{\pi_k}(s) \, \KL(\pi_{\khalf}(\cdot | s) \| \pitheta(\cdot | s)) \label{eq:proj_pqa_update_projection}\,.
\end{align}    
\end{subequations}
The intermediate policy $\pi_{\khalf}$ corresponds to the tabular \texttt{PQA} update (cf.~\cref{eq:pqa_update_appendix}).
However, in the presence of function approximation, $\pi_{\khalf}$ may not lie in the set of feasible policies class $\Pi_\theta$.
As a result, we project $\pi_{\khalf}$ onto the set of feasible policies using the forward KL divergence to attain $\pi_{k+1}$.

However, solving the projection in \cref{eq:proj_pqa_update_projection} is generally intractable. To alleviate this issue, following the prior works~\citep{vaswani2021general,asad2024fast,asad2025revisiting,tomar2020mirror,lavington2023target}, we transform this projection into an unconstrained optimization over the parameters $\theta$.
In particular, for each $t \in [T]$,
\begin{align*}
    \pi_{k+1} &=\argmin_{\pitheta \in \Pi_\theta} \sum_{s \in \cS} d^{\pi_k}(s) \,\KL(\pi_{\khalf}(\cdot | s)  \| \pitheta(\cdot | s)) \numberthis \label{eq:statewise_projected_pqa_update} \\
    &= \argmin_{\pitheta \in \Pi_\theta} \sum_{s \in \cS} d^{\pi_k}(s) \,\E_{a \sim \pi_{\khalf}(\cdot | s)} \bracks*{\log \, \parens*{\frac{ \pi_{\khalf}(a | s)}{\pi_\theta(a | s)}}} \\
    &= \argmin_{\pitheta \in \Pi_\theta} \sum_{s \in \cS} d^{\pi_k}(s) \,\E_{a \sim \pi_{\khalf}(\cdot | s)} [-\log \pi_\theta(a | s)] \tag{Dropping constants independent of $\pitheta$} \,.
\end{align*}
which is equivlaent to
\begin{align*}
    \iff \theta^*_{k+1} &= \argmin_{\theta \in \Theta} \sum_{s \in \cS} d^{\pi_k}(s) \, \E_{a \sim \pi_{\khalf}(\cdot | s)} [-\log \pi_\theta(a | s)] \,.
\end{align*}
As a result, computing $\pi_{k+1} = \pi(\theta_{k+1}^*)$ reduces to minimizing the following surrogate objective: 
\begin{equation}\label{eq:projected_pqa_surrogate}
    \ell_k(\theta) = \sum_{s \in \cS} d^{\pi_k}(s) \, \E_{a \sim \pi_{\khalf}(\cdot | s)} [-\log \pi_\theta(a | s)] \,.
\end{equation}
In practice, this minimization is carried out approximately using a finite number of gradient steps where $\theta_{k+1} \approx \argmin_{\theta} \ell_k(\theta)$ and set $\pi_{k+1} = \pitt$. 
In the following Lemma, we prove $\ell_k$ is convex and is $1$-smooth.
\begin{lemma}\label{lemma:projected_pqa_surrogate_smooth}
Consider the surrogate function $\ell_k(\theta) = \sum_{s \in \cS} d^{\pi_k}(s) \KL( \pi_{\khalf}(\cdot | s) \| \pitheta(\cdot | s))$
where $\pi_{\khalf}$ is defined in~\cref{eq:proj_pqa_update}. It holds that $\ell_k(\theta)$ is convex and $1$-smooth.
\end{lemma}
\begin{proof}
Expanding the forward KL term in the surrogate yields that
\begin{align*}
   \ell_k(\theta) &= \sum_{s \in \cS} d^{\pi_k}(s) \KL( \pi_{\khalf}(\cdot | s) \| \pitheta(\cdot | s))  \\
    &= \sum_{s \in \cS} d^{\pi_k}(s) \,\E_{a \sim \pi_{\khalf}(\cdot | s)} \bracks*{\log \, \parens*{\frac{ \pi_{\khalf}(a | s)}{\pi_\theta(a | s)}}} \\
    &= \sum_{s \in \cS} d^{\pi_k}(s) \,\E_{a \sim \pi_{\khalf}(\cdot | s)} [-\log \pi_\theta(a | s)] + Z_k \tag{Let $Z_k\coloneq \sum_{s\in\cS} d^{\pi_k}(s)\,
\E_{a\sim \pi_{\khalf}(\cdot| s)}\!\left[\log \pi_{\khalf}(a| s)\right]$
    }\\
    &= \sum_{s \in \cS} d^{\pi_k}(s) \,\E_{a \sim \pi_{\khalf}(\cdot | s)}\bracks*{\log\parens*{\sum_{a' \in \cA} e^{\dpd{\theta, \varphi(s, a')}}} - \dpd{\theta, \varphi(s, a)}} + Z_k \tag{$\pitheta$ is a log-linear policy}
\end{align*}
Since $\theta \mapsto \log\!\sum_{a'} e^{\langle \theta,\varphi(s,a')\rangle}$ is convex (log-sum-exp of affine functions) and the remaining terms are affine in $\theta$, $\ell_k$ is convex in $\theta$. Moreover, according to~\citep{agarwal2021theory}, $\log \pitheta(a | s)$ is $w^2$-smooth where $w = \max_{s, a}\norm{\varphi(s, a)}_2$. Hence $\ell_k$ is $1$-smooth.
\end{proof}
As a result, since $\ell_k(\theta)$ is a smooth, convex function, this implies monotone descent under standard gradient-based optimization methods and yields a valid approximate projection.
Putting everything together, we summarize the complete \texttt{PPQA} update in~\cref{alg:projected_pqa}.
\begin{algorithm}[H]
\begin{algorithmic}[1]
    \STATE \textbf{Input}: $\theta_k$ (parameters for policy $\pi_k$), $\eta$ (policy update step-size), $N$ (number of surrogate iterations),
    \STATE Use $\pi_k = \pi_{\theta_k}$ and $\eta$ to compute $\pi_{\khalf}$:
    \[\pi_{\khalf} = \argmax_{\pi \in (\Delta_A)^S} \bracks*{\sum_{s \in \cS} \frac{d^{\pi_k}(s)}{1 - \gamma} \, \dpd{\qgupik(s, \cdot), \pi(\cdot | s) - \pi_k(\cdot | s)} - \frac{1}{2 \, \eta} \, \sum_{s \in \cS} \frac{d^{\pi_k}(s)}{1 - \gamma} \normsq{\pi(\cdot | s) - \pi_k(\cdot | s)} } \]
    \STATE Form surrogate $\ell_k(\theta) = \sum_{s \in \cS} d^{\pi_k}(s) \, \E_{a \sim \pi_{\khalf}(\cdot | s)} [-\log \pi_\theta(a | s)]$
    \STATE Initialize inner-loop: $\omega_0 = \theta_k$
        \FOR{$n = 0, 1, \dots, {N-1}$ }
        \STATE $\omega_{n+1} = \omega_n - \zeta \, \nabla_\omega \ell_k(\omega_n)$
        \ENDFOR
    \STATE $\theta_{k+1} = \omega_{N}$
    \STATE Return $\theta_{k+1}$
\end{algorithmic}
\caption{\texttt{PPQA} Update}
\label{alg:projected_pqa}
\end{algorithm}

To control the projection error induced by approximate minimization of $\ell_k(\theta)$, we recall the following assumptions assumptions.
\pqaBias*
\pqaOptErr*
\begin{theorem}\label{theorem:projected_pqa_last_iterate}
Suppose~\cref{assumption:bounded_gen_util_reward,assumption:exploration,assumption:gu_smooth} holds. 
Additionally under~\cref{assumption:pqa_bias,assumption:pqa_optimizaiton_error} such that $\frac{\sqrt{2} \, U}{(1 - \gamma)^3 \, \rho_{\min}} \, \, \sqrt{\eps_{\mathrm{opt}} + \eps_{\mathrm{bias}}} < C$, the \texttt{PPQA} update in~\cref{alg:projected_pqa} with $\eta = \frac{\rho_{\min}}{ L }$,  yields that
    \begin{equation*}
     G(\pi^*) - G(\pi_{K})  \leq \frac{4 \, C}{K}  + 2 \, C \, \sqrt{\frac{\sqrt{2} \, U \, \sqrt{\eps_{\mathrm{opt}} + \eps_{\mathrm{bias}}}}{2 \, C \, (1 - \gamma)^3 \, \rho_{\min}}}.
    \end{equation*}
   where $C \coloneq \frac{16}{(1 - \gamma)^2 \, \rho_{\min}} \, \parens*{1 + \frac{1}{(1 - \gamma) \, \rho_{\min}}}$.
\end{theorem}
\begin{proof}
Under~\cref{assumption:exploration,assumption:gu_smooth}, following the initial proof of~\cref{theorem:pqa_last_iterate}, except with $\pi = \pi_{\khalf}$, from~\cref{eq:pqa_descent_2} we have
\begin{align*}
    G(\pi^*) - G(\pi_{\khalf}) &\leq \min_{\alpha_k \in [0, 1]} \big\{ (1 - \alpha_k) \, \parens*{G(\pi^*) - G(\pi_k)} + \alpha_k^2 \, \underbrace{\frac{16 \, L \, (1 + \nicefrac{1}{(1 - \gamma) \, \rho_{\min}})}{\rho_{\min} \, (1 - \gamma)^2}}_{:= C}\big\}  \\
    \intertext{Adding and subtracting $G(\pi_{k+1})$,}
    \implies G(\pi^*) - G(\pi_{k+1}) &\leq 
   \min_{\alpha_k \in [0, 1]} \braces*{ (1 - \alpha_k) \, \parens*{G(\pi^*) - G(\pi_k)} + \alpha_k^2 \, C}  + G(\pi_{\khalf}) - G(\pi_{k+1}). \numberthis \label{eq:utility_update_diff}
\end{align*}
We first bound $G(\pi_{\khalf}) - G(\pi_{k+1})$:
\begin{align*}
G(\pi_{\khalf}) - G(\pi_{k+1}) &\leq \frac{1}{1 - \gamma} \, \E_{s \sim d^{\pi_{\khalf}}} \bracks*{\dpd{\pi_{\khalf}(\cdot | s) - \pi_k(\cdot | s), \qgupikk(s, \cdot)}} \tag{Using~\cref{lemma:gu_value_difference}}\\
     &= \frac{1}{1 - \gamma} \, \sum_{s} d^{\pi_{\khalf}}(s) \, \dpd{\qgupikk(s, \cdot) , \pi_{\khalf}(\cdot | s) - \pi_{k+1}(\cdot | s)} \\
    &\leq \frac{1}{1 - \gamma} \, \sum_{s} d^{\pi_{\khalf}}(s) \, \supnorm{\qgupikk(s, \cdot)} \norm{\pi_{\khalf}(\cdot | s) - \pi_{k+1}(\cdot | s)}_1 \tag{Holder's inequality}\\ 
    &\leq \frac{U}{(1 - \gamma)^2} \,\sum_{s} d^{\pi_{\khalf}}(s) \, \norm{\pi_{\khalf}(\cdot | s) - \pi_{k+1}(\cdot | s)}_1 \tag{Under~\cref{assumption:bounded_gen_util_reward}, $\supnorm{\qgupik(s, \cdot)} \leq \frac{U}{1 - \gamma}$}\\ 
    &= \frac{U}{(1 - \gamma)^2} \,\sum_{s} \frac{d^{\pi_{\khalf}}(s)}{\rho(s)} \, \rho(s) \, \norm{\pi_{\khalf}(\cdot | s) - \pi_{k+1}(\cdot | s)}_1 \\ 
    &\leq \frac{U}{(1 - \gamma)^2} \,\supnorm{\frac{d^{\pi_{\khalf}}(s)}{\rho(s)}} \, \sum_{s}  \, \rho(s) \, \norm{\pi_{\khalf}(\cdot | s) - \pi_{k+1}(\cdot | s)}_1 \\ 
    &\leq \frac{U}{(1 - \gamma)^3 \, \rho_{\min}} \,\sum_{s}  \, \rho(s) \, \norm{\pi_{\khalf}(\cdot | s) - \pi_{k+1}(\cdot | s)}_1 \tag{Since $d^{\pi_{\khalf}}(s) \leq 1$ and under~\cref{assumption:exploration}, $d^{\pi_{\khalf}} \geq (1 - \gamma) \rho_{\min}$} \\ 
    &\leq \frac{\sqrt{2} \, U}{(1 - \gamma)^3 \, \rho_{\min}} \,\sum_{s}  \, d^{\pi_k}(s) \, \sqrt{\KL(\pi_{\khalf}(\cdot | s) \| \pi_{k+1}(\cdot | s))} \tag{Pinsker's inequality} \\
    &\leq \frac{\sqrt{2} \, U}{(1 - \gamma)^3 \, \rho_{\min}} \, \, \sqrt{\sum_{s}  \, d^{\pi_k}(s) \, \KL(\pi_{\khalf}(\cdot | s) \| \pi_{k+1}(\cdot | s))} \tag{Due to the concavity of $\sqrt{\cdot}$ and Jensen's inequality} \\
    &= \frac{\sqrt{2} \, U}{(1 - \gamma)^3 \, \rho_{\min}} \, \, \sqrt{\sum_{s}  \, d^{\pi_k}(s) \, \KL(\pi_{\khalf}(\cdot | s) \| \pi_{\theta_{k+1}}(\cdot | s))} \\
    &= \frac{\sqrt{2} \, U}{(1 - \gamma)^3 \, \rho_{\min}} \, \, \sqrt{\ell_k(\thtt)} \tag{Definition of $\ell_k$} \\
    &= \frac{\sqrt{2} \, U}{(1 - \gamma)^3 \, \rho_{\min}} \, \, \sqrt{\ell_k(\thtt) - \min_{\theta} \ell_k(\theta) + \min_{\theta} \ell_k(\theta)}  \tag{Add/Subtract $\min_{\theta} \ell_k(\theta)$} \\
    &\leq \frac{\sqrt{2} \, U}{(1 - \gamma)^3 \, \rho_{\min}} \, \, \sqrt{\ell_k(\thtt) - \min_{\theta} \ell_k(\theta) + \eps_{\mathrm{bias}}} \tag{Under~\cref{assumption:pqa_bias}, $\min_{\theta} \ell_k(\theta) \leq \eps_{\bias}$} \\
    &\leq \frac{\sqrt{2} \, U}{(1 - \gamma)^3 \, \rho_{\min}} \, \, \sqrt{\eps_{\mathrm{opt}} + \eps_{\bias}} \tag{Under~\cref{assumption:pqa_optimizaiton_error}, $\abs{\ell_k(\thtt) - \min_{\theta} \ell_k(\theta)} \leq \eps_{\opt}$}
\end{align*}
\begin{align*}
    \implies G(\pi_{\khalf}) - G(\pi_{k+1}) \leq \underbrace{\frac{\sqrt{2} \, U}{(1 - \gamma)^3 \, \rho_{\min}} \, \, \sqrt{\eps_{\mathrm{opt}} + \eps_{\mathrm{bias}}}}_{:= E}
    \,. \numberthis \label{eq:utility_projection_diff}
\end{align*}
Combing~\cref{eq:utility_update_diff} and~\cref{eq:utility_projection_diff}, we have that
\begin{align*}
    G(\pi^*) - G(\pi_{k+1}) &\leq \min_{\alpha_k \in [0, 1]} \braces*{ (1 - \alpha_k) \, \underbrace{(G(\pi^*) - G(\pi_k))}_{:= \delta_k} + \alpha_k^2 \, C} + E \,. \\
\implies    \delta_{k+1} &\leq \min_{\alpha_k \in [0, 1]} \braces*{(1 - \alpha_k) \, \delta_k + \alpha_k^2 \, C} + E \numberthis \label{eq:ppqa_delta_k_recursion}
\end{align*}
The minimizer of $(1 - \alpha_k) \, \delta_k + \alpha_k^2 \, C$ over $\alpha_k \in [0, 1]$ is $\alpha_k^* = \min\braces*{1, \frac{\delta_k}{2C}}$.
Define for all $k \geq1$.
\begin{align*}
   \alpha_k := \alpha_k^* = \min\braces*{1, \frac{\delta_k}{2C}}.
\end{align*}
Under the assumption that $E < C$, we first show that for all $k \geq2$, 
\begin{equation*}
    \alpha_{k+1} \leq \alpha_k - \frac{\alpha_k^2}{2} + \frac{E}{2C}.
\end{equation*}
Note that for $k = 1$, we have $\alpha_1 \leq 1$ and using this fact with~\cref{eq:ppqa_delta_k_recursion},
\begin{align*}
    \delta_{2} &\leq (1 - 1) \delta_1 + 1^2 \, C + E < 2C
\end{align*}
For $k \geq2$, we will prove using induction that $\alpha_k = \frac{\delta_k}{2C}$. \\
\textbf{Base case.} For $k = 2$, it follows from the above expression after dividing by $2C$. \\
\textbf{Induction hypothesis.} Suppose for some $k \geq2$, $\alpha_k = \frac{\delta_k}{2C}$. \\
\textbf{Induction step.} 
Using~\cref{eq:ppqa_delta_k_recursion}
\begin{align*}
   \delta_{k+1} &\leq \parens*{1 - \frac{\delta_k}{2C}} \delta_k + \frac{\delta_k^2 }{4 \, C^2} \, C + E \tag{By the induction hypothesis, $\alpha_k = \frac{\delta_k}{2C}$} \\
   &= \delta_k - \frac{\delta_k^2}{4C}  + E \\
   &\leq \max_{x \in [0, 2C]} \braces*{x - \frac{x^2}{4C}} + E \tag{Since $\alpha_k \leq 1 \implies \delta_k \leq 2C$}\\
   &\leq 2C - \frac{(2C)^2}{4C} + E \\
   &= C + E \\
   &< 2C \tag{Assuming $C < E$} \\
   \implies \alpha_{k+1} = \frac{\delta_{k+1}}{2C} &< 1.
\end{align*}
As a result, using the fact that $\alpha_k = \frac{\delta_{k}}{2C}$ with~\cref{eq:ppqa_delta_k_recursion}, 
for all $k \geq2$
\begin{equation}
    \alpha_{k+1} \leq \alpha_k - \frac{\alpha_k^2}{2} + \underbrace{\frac{E}{2C}}_{:= D} \label{eq:alpha_k_recursion}.
\end{equation}
Using~\cref{eq:alpha_k_recursion}, we will prove using induction  that for all $k \geq1$,
\begin{equation*}
    \alpha_k \leq \frac{2}{k} + \sqrt{2D}.
\end{equation*}
\textbf{Base cases}. For $k = 1$, we have $\alpha_1 \leq 1$ and $\alpha_1 \leq 2 + \sqrt{2D}$. Similarly, for $k = 2$, we have $\alpha_2 \leq 1$ and $\alpha_2 \leq 1 + \sqrt{2D}$. \\
\textbf{Induction hypothesis}. Suppose for some $k \geq2$, $\alpha_{k} \leq \frac{2}{k} + \sqrt{2 \, D}$. \\
\textbf{Inductive Step}. We consider two cases. \\
\textit{Case (i)}: Suppose $1 < \frac{2}{k} + \sqrt{2 \, D}$. 
Since $\alpha_k \leq 1$ for all $k \geq1$, we have
\begin{align*}
    \alpha_{k+1} &\leq \alpha_k - \frac{\alpha_k^2}{2} + D \\
    &\leq \max_{x \in [0, 1]}\braces*{x - \frac{x^2}{2}} + D \\
    &\leq 1 - \frac{1}{2} + D \\
    &= \frac{1}{2} + D \, .
\end{align*}
It suffices to show for $k \geq2$ 
\begin{equation*}
   \frac{1}{2} + D \leq \frac{2}{k+1} + \sqrt{2D}.
\end{equation*}
Define $X := \sqrt{2D}$. As a result, we can rewrite $\frac{1}{2} + D \leq \frac{2}{k+1} + \sqrt{2D}$ in terms of $X$,
\begin{align*}
    \frac{1}{2} + \frac{X^2}{2}  &\leq \frac{2}{k + 1} + X \\
    \iff (X - 1)^2 = &\leq \frac{4}{k+1}
\end{align*}
Since $1 < \frac{2}{k} + X$ and under the assumption that $E < C $, $X = \sqrt{2D} = \sqrt{\frac{E}{C}} < 1$ we have
\begin{align*}
    0 \leq 1 - \frac{2}{k} &< X < 1\\
    \implies -\frac{2}{k} &< X - 1 < 0 \\
    \implies (X -1)^2 &< \frac{4}{k} < \frac{4}{k+1}. \tag{For $k \geq2$, $k^2 \geq k + 1$}
\end{align*}
Hence, 
\begin{equation*}
   \alpha_{k+1} \leq \frac{1}{2} + D \leq \frac{2}{k+1} + \sqrt{2D}.
\end{equation*}
\textit{Case (ii)}: Suppose $\frac{2}{k} + \sqrt{2 \, D} \leq 1$.  By the induction hypothesis, $\alpha_k \leq \frac{2}{k} + \sqrt{2 \, D} \leq 1$.
Define the map
\begin{equation}
    h(x) := x - \frac{x^2}{2} + D.
\end{equation}
Since $h'(x) = 1 - x$, $h$ is monotonically increasing on $[0, 1]$. As a result,
\begin{align*}
    \alpha_{k+1} &\leq h(\alpha_k) \tag{By~\cref{eq:alpha_k_recursion}} \\
    &\leq h \parens*{\frac{2}{k} + \sqrt{2 \, D}} \tag{By the induction hypothesis}\\
    &= \frac{2}{k} + \sqrt{2D} - \frac{2}{k} - \frac{2 \, \sqrt{2D}}{k} - 2D + D \\
    &= \frac{2}{k} - \frac{2}{k} + \sqrt{2 \, D} \parens*{1 - \frac{2}{k}} \\
    &= \parens*{\frac{2}{k + 1} + \sqrt{2 \, D}} + \frac{2}{k} - \frac{2}{k} + \sqrt{2 \, D} \parens*{1 - \frac{2}{k}} - \parens*{\frac{2}{k + 1} + \sqrt{2 \, D}} \tag{Add/Subtract $\frac{2}{k + 1} + \sqrt{2 \, D}$}\\ 
    &= \parens*{\frac{2}{k + 1} + \sqrt{2 \, D}} + \parens*{\frac{2}{k} - \frac{2}{k + 1} - \frac{2}{k}} + \parens*{\sqrt{2 \, D} \parens*{1 - \frac{2}{k}} - \sqrt{2D}} \\
    &= \frac{2}{k + 1} + \sqrt{2 \, D} - \frac{2}{k^2 (k+1)} - \frac{2\sqrt{2 \, D}}{k} 
    \tag{Simplifying the middle term: $\frac{2}{k} - \frac{2}{k + 1} - \frac{2}{k}  = \frac{2}{k \, (k+1)} - \frac{2}{k} = -\frac{2}{k^2 \, (k+1)}$} \\
    &\leq \frac{2}{k + 1} + \sqrt{2 \, D} \,.
\end{align*}
Putting both cases together, we have
\begin{equation*}
    \alpha_{k+1} \leq \frac{2}{k + 1} + \sqrt{2D}.
\end{equation*}
Thus, for all $K \geq 2$,
\begin{equation*}
    \alpha_k = \frac{\Delta_k}{2 \, C} \leq \frac{2}{K} + \sqrt{2D}.
\end{equation*}
As a result, we have
\begin{equation*}
     G(\pi^*) - G(\pi_{K})  \leq \frac{4 \, C}{K}  + 2 \, C \, \sqrt{\frac{\sqrt{2} \, U \, \sqrt{\eps_{\mathrm{opt}} + \eps_{\mathrm{bias}}}}{2 \, C \, (1 - \gamma)^3 \, \rho_{\min}}}.
\end{equation*}
where $C  = \frac{16}{(1 - \gamma)^2 \, \rho_{\min}} \, \parens*{1 + \frac{1}{(1 - \gamma) \, \rho_{\min}}}$.
\end{proof}

\subsection{Supporting Lemmas}
\begin{lemma}[Value Difference Lemma]\label{lemma:value_difference}
    For an arbitrary comparator $u \in [0, U]^{SA}$ with $U > 0$, 
    \begin{equation}
        V^{\pi'}_u - V^{\pi}_u = (I - \gamma P^{\pi'})^{-1} [\cT^{\pi'}_u V^{\pi}_u - V^{\pi}_u] \,,
    \end{equation}
    where $\cT^{\pi'}_u V^{\pi} = u^{\pi'} + \gamma \, P^{\pi'} V^{\pi}$ is the Bellman operator for policy $\pi'$.
\end{lemma}
\begin{proof}
The proof follows from the standard Value difference lemma, except using a comparator $u$ instead of the reward function $r$, and is added for completeness.   
The following notation will be useful in the proof.
\begin{align*}
u^\pi \in \R^S \quad &\text{s.t.} \quad  u^\pi(s) \coloneq \sum_{a \in \cA} u(s, a) \, \pi( a | s) \\
P^\pi \in \R^{S \times S} \quad &\text{s.t.} \quad P^\pi[s, s'] = \text{Pr}^{\pi}(s \to s') \coloneq \sum_{a \in \cA} \Pr(s' | s, a) \, \pi( a | s)
\end{align*}
Recall the fact that the value function for the policy $\pi'$ can be written as $V^{\pi'}_u = (I_S - \gamma \, P_{\pi'})^{-1} u^{\pi'}$. As a result,
\begin{align*}
    V^{\pi'}_u - V^{\pi}_u &= (I_S - \gamma \, P^{\pi'})^{-1} u^{\pi'} - V^{\pi}_u \\
    &= (I_S - \gamma \, P^{\pi'})^{-1} \parens*{u^{\pi'} - (I_S - \gamma \, P_{\pi'}) \, V^{\pi}_u} \\
    &= (I_S - \gamma \, P^{\pi'})^{-1} \parens*{u^{\pi'} + \gamma \, P^{\pi'} V^{\pi}_u  - V^{\pi}_u} \\
    &= (I_S - \gamma \, P^{\pi'})^{-1} \parens*{\cT^{\pi'}_u V^{\pi}_u  - V^{\pi}_u} \tag{$\cT^{\pi'}_u V^{\pi} = u^{\pi'} + \gamma \, P^{\pi'} V^{\pi}$}
\end{align*}
\end{proof}

\begin{lemma}[General Utility Value Difference Lemma]\label{lemma:gu_value_difference}
For any policies $\pi$, and $\pi'$ with general utility pseudo-reward $\Gamma(\pi) = \left.\frac{d \cF(x)}{d x}\right|_{x=\saom}$, it holds that
\begin{equation*}
    G(\pi') - G(\pi) \leq \frac{1}{1 - \gamma} \E_{s \sim d^{\pi'}}\bracks*{[\cT^{\pi'}_{\gupi} V^{\pi}_{\gupi}](s) - V^{\pi}_{\gupi}(s)}.
\end{equation*}
\end{lemma}
\begin{proof}
According to the value difference lemma (\cref{lemma:value_difference}), for any value function $V^\pi_u \in \R^S$ with utility $u \in \R^{SA}$ and any policies $\pi'$, $\pi$, 
\begin{align*}
    V^{\pi'}_u - V^{\pi}_u &= (I_S - \gamma P_{\pi'})^{-1} [\cT^{\pi'}_u V^{\pi}_u - V^{\pi}_u] \tag{$\cT^{\pi'}_u V^{\pi}_u = u^{\pi'} + \gamma P^{\pi} V^{\pi}_u$} \\
   \implies  \rho \, [V^{\pi'}_u - V^{\pi}_u] &= \rho \, (I_S - \gamma P_{\pi'})^{-1} [\cT^{\pi'}_u V^{\pi}_u - V^{\pi}_u] \tag{Multiplying both sides by initial state distribution $\rho$} \\ 
   \implies \E_{s \sim \rho}[V^{\pi'}_u(s) - V^{\pi}_u(s)] &= \frac{(1 - \gamma) \, \rho \, (I_S - \gamma P_{\pi'})^{-1} [\cT^{\pi'}_u V^{\pi}_u - V^{\pi}_u]}{1 - \gamma}  \\
   &= \frac{1}{1 - \gamma} \E_{s \sim d^{\pi'}}\bracks*{[\cT^{\pi'}_u V^{\pi}_u](s) - V^{\pi}_u(s)} \tag{$d^{\pi'} \coloneq (1 - \gamma) \, \rho \, (I_S - \gamma P_{\pi'})^{-1}$} \\
\implies V^{\pi'}_{\gupi}(\rho) - V^{\pi}_{\gupi}(\rho) &= \frac{1}{1 - \gamma} \E_{s \sim d^{\pi'}} \bracks*{[\cT^{\pi'}_{\gupi} V^{\pi}_{\gupi}](s)  - V^{\pi}_{\gupi}(s)} 
\tag{Setting $u = \gupi \coloneq \nabla_{\mu} \cF(\saom)$}
\end{align*}
Next, we use the LP formulation of MDPs to relate the LHS of the above equation to the general utility objective
\begin{align*}
   V^{\pi'}_{\gupi}(\rho) - V^{\pi}_{\gupi}(\rho)  &= \dpd{\gupi, \mu^{\pi'} - \saom} \\
   &= \dpd{\nabla_{\mu} \cF(\saom), \mu^{\pi'} - \saom} \tag{By definition of $\gupi$} \\
   &\geq \cF(\mu^{\pi'}) - \cF(\saom) \tag{Since $\cF$ is concave} \\
   &= G(\pi') - G(\pi) \tag{Using the one-to-one relationship between $\pi$ and $\saom$}.
\end{align*}
Putting everything together,
\begin{align*}
    G(\pi') - G(\pi) &\leq \frac{1}{1 - \gamma} \E_{s \sim d^{\pi'}}\bracks*{[\cT^{\pi'}_{\gupi} V^{\pi}_{\gupi}](s) - V^{\pi}_{\gupi}(s)}.
\end{align*}
\end{proof}

\begin{lemma}[Euclidean norm Lipschitz bound]\label{lemma:gu_lipschitz}
Under~\cref{assumption:bounded_gen_util_reward,assumption:gu_smooth}, the general utility objective $G(\pi)$ is $\frac{U \, \sqrt{A}}{(1 - \gamma)^2}$-Lipschitz with respect to the Euclidean norm.
\end{lemma}
\begin{proof}
Under~\cref{assumption:gu_smooth}, $G$ is $L$-smooth.
To prove the general utility objective is $B$-Lipschitz, it suffices to bound $\norm{\nabla_\pi G(\pi)}_2$.
According to the policy gradient theorem for general utilities (\cref{theorem:gen_utility_pg}), 
\begin{align*}
    [\nabla_\pi G(\pi)]_{s, a} &= \frac{d^\pi(s)}{1 - \gamma} \, \qgu(s, a).
\end{align*}
This further implies that
\begin{align*}
    \normsq{\nabla_\pi G(\pi)} &= \frac{1}{(1 - \gamma)^2}\sum_{s, a} \parens*{d^{\pi}(s) \qgu(s, a)}^2 \\
    &\leq \frac{[\max_{s, a}\qgu(s, a)]^2 \, A}{(1 - \gamma)^2} \, \sum_{s} [d^\pi(s)]^2 \\
    &\leq \frac{U^2 \, A}{(1 - \gamma)^4} \, \norm{d^\pi}_2 \tag{Under~\cref{assumption:bounded_gen_util_reward}, $\qgu(s, a) \leq \frac{U}{1 - \gamma}$} \\
    &\leq \frac{U^2 \, A}{(1 - \gamma)^4} \tag{Since $\norm{d^\pi}_2 \leq \norm{d^{\pi}}_1 = 1$}.
\end{align*}
Finally, we have that
\begin{align*}
    \norm{\nabla_\pi G(\pi)}_2 &\leq \frac{U \, \sqrt{A}}{(1 - \gamma)^2}.
\end{align*}
\end{proof}
\subsection{Additional Lemmas for the Augmented Lagrangian General Utility}\label{appendix:al_gu_lemmas}
\begin{restatable}{proposition}{auglagprop}\label{proposition:alm_gu_pg}
For each iteration $t \in [T]$, the policy gradient of the AL in~\cref{eq:al_max_min} is given by
\begin{align}
& \nabla_\theta \F(\pi(\theta)) = \textstyle \sum_{s} \frac{d^{\pi}(s)}{1 - \gamma} \sum_{a} \, \qgu(s,a) \frac{d \pitheta(a | s)}{d \theta} \; \nonumber \\
& \Gamma(\pi) = r - \beta \, \textstyle \sum_{i=1}^m c_i \, \min\braces*{V^{\pi}_{c_i}(\rho) - b_i - \frac{\lambda_i}{\beta}, 0}, \label{eq:alm_pg}
\end{align}
and $Q^{\pi}_{\Gamma(\pi)}$ is the state-action value function for the general utility pseudo-reward $\Gamma(\pi)$.
\end{restatable}
\begin{proof}
By~\cref{theorem:gen_utility_pg}, the policy gradient of a general utility  is
\begin{equation*}
    \nabla_\theta G(\pi) = \cF(\saom) = \frac{1}{1 - \gamma} \sum_{s \in \cS} d^{\pi}(s) \, \sum_{a \in \cA} \qgu(s, a) \frac{d \pitheta(a | s)}{d \theta}
\end{equation*}
where $\Gamma(\pi) \coloneq \left.\frac{d \cF(x)}{d x}\right|_{x=\saom} \in \R^{SA}$ is the general utility pseudo-reward function.
According to~\cref{proposition:aug_lag_is_gu_function}, the augmented Lagrangian in~\cref{eq:al_max_min} is equivalent to the following general utility function
\begin{equation*}
   G(\pi) = \cF(\saom) = \dpd{\saom, r} + \frac{\alr}{2} \sum_{i=1}^m \, \parens*{-\min\braces*{\dpd{\saom, c_i} - b_i - \frac{\lambda_i}{\beta}, 0}^2 + \frac{\lambda_i^2}{\alr^2}}.
\end{equation*}
Moreover, for each $i \in [m]$,
\begin{align*}
   [h(\saom)]_i &= \begin{cases}
       \frac{\lambda_i^2}{\beta^2} & \text{If }\dpd{\saom, c_i} \geq b_i + \frac{\lambda_i}{\beta}\\
       -\parens*{\dpd{\saom, c_i} - b_i - \frac{\lambda_i}{\beta}}^2 + \frac{\lambda_i^2}{\beta^2}& \text{Else if } \dpd{\saom, c_i} < b_i + \frac{\lambda_i}{\beta}.
   \end{cases}
\end{align*}
Hence,
\begin{align*}
    \frac{d[h(\saom)]_i}{d \saom} &=
  \begin{cases} 
  0 & \text{If }\dpd{\saom, c_i} \geq b_i + \frac{\lambda_i}{\beta}\\
  -2 \, c_i \, \parens*{\dpd{\saom, c_i} - b_i - \frac{\lambda_i}{\beta}} & \text{Else if } \dpd{\saom, c_i} < b_i + \frac{\lambda_i}{\beta}.
  \end{cases} \\
  &= -2 c_i \, \min\braces*{\dpd{\saom, c_i} - b_i - \frac{\lambda_i}{\beta}, 0} 
\end{align*}
Putting everything together, we have
\begin{align*}
\Gamma(\pi) &= r - \beta \, \sum_{i=1}^m c_i \, \min\braces*{\dpd{\saom, c_i} - b_i - \frac{\lambda_i}{\beta}, 0}.
\end{align*}
\end{proof}

\begin{proposition}\label{proposition:aug_lag_is_gu_function}
For any dual variable $\lambda \in \R^m$ and penalty parameter $\beta > 0$, the augmented Lagrangian in~\cref{eq:al_max_min} is a concave general utility function:
\begin{equation*}
   \F(\saom) = \dpd{\saom, r} + \frac{\alr}{2} \sum_{i=1}^m \, \parens*{-\min\braces*{\dpd{\saom, c_i} - b_i - \frac{\lambda_i}{\beta}, 0}^2 + \frac{\lambda_i^2}{\alr^2}}.
\end{equation*}
\end{proposition}
\begin{proof}
Recall for any dual variable $\lambda \in \R^m$ and penalty parameter $\beta > 0$ the augmented Lagrangian in~\cref{eq:al_max_min} is defined as 
\begin{align*}
\al^\beta(\pi, \lambda) &= V^{\pi}_r(\rho) + \frac{\alr}{2} \sum_{i=1}^m \, \parens*{-\min\braces*{V^{\pi}_{c_i}(\rho) - b_i - \frac{\lambda_i}{\beta}, 0}^2 + \frac{\lambda_i^2}{\alr^2}}  \\
&= \underbrace{\dpd{\saom, r} + \frac{\alr}{2} \sum_{i=1}^m \, \parens*{-\min\braces*{\dpd{\saom, c_i} - b_i - \frac{\lambda_i}{\beta}, 0}^2 + \frac{\lambda_i^2}{\alr^2}}}_{:=\F(\saom)} \tag{Using the linear programming formulation of CMDPs}
\end{align*}
Since the function $x \mapsto -\min\braces*{x, 0}^2$ is concave, $\cF(\saom)$ is concave with respect to $\saom$.
Hence, $\F(\saom)$ is a concave general utility function.
\end{proof}

\begin{lemma}\label{lemma:alm_saom_smooth}
Under the direct parameterization, for any dual variable $\lambda \in \R^m$, and penalty parameter $\beta > 0$
\begin{equation*}
\al^\beta(\pi, \lambda) =  V^{\pi}_r(\rho) + \frac{\alr}{2} \sum_{i=1}^m \, \parens*{-\min\braces*{V^{\pi}_{c_i}(\rho) - b_i - \frac{\lambda_i}{\beta}, 0}^2 + \frac{\lambda_i^2}{\alr^2}}
\end{equation*}
is 
$\frac{2 \gamma \, A}{(1-\gamma)^3} \parens*{1 + \beta \, \sum_{i=1}^m \parens*{\frac{1}{1 - \gamma} + b_i + \frac{\lambda_i}{\beta}}} + \frac{\beta \, m \, A^2}{1 - \gamma}$-smooth.
\end{lemma}
\begin{proof}
By Taylor’s theorem it suffices to show that the Hessian is bounded
We can split the analysis into two extreme cases, depending on whether which if all the constraints are either active or inactivate. Define
\begin{equation*}
    g_i(\pi) \coloneq V^{\pi}_{c_i}(\rho) - b_i - \frac{\lambda_i}{\beta}.
\end{equation*}
\noindent \textbf{Case (I)}: 
If $g_i(\pi) \geq 0$ for all $i \in [m]$, $G(\pi) = V^{\pi}_r(\rho) + \sum_{i=1}^m \frac{\lambda_i^2}{2 \, \alr}$.
According to~\citet[Lemma 54]{agarwal2021theory}, we know that $V_r^\pi(\rho)$ is $\iota = \frac{2 \gamma \, A}{(1-\gamma)^3}$-smooth. Hence $G(\pi)$ is $\iota$-smooth.

\noindent \textbf{Case (II)}: 
If $g_i(\pi) < 0$ for all $i \in [m]$, we first calculating the it's Hessian,
\begin{equation}\label{eq:g_i_hessian}
   \nabla_\pi^2 [g_i(\pi)]^2  =  2 \parens*{\nabla_\pi V^{\pi}_{c_i}(\rho) \, \nabla_\pi V^{\pi}_{c_i}(\rho)^\top - g_i(\pi) \, \nabla_\pi^2 V^{\pi}_{c_i}(\rho)}
\end{equation}
\begin{align*}
\norm{\nabla_\pi^2 G(\pi)} &= \norm*{\nabla^2_\pi V^{\pi}_r(\rho) + \frac{\beta}{2} \sum_{i=1}^m \parens*{- \nabla_\pi^2 [g_i(\pi)]^2_{-}}} \\
 &\leq \norm*{\nabla^2_\pi V^{\pi}_r(\rho) + \frac{\beta}{2} \sum_{i=1}^m \parens*{- \nabla_\pi^2 [g_i(\pi)]^2}} \\
 &\leq \norm*{\nabla^2_\pi V^{\pi}_r(\rho) - \beta \sum_{i=1}^m \parens*{\nabla_\pi V^{\pi}_{c_i}(\rho) \, \nabla_\pi V^{\pi}_{c_i}(\rho)^\top + g_i(\pi) \, \nabla_\pi^2 \nabla_\pi V^{\pi}_{c_i}(\rho)}} \tag{Using~\cref{eq:g_i_hessian}}\\
 &\leq \norm*{\nabla^2_\pi V^{\pi}_r(\rho)} + \beta \sum_{i=1}^m \normsq{\nabla_\pi V^{\pi}_{c_i}(\rho)}  + \beta \, \sum_{i=1}^m  \abs{g_i(\pi)} \, \norm{\nabla_\pi^2  V^{\pi}_{c_i}(\rho)} \tag{Using triangle inequality}\\
 &\leq \iota + \beta \sum_{i=1}^m \normsq{\nabla_\pi V^{\pi}_{c_i}(\rho)}   + \beta \, \sum_{i=1}^m  \abs{g_i(\pi)} \, \iota \tag{Since $V^{\pi}_\diamond(\rho)$ is $\iota$-smooth}\\
 &\leq \iota + \beta \sum_{i=1}^m \normsq{\nabla_\pi V^{\pi}_{c_i}(\rho)}   + \beta \, \iota \sum_{i=1}^m \parens*{\frac{1}{1 - \gamma} + b_i + \frac{\lambda_i}{\beta}} \tag{Since $c_i \in [0, 1]$}\\
 &\leq \iota + \frac{\beta \, m \, A^2}{1 - \gamma}    + \beta \, \iota \sum_{i=1}^m \parens*{\frac{1}{1 - \gamma} + b_i + \frac{\lambda_i}{\beta}} \tag{Using the policy gradient theorem $[\nabla_\pi V^{\pi}_{c_i}(\rho)]_{s, a} = d^{\pi}(s) \, Q_{c_i}^{\pi}(s, a)$}\\
 &= \iota \parens*{1 + \beta \, \sum_{i=1}^m \parens*{\frac{1}{1 - \gamma} + b_i + \frac{\lambda_i}{\beta}}} + \frac{\beta \, m \, A^2}{1 - \gamma}
\end{align*}
Combining both cases, we have $G$ is $\iota \parens*{1 + \beta \, \sum_{i=1}^m \parens*{\frac{1}{1 - \gamma} + b_i + \frac{\lambda_i}{\beta}}} + \frac{\beta \, m \, A^2}{1 - \gamma}$-smooth.
\end{proof}

\begin{lemma}\label{lemma:U_alm}
In~\cref{alg:generic_alm_pg}, using $\eps_t = \frac{\sigma}{t^2}$, $\sigma > 0$, $\alr > 0$ for all iterations $t \in [1, T]$, the general utility pseudo-reward of the augmented Lagrangian defined in~\cref{proposition:alm_gu_pg} satisfies~\cref{assumption:bounded_gen_util_reward} with 
$U = 1 + \alr \,\norm{b}_1   + \sqrt{m} \, \sqrt{\normsq{M} + \frac{\beta \sigma \pi^2}{3}} + \norm{M}_1$
where $M \coloneq (\min_{i \in [m]} \zeta_i, (1 - \gamma)^{-1}$ and $\zeta_i = \max_{\pi} V^{\pi}_{c_i}(\rho) - b > 0$.
\end{lemma}
\begin{proof}
Fix $t \geq 1$. Recall from~\cref{proposition:alm_gu_pg} the general utility pseudo-reward of the augmented Lagrangian is
\begin{align*}
\Gamma(\pi) &= r - \beta \, \sum_{i=1}^m c_i \, \min\braces*{V^{\pi}_{c_i}(\rho) - b_i - \frac{\lambda_i}{\beta}, 0} \\
&= r + \beta \, \sum_{i=1}^m c_i \, \max\braces*{b_i + \frac{\lambda_i}{\beta} - V^{\pi}_{c_i}(\rho), 0} \tag{Since $-\min(a, 0) = \max(-a, 0)$}.
\end{align*}
Hence for any arbitrary $\pi$,
\begin{align*}
[\Gamma(\pi)](s, a)&= r(s, a) + \beta \, \sum_{i=1}^m c_i(s, a) \, \max\braces*{b_i + \frac{\lambda_i}{\beta} - V^{\pi}_{c_i}(\rho), 0} \\
&\leq r(s, a) + \beta \, \sum_{i=1}^m c_i(s, a) \, \abs*{b_i + \frac{\lambda_i}{\beta}} \tag{$V^{\pi}_{c_i}(\rho) \in \bracks*{0, \frac{1}{1 -\gamma}}$} \\
&\leq 1 + \beta \, \sum_{i=1}^m \, \abs*{b_i + \frac{\lambda_i}{\beta}} \tag{Since $r(s, a) \in [0, 1]$ and $c_i(s, a) \in [0, 1]$} \\
&= 1 + \beta \, \norm{b}_1 + \norm{\lambda}_1 \tag{Using triangle inequality}
\end{align*}
This further implies that
\begin{align*}
    [\Gamma(\pi_t)](s, a) &\leq 1 + \beta \, \norm{b}_1 + \norm{\lambda_t}_1 \\
    &= 1 + \beta \, \norm{b}_1 + \norm{\lambda_t - \lambda^* + \lambda^*}_1 \tag{Add/Subtract $\lambda^*$}  \\
    &\leq 1 + \beta \, \norm{b}_1 + \norm{\lambda_t - \lambda^*} + \norm{\lambda^*}_1 \tag{Using triangle inequality} \\
    &\leq 1 + \beta \, \norm{b}_1 + \sqrt{m} \, \norm{\lambda_t - \lambda^*}_2 + \norm{\lambda^*}_1 \\
    &\leq 1 + \alr \,\norm{b}_1  + \sqrt{m} \, \sqrt{\normsq{\lambda_1- \lstar} + 2 \, \beta \, \sum_{i = 1}^{t-1} \epsilon_i } + \norm{\lambda^*}_1    
     \tag{Using~\cref{lemma:bounded-dual-iterates} to bound $\norm{\lambda_k - \lambda^*}_2$} \\
    &\leq 1 + \alr \,\norm{b}_1  + \sqrt{m} \, \sqrt{\normsq{\lstar} + \frac{\beta \sigma \pi^2}{3}} + \norm{\lambda^*}_1    
     \tag{Since $\sum_{i=1}^{t-1}\frac{1}{i^2} \le \sum_{k=1}^\infty = \frac{\pi^2}{6}$ and $\lambda_1 = 0$} \\
    &\leq 1 + \alr \,\norm{b}_1  
    + \sqrt{m} \, \sqrt{M^2 + \frac{\beta \sigma \pi^2}{3}} + \norm*{M}_1 \tag{Using~\cref{eq:lambda_star_ub}}   
\end{align*}
\end{proof}

%% file: Appendix/A3_experiments.tex
\pagebreak
\section{Large-Scale Experiments on Continious Control Tasks}\label{appendix:large_experiments}
\subsection{Key Results}
We evaluate our framework on three MuJCo tasks (each with a single constraint) from Safety-Gymnasium~\citep{ji2023safety}
and compare against standard on-policy PG methods in the constrained setting: $\PPOLag$~\citep{ray2019benchmarking} and \CPO~\citep{achiam2017constrained}.\footnote{We omit \texttt{ReLoad} since their implementation for the deep RL setting was not publicly available.}

To instantiate the oracle-solver in~\cref{alg:generic_alm_pg}, we use
~\texttt{PPO}~\citep{schulman2017proximal}, %
and the more recent~\texttt{SPMA} algorithm~\citep{asad2024fast}.
We choose \texttt{SPMA} since it has been shown to achieve competitive performance on MuJoCo tasks and, unlike \texttt{PPQA}, it employs a surrogate that easily generalizes to continuous action spaces. 
To ensure a fair comparison with baseline methods, we match the total number of PG estimates (i.e., samples used to form the gradients) across all algorithms. Note that the framework itself is agnostic to this choice; other PG algorithms such as \TRPO~\citep{schulman2015trust} could be used interchangeably.
\begin{figure*}[h!]
  \begin{center}
    \centerline{\includegraphics[scale=0.6]{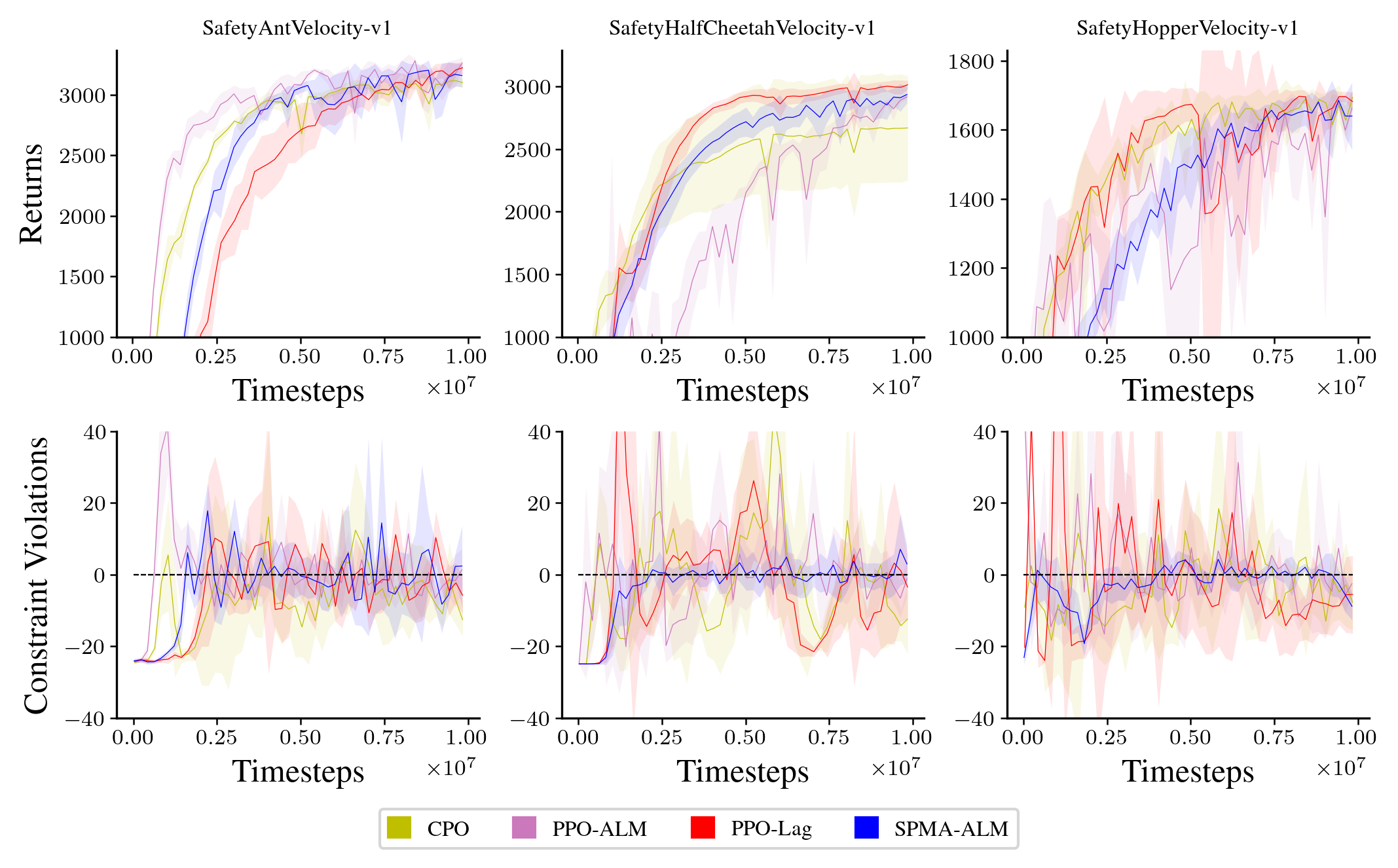}}
    \caption{For each experiment, we use $5$ seeds and report their 95\% confidence intervals.
    Across all environments, \texttt{SPMA-ALM} stays closest to the constraint boundary and exhibits the least oscillatory behavior.
    }
    \label{fig:alm_vs_all_safety_gym}
  \end{center}
\end{figure*}

\vspace{-2ex}
\textbf{Results}. 
\cref{fig:alm_vs_all_safety_gym} shows~\texttt{AL}-based methods achieves comparable performance to the baseline approaches, justifying the practical effectiveness of our framework. 
During the experiments, we observed that several Lagrangian implementations in \texttt{omnisafe}~\citep{ji2024omnisafe} rely on heavily tuned heuristics, such as scaling the objective
by the multiplier~$\lambda$, despite the lack of theoretical support for this practice.
For consistency and fairness, this modification was removed from all baseline methods.
Our \texttt{ALM} variants does not require this heuristic, and all results are obtained under a unified formulation of the \texttt{AL} framework. Unlike methods such as \texttt{PPO-Lag} or $\CPO$, our approach is backed by a clear theoretical foundation, yielding a method that is both principled and practical.

For clarity, we list the performance of the final policy from~\cref{fig:alm_vs_all_safety_gym} in~\cref{table:safety_gym_table}, 
\begin{table}[h]
    \centering
    \caption{Results of the final policy in~\cref{fig:alm_vs_all_safety_gym}}
\resizebox{\textwidth}{!}{
    \begin{tabular}{
        l 
        c c 
        c c 
        c c
    }
    \toprule
    & \multicolumn{2}{c}{\textbf{SafetyAntVelocity-v1}} 
    & \multicolumn{2}{c}{\textbf{SafetyHalfCheetahVelocity-v1}} 
    & \multicolumn{2}{c}{\textbf{SafetyHopperVelocity-v1}} \\
    \cmidrule(lr){2-3} \cmidrule(lr){4-5} \cmidrule(lr){6-7}
    \textbf{Algorithm} 
    & \textbf{Return} & \textbf{CV} 
    & \textbf{Return} & \textbf{CV} 
    & \textbf{Return} & \textbf{CV} \\
    \midrule
\texttt{SPMA-ALM} & \textbf{3202.29} $\pm$ \textbf{79.14} & -2.68 $\pm$ 5.69 & 2933.20 $\pm$ 69.90 & -0.74 $\pm$ 3.34 & \textbf{1687.37} $\pm$ 7.64 & -6.06 $\pm$ 2.27 \\
\texttt{PPO-ALM} & 3042.76 $\pm$ 161.36 & 2.40 $\pm$ 5.45 & 2933.91 $\pm$ 21.05 & -7.06 $\pm$ 1.87 & 1644.06 $\pm$ 45.45 & -9.92 $\pm$ 2.85 \\
\texttt{PPO-Lag} & 3223.58 $\pm$ 52.21 & -4.96 $\pm$ 2.50 & \textbf{3010.19} $\pm$ \textbf{37.12} & -7.74 $\pm$ 3.30 & 1675.45 $\pm$ 42.51 & -6.56 $\pm$ 9.12 \\
\texttt{CPO} & 3011.19 $\pm$ 123.74 & -11.26 $\pm$ 5.13 & 2685.90 $\pm$ 419.20 & 6.18 $\pm$ 31.42 & 1682.68 $\pm$ 40.78 & 3.94 $\pm$ 23.10 \\
\bottomrule
\end{tabular}
}
\label{table:safety_gym_table}
\end{table}

\newpage

\subsection{Experimental Details}
The Safety MuJoCo benchmark~\citep{ji2023safety} augments standard continuous-control tasks from Gymnasium~\citep{towers2024gymnasium} with a velocity-based safety signal.
Each task is formulated as the following constrained MDP:
$\max_{\pi} V^{\pi}_r(\rho) \quad \text{ s.t } \quad V^{\pi}_c(\rho) \leq b.$
In this setting, the constraint reward is treated as a safety-related quantity that the agent must keep below a specified threshold, with $b=25$ Descriptions of the individual environments without constraints can be found in~\citet{towers2024gymnasium}.

Hyperparameters are selected via a grid search, and each variant is run with five random seeds. For each configuration, we select the policy with the highest return among those satisfying the relaxed constraint $V^{\pi}_c(\rho) \leq b + \varepsilon$. Each seed evaluates its final policy over 10 independent episodes, resulting in $50$ total rollouts used to estimate $V^{\pi}_c(\rho)$. For \texttt{SPMA}, to set step-size $\eta$ for the idealized policy update, we perform a grid search over the following fixed values: \{0.001, 0.01, 0.1, 0.3, 0.5, 0.7, 0.9\}. 
A complete list of hyperparameters used in the experiments are presented in~\cref{tab:mujoco-hyperparams}.

These experiments were conducted on compute cluster entirely on an CPU.
using an AMD EPYC 9655 (Zen 5) processors.
Each individual run took around $\approx 4$ CPU-hours to complete.
The total computational budget used for these experiments was approximately 800 CPU-hours.

\begin{table}[h]
\caption{Hyper-parameters for the Safety-Gymnasium experiments.}
\centering
\resizebox{\textwidth}{!}{
\begin{tabular}{|l|c|c|c|c|}
\hline
\textbf{Hyperparameter} & \texttt{SPMA-ALM} & \texttt{PPO-ALM} & \texttt{PPO-LAG} & \texttt{CPO} \\ 
\hline
Minibatch size & \{64, 256\} & \{64, 256\} & \{64, 256\} & 128 \\
Probability ratio clipping & \xmark & \{0.1, 0.2 ,0.3\} & \{0.1, 0.2, 0.3\} & \xmark \\
Adam step-size (surrogate) & \{3e-3, 3e-2, 3e-1\} & \{3e-3, 3e-2, 3e-1\} & \{3e-3, 3e-2, 3e-1\}  & 0.001 \\
Adam step-size (dual) & \{1e-4, 1e-3, 1e-2, 1e-1\}  & \{1e-4, 1e-3, 1e-2, 1e-1\} & \{1e-4, 1e-3, 1e-2, 1e-1\} & 0.001 \\
Gradient clipping & \xmark & \xmark & \cmark & \cmark \\
Number of surrogate loop updates (N) & 10 & 10 & 10 & 10 \\
\texttt{AL} Updates & 10 & 10 & \xmark & \xmark \\
\texttt{AL} Penalty & \{1e-3, 1e-2, 1e-1\} & \{1e-3, 1e-2, 1e-1\}  & \xmark & \xmark \\
Steps per Epoch & 2,000 & 2,000 & 20,000 & 20,000 \\
\hline
Normalize Lagrangian & \multicolumn{4}{c|}{\xmark} \\
Observation Normalization Wrapper & \multicolumn{4}{c|}{\cmark} \\
Action Rescale Wrapper & \multicolumn{4}{c|}{\cmark} \\
Standard Advantage & \multicolumn{4}{c|}{\cmark} \\
Total number of samples & \multicolumn{4}{c|}{$10^7$} \\
GAE $\lambda$ & \multicolumn{4}{c|}{0.95} \\
Discount factor & \multicolumn{4}{c|}{0.99} \\
\hline
\end{tabular}}
\label{tab:mujoco-hyperparams}
\end{table}

\newpage
\subsection{Pseudo-Code for \texttt{SPMA-ALM}}
For simplicity, we present the pseudo-code for \texttt{SPMA-ALM} in the single-constraint setting.
We omit implementation details such as the surrogate optimizer, the dual step-size, and the value networks.The corresponding pseudo-code for \texttt{PPO-ALM} is analogous.
\begin{algorithm}[H]
\begin{algorithmic}[1]
    \STATE \textbf{Input}:
    $\bar{\theta}_1$ (initial paramter of policy) $\lambda_1 = 0$ (dual variable), $\alr > 0$ (constant constraint penalty),  $T$ (number of iterations), $K$ (number of \texttt{AL} updates), $N$ (number of surrogate updates) $\eta$ (policy step-size), $\eta_{\text{surrogate}}$ (surrogate step-size) $\eta_{\text{dual}}$ (dual step-size)
    \FOR{$t = 1, 2, \dots, T$ }
    \STATE Initialize \texttt{AL} sub-problem with $\theta_1 = \bar{\theta}_t$ (equivalent policy $\pi_t = \pi_{\bar{\theta}_t}$)
    \FOR{$k=1, 2, \dots K$}
    \STATE Interact with environment using $\pi_k = \pi_{\theta_k}$ and construct estimate $\hat{V}^{\pi_k}_{c}(\rho)$, $\hat{A}^{\pi_k}_{r}$, $\hat{A}^{\pi_k}_{c}$ .
    \IF{$V^{\pi_k}_{c}(\rho) \geq b_1 + \frac{\lambda_t}{\beta}$}
        \STATE $\hat{A}^{\pi_k}_{\Gamma(\pi_k)} = \hat{A}^{\pi_k}_r$
    \ELSE
        \STATE $\hat{A}^{\pi_k}_{\Gamma(\pi_k)} = \hat{A}^{\pi_k}_r + \hat{A}^{\pi_k}_c \parens*{b + \frac{\lambda_t}{\beta} - \hat{V}^{\pi_k}_{c}(\rho)} \, \beta$
    \ENDIF
    \STATE Construct \texttt{SPMA} surrogate: $$\ell_k(\theta) = \E_{s \sim d^{\pi_k}}\E_{a \sim \pi_k(\cdot | s)}\left[\hat{A}^{\pi_k}_{\Gamma(\pi_k)}(s, a) + \frac{1}{\xi} \log\parens*{\frac{\pi_\theta(a | s)}{\pi_k(a | s)}}\right]$$
    \STATE Initialize: $\omega_0 = \theta_k$
    \FOR{$n = 1, 2, \dots, N$}
    \STATE $\omega_{n+1} = \omega_n - \eta_{\text{surrogate}} \nabla_\omega \ell_k(\omega_n)$
    \ENDFOR
    \STATE $\theta_{k+1} = \omega_{N+1}$
    \ENDFOR
    \STATE Recover approximation solution to AL-subproblem: $\bar{\theta}_t+1 = \theta_{K+1}$ ;  $\pi_{t+1} = \pi_{K+1} = \pi_{\theta_{K+1}}$
    \STATE $\lambda_{t+1} = \lambda_t - \eta_{\text{dual}} \, \parens*{\hat{V}^{\pi_{t+1}}_{c}(\rho) - (b_i + \xi(\pi_{t+1})]}$ where $\xi(\pi) = \max\braces*{\hat{V}^{\pi}_{c}(\rho) - b - \frac{\lambda_t}{\beta}, 0}$
    \ENDFOR
    \STATE Return $\pi_{T+1}$
\end{algorithmic}
\caption{\texttt{SPMA-ALM}}
\end{algorithm}

\newpage
\subsection{Sensitivity Analysis Of The Final Policy}

In~\cref{fig:boxplot}, we investigate the sensitivity of how well the final policy satisfies the constraints.
We report mean, worst-case, 90th (P90), and 95th (P95) percentiles in~\cref{table:boxplot}. Overall, \texttt{SPMA-ALM} achieves the most consistent constraint satisfaction across environments, with lower mean and tail violations.
\begin{figure}[h]
    \centering
    \includegraphics[scale=0.6]{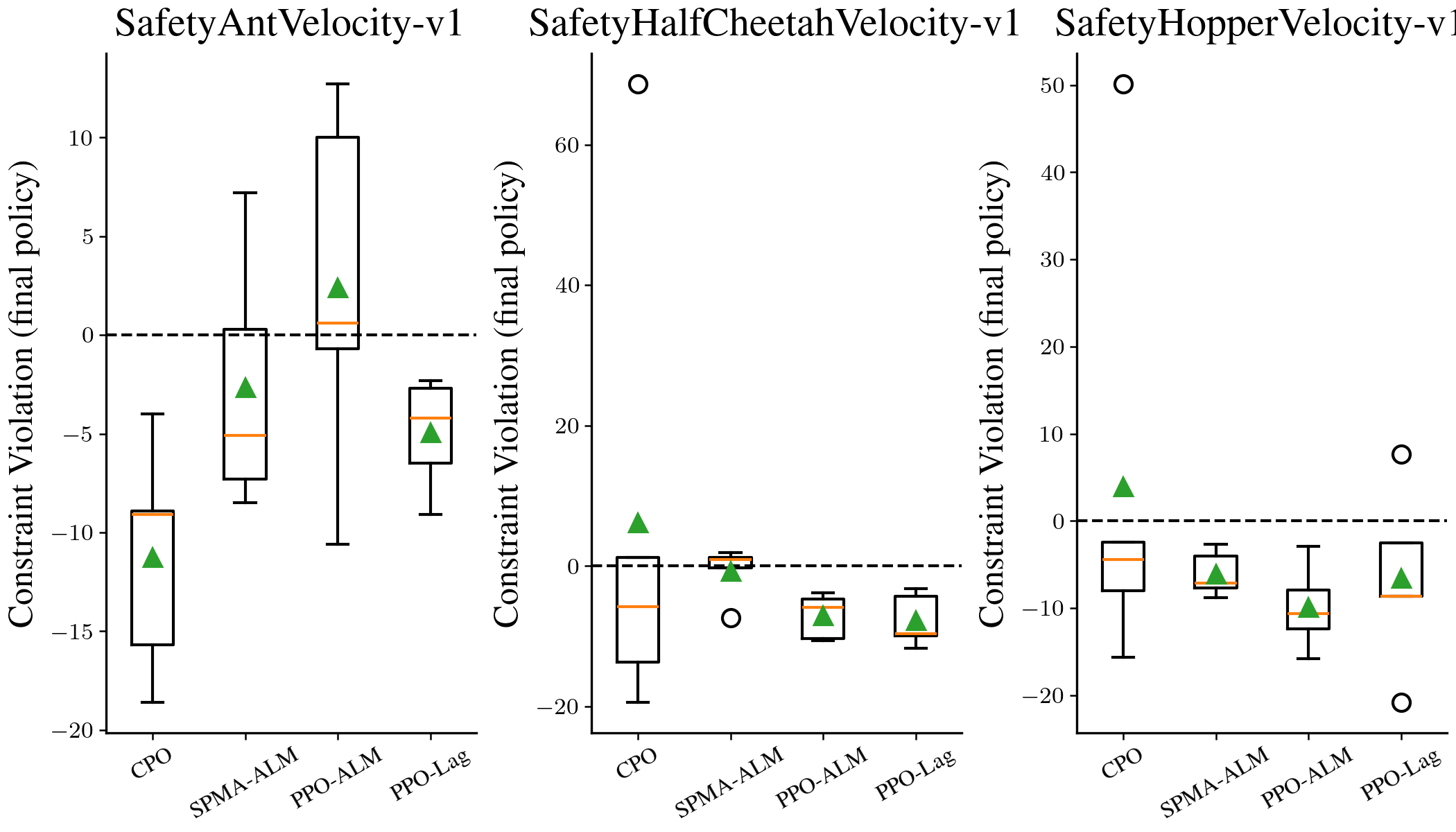}
    \caption{Box plot of constraint violations of the final policy, evaluated over 10 rollouts per seed across 10 seeds. The green triangle denotes the mean, and the orange line indicates the median.}
    \label{fig:boxplot}
\end{figure}
\begin{table}[h]
    \centering
    \caption{Constraint violation statistics over the final policy across in~\cref{fig:alm_vs_all_safety_gym}}
\resizebox{\textwidth}{!}{
    \begin{tabular}{
        l 
        c c c c 
        c c c c 
        c c c c 
    }
    \toprule
    & \multicolumn{4}{c}{\textbf{SafetyAntVelocity-v1}} 
    & \multicolumn{4}{c}{\textbf{SafetyHalfCheetahVelocity-v1}} 
    & \multicolumn{4}{c}{\textbf{SafetyHopperVelocity-v1}} \\
    \cmidrule(lr){2-5} \cmidrule(lr){6-9} \cmidrule(lr){10-13}
    \textbf{Algorithm} 
    & Mean & Worst & P90 & P95  & Mean & Worst & P90 & P95  & Mean & Worst & P90 & P95 \\
    \midrule
    \texttt{CPO}
    & -11.26 & -4.00 & -5.96 & -4.98
    & 6.18 & 68.60 & 41.64 & 55.12
    & 3.94 & 50.10 & 29.10 & 39.60 \\
    \texttt{SPMA-ALM}
    & -2.68 & 7.20 & 4.44 & 5.82
    & -0.74 & 1.90 & 1.62 & 1.76
    & -6.06 & -2.70 & -3.22 & -2.96 \\
    \texttt{PPO-ALM}
    & 2.40 & 12.70 & 11.62 & 12.16
    & -7.06 & -3.80 & -4.16 & -3.98
    & -9.92 & -2.90 & -4.90 & -3.90 \\
    \texttt{PPO-Lag}
    & -4.96 & -2.30 & -2.46 & -2.38
    & -7.74 & -3.20 & -3.64 & -3.42
    & -6.56 & 7.70 & 3.62 & 5.66 \\
    \bottomrule
    \end{tabular}
    \label{table:boxplot}
}
\end{table}

\newpage

\subsection{Effect of the Number of \texttt{AL} Iterations and Penalty $\beta$}
\begin{figure}[h]
    \centering
    \includegraphics[scale=0.35]{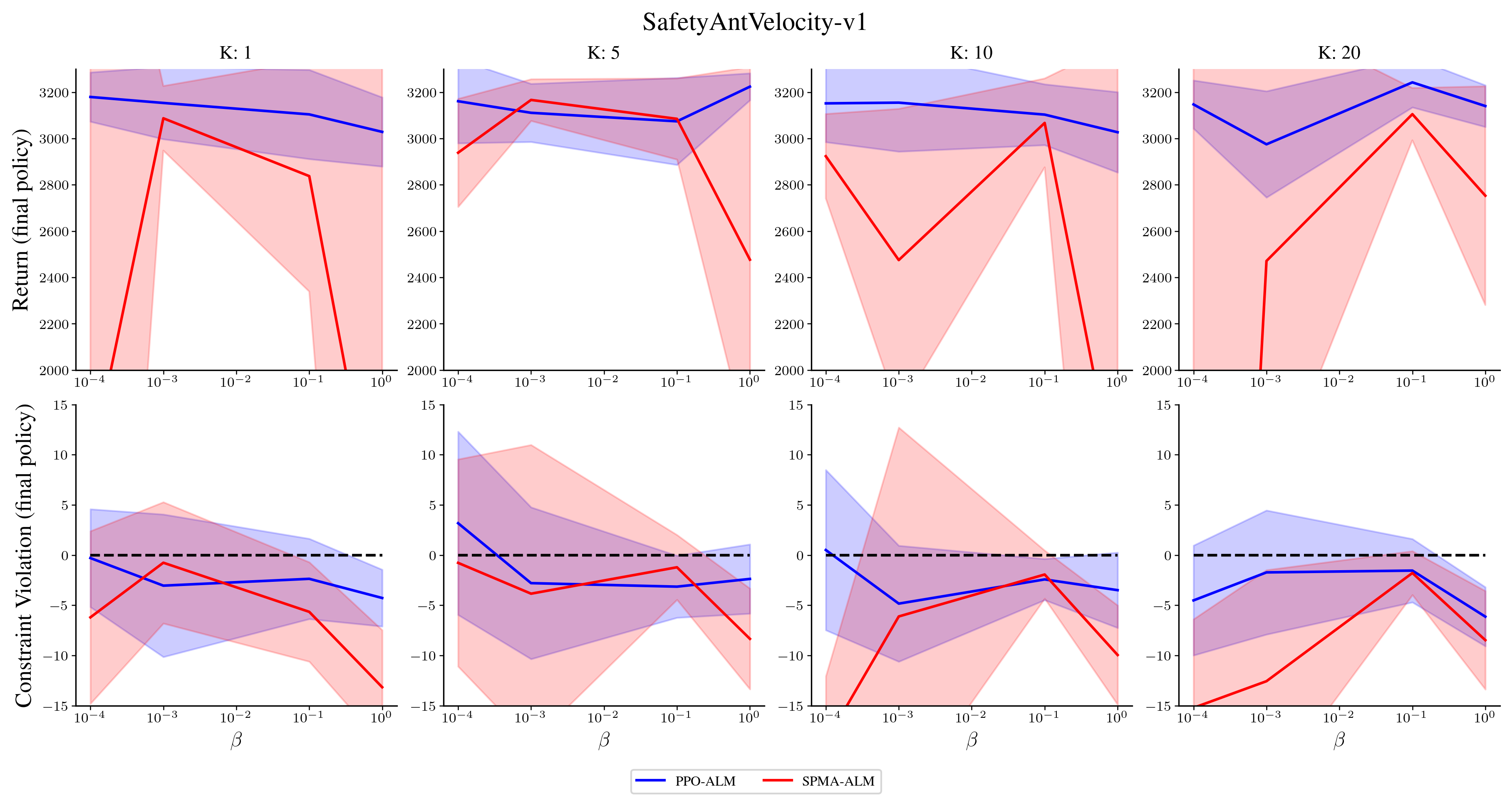}
    \includegraphics[scale=0.35]{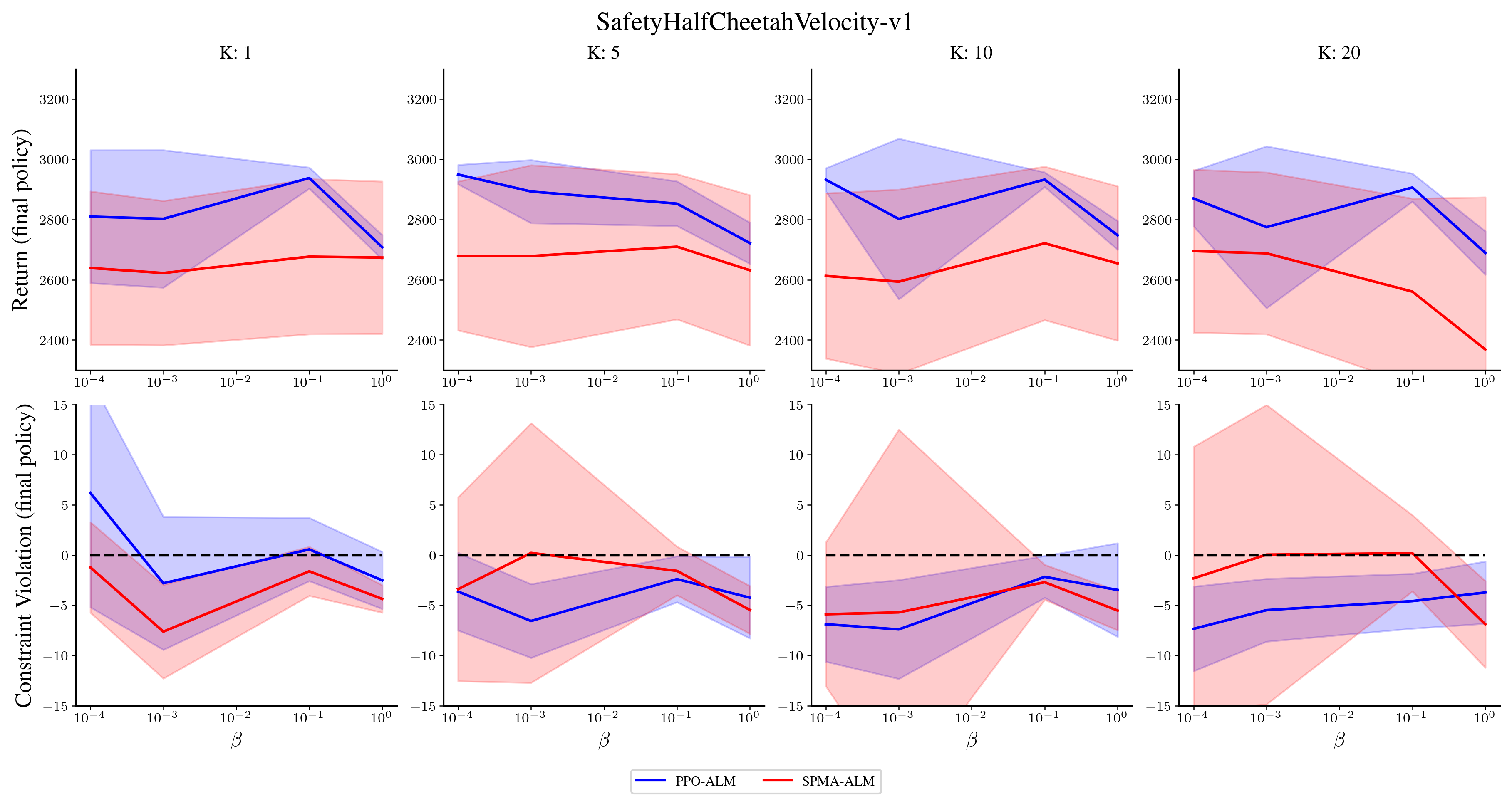}
    \includegraphics[scale=0.35]{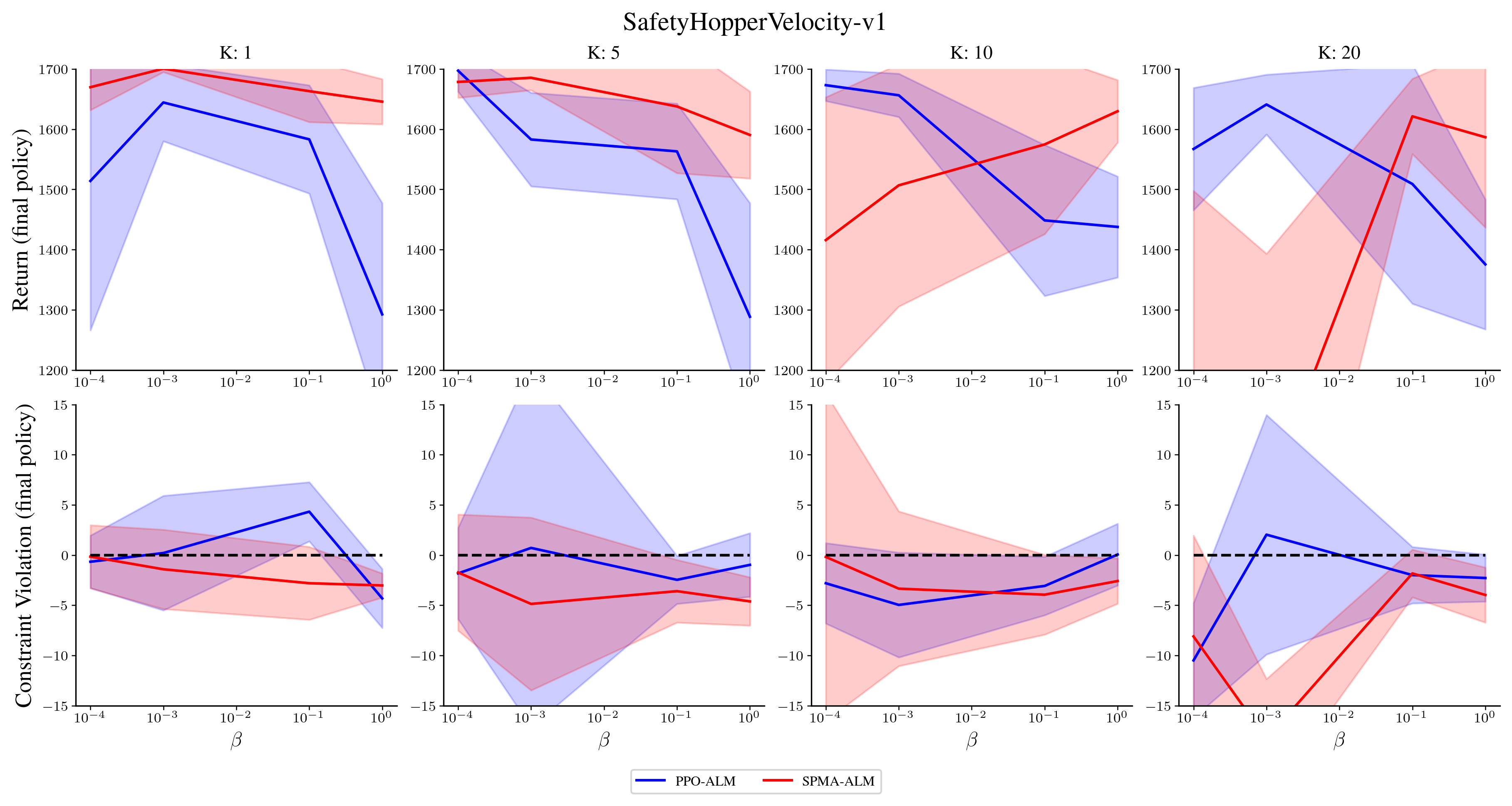}
    \caption{Ablation study on the initial penalty parameter ($\beta$) and the number of AL iterations ($K$) of the final policy. We vary $\beta \in  \{0.0001, 0.001, 0.1, 1\}$ and the number of AL iterations $K \in \{1, 5, 10, 20\}$, while keeping all other hyperparameters to their individually best-tuned values. Performance is evaluated for the final policy using $10$ rollouts per seed across $10$ seeds. The shaded regions indicate 95\% confidence intervals. We observe that performance is generally stable for smaller $\beta$ and exhibits limited sensitivity to $K$.}
\end{figure}
\newpage

\subsection{Effect of the Number of \texttt{AL} Iterations and Dual Step-size}
\begin{figure}[h]
    \centering
    \includegraphics[scale=0.33]{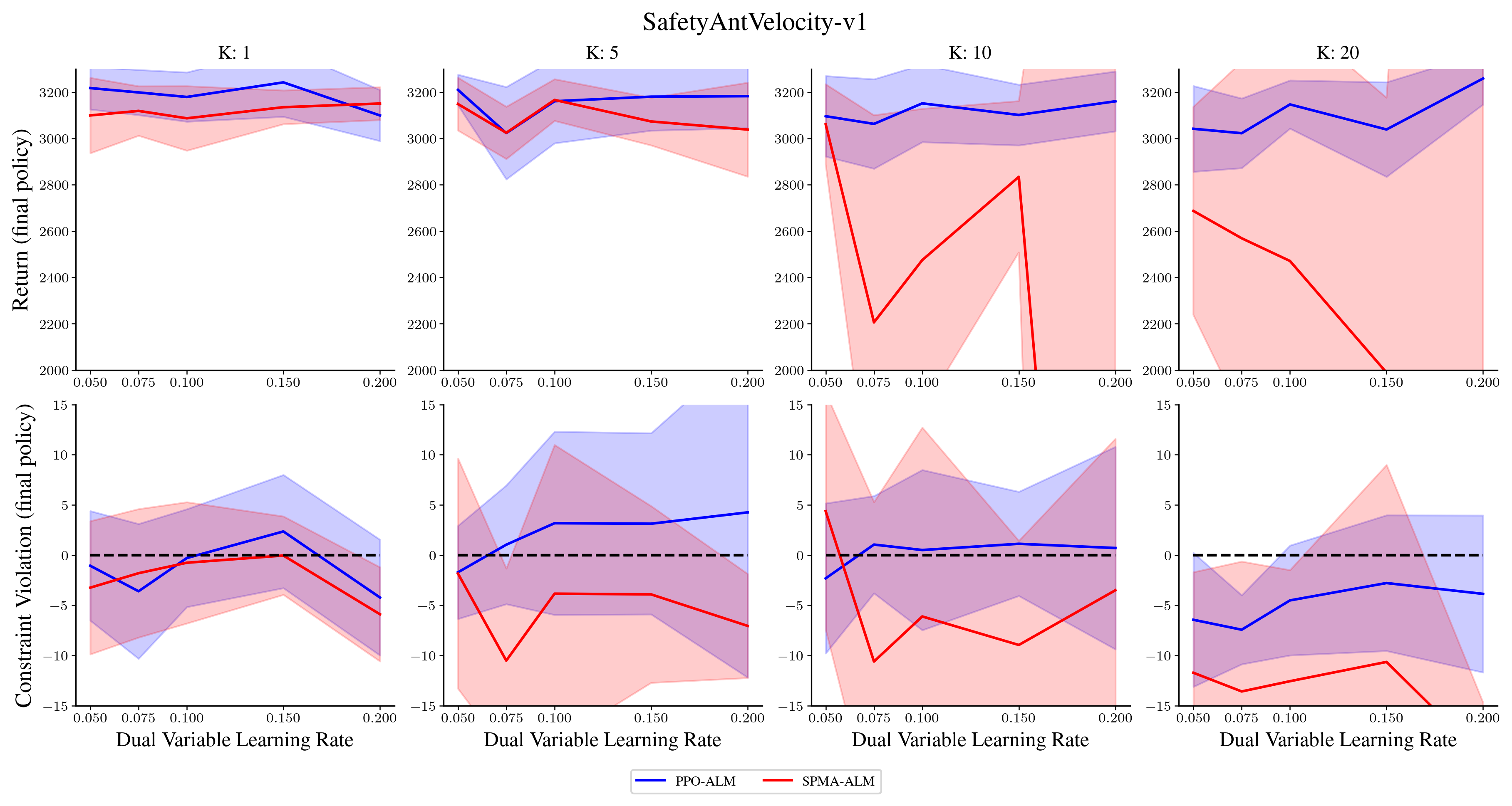}
    \includegraphics[scale=0.33]{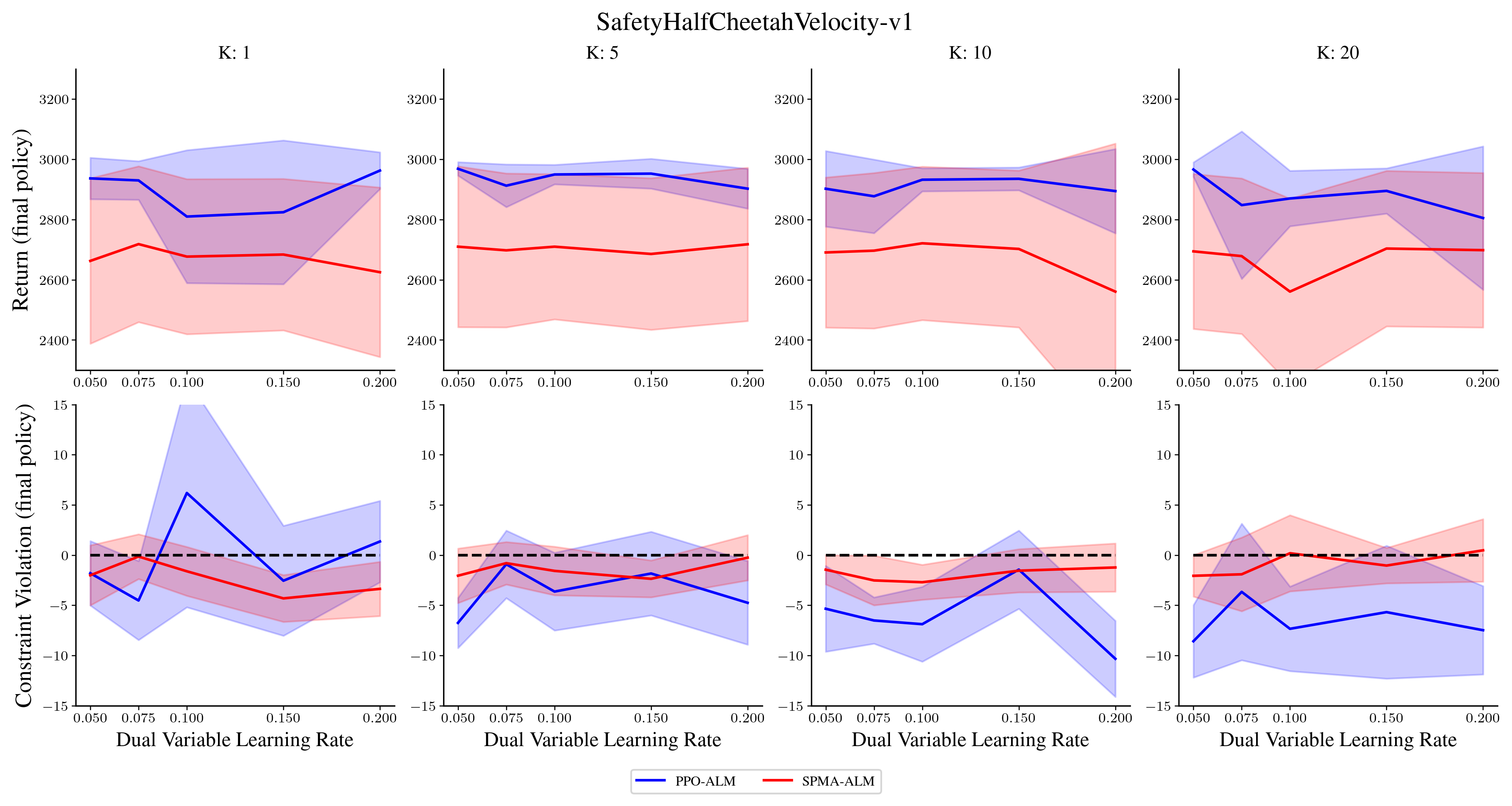}
    \includegraphics[scale=0.33]{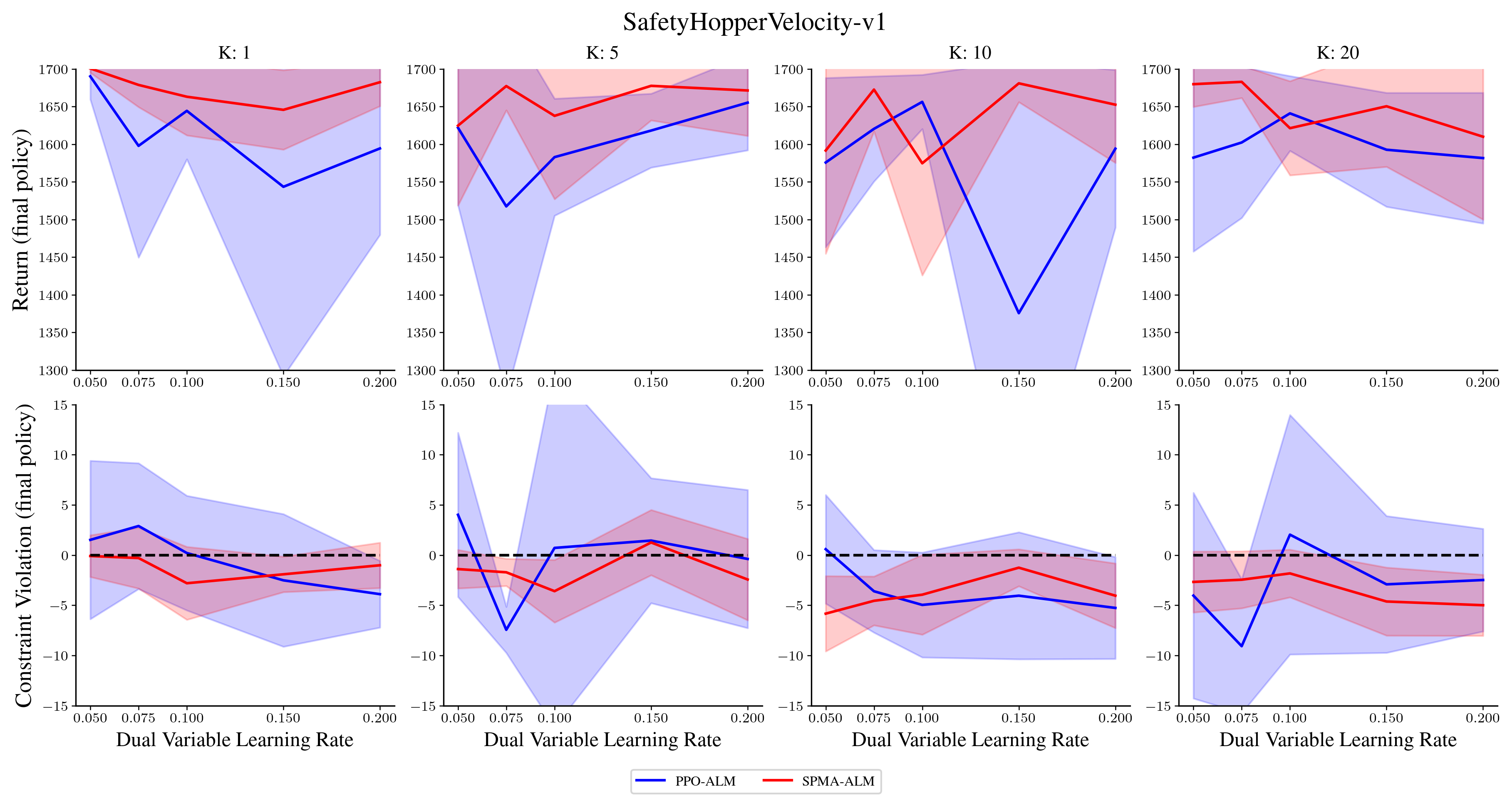}
    \caption{Ablation study on the dual step-size and the number of \texttt{AL} iterations ($K$) of the final policy. Performance is evaluated for the final policy using $10$ rollouts per seed across $10$ seeds. We vary the dual variable step-size: $\{0.05, 0.075, 0.1, 0.15, 0.2\}$ and the number of \texttt{AL} iterations $K \in \{1, 5, 10, 20\}$, while keeping all other hyperparameters fixed to their individually best-tuned values. The shaded regions indicate 95\% confidence intervals. We find that \texttt{SPMA-ALM} is more sensitive to the dual step-size  compared to \texttt{PPO-ALM}.}
\end{figure}

\newpage

\section{Tabular Experiments}\label{appendix:experiments}
\subsection{Environmental Details}
In each of the following environments, we set the initial state distribution to be uniform, i.e. for all $s \in \gS$, $\rho(s) = \frac{1}{S}$.

\textbf{Constrained Cliff World:} 
This environment modifies the original Cliff World~\citep[Example 6.6]{sutton2018reinforcement}.
The environment consists of $21$ states and $4$ actions. The objective is for an agent to each the goal state while avoiding a cliff. 
If the agent falls into the chasm, the agent receives a reward of $0$ and a cost of $-1$. 
If the agent reaches the goal, the agent receives a reward of $+1$. 
All other rewards and costs are $0$. 
In this environment $\gamma = 0.9$ and $b = -0.17$.

\begin{figure}[h]
    \centering
    \includegraphics[scale=0.4]{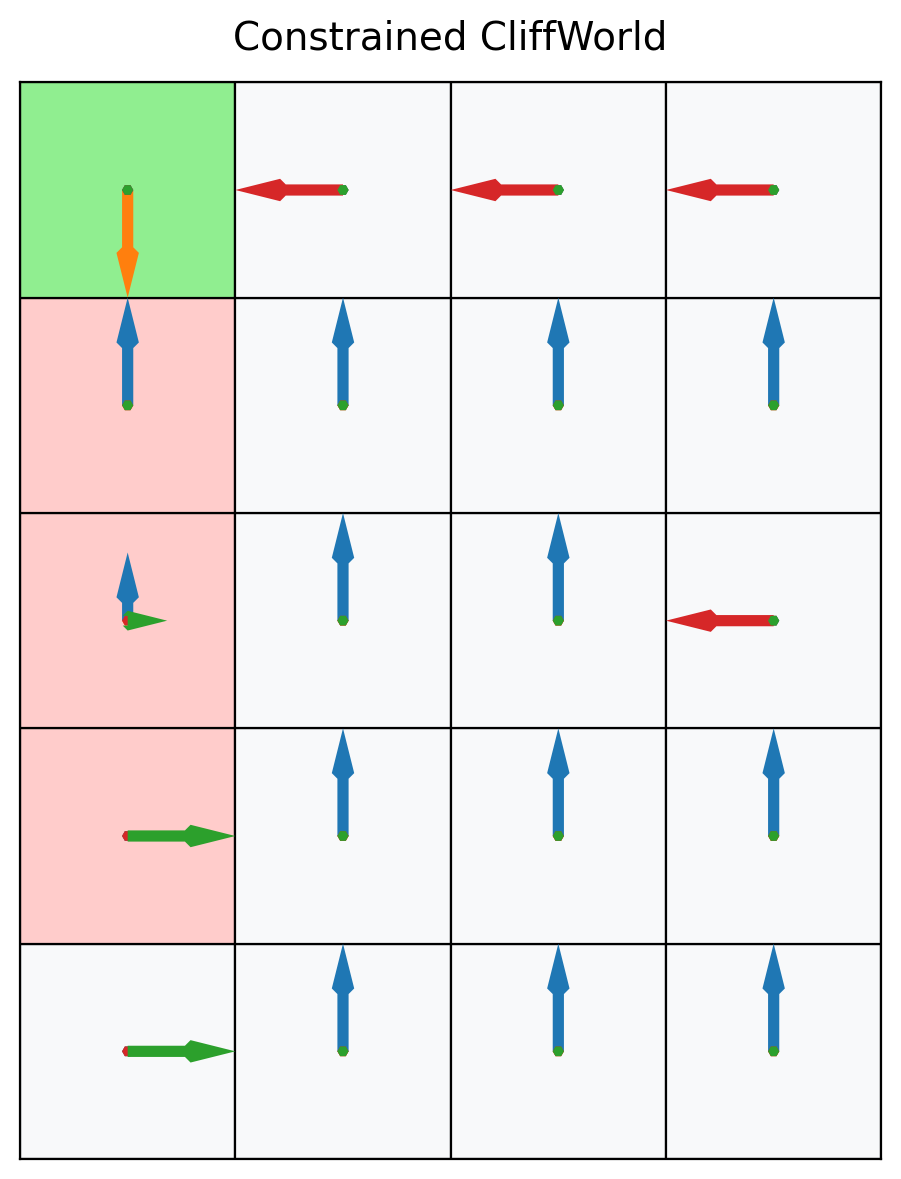}
    \caption{Visualization of the optimal policy for the cliff world environment. A longer arrows indicates a higher the probability of taking a particular action. The red regions represents the cliff and the green region indicates the goal.}
\end{figure}

\textbf{Constrained Deep Sea Treasure} 
This environment modifies the original Deep sea treasure
~\citep{osband2019behaviour}:
The environment consists $25$ states and $2$ actions. The agent begin from the top-left corner of the grid and descends one row per each time it takes an action. The goal of the agent is to stay left in order to reach the treasure. If the agent transitions to the right, it receives a reward of $-0.02$. Otherwise if the agent reaches the treasure, it receives a reward of $+1$. 
Additionally, in the environment there are two landmines. 
If the agent reaches a landmine, the agent receives a cost of $-2$. All other costs are $0$.
In this environment $\gamma = 0.9$ and $b = -0.1$.

\begin{figure}[h]
    \centering
    \includegraphics[scale=0.4]{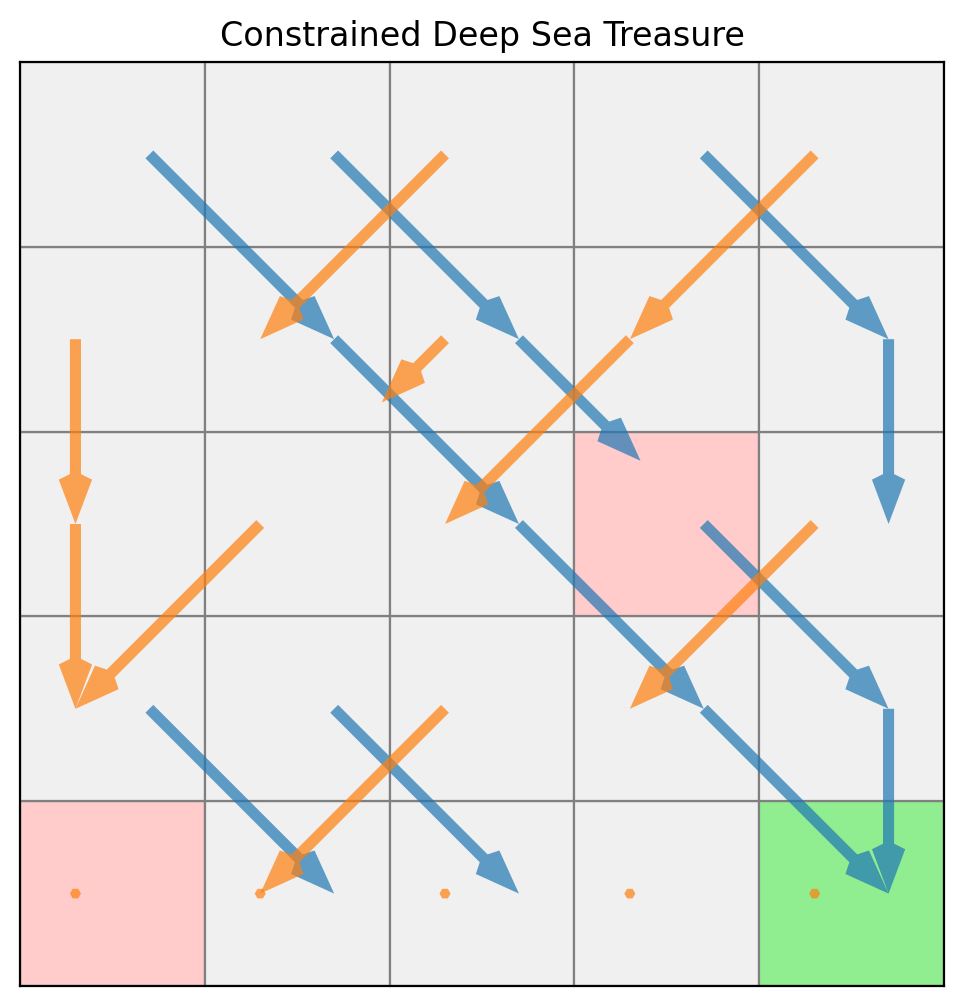}
    \caption{Visualization of the optimal policy for the deep sea treasure environment. A longer arrows indicates a higher the probability of taking a particular action. The red regions represents the landmine and the green region indicates the goal.}
\end{figure}

\newpage

\subsection{Experimental Details \& Additional Results}
For all methods, hyperparameters were selected via grid search, and the full configurations are listed in \cref{table:tabular_grid_search}.
We report results using the configuration that achieved the smallest optimality gap subject to the final policy satisfying the constraint violation $b - V^{\pi_{T+1}}_{c}(\rho) \leq \varepsilon$ where $\varepsilon \in \{0.001, 0.01\}$. Overall, \texttt{PQA-ALM} perform well in practice.

\begin{table}[h!]
\centering
\caption{Grid search ranges for the tabular experiments. $n$ denotes the number of iterates used to update the corresponding surrogate.
}
\resizebox{0.9\textwidth}{!}{
\begin{tabular}{l l l}
\toprule
\textbf{Method} & \textbf{Fixed Params.} & \textbf{Search Space} \\
\midrule
\texttt{PQA-ALM} 
& $T=10$, $K=100$, $\beta=10$ 
& Primal step-size$\{0.1,1,10\}$ \\
\midrule
\texttt{PPO-Lag}~\citep{ray2019benchmarking}
& $T=10$, $K=100$,  
& Primal step-size: $\{0.01,0.1,1,10,100\}$  \\
& & Dual step-size: $\{0.01,0.1,1,10,100\}$ \\
& & Clip: $\{0.1, 0.2, 0.3\}$ \\
\midrule
\texttt{RPG-PD}~\citep{ding2023last} 
& $T=1000$  & Entropy Coefficient $\{10^{-4},10^{-2},10^{-1}\}$ \\
& & Primal step-size: $\{0.01,0.1,1,10,100\}$  \\
& & Dual step-size: $\{0.01,0.1,1,10,100\}$ \\
\midrule
\texttt{NPG-DD}~\citep{ying2022dual} 
& $T=100$, $K=10$ 
& Entropy Coefficient: $\{0.001,0.01,0.1,1\}$  \\
& & Dual step-size: $\{0.001,0.01,0.1,1\}$ \\
\midrule
\texttt{ReLOAD}~\citep{moskovitz2023reload} 
& $T=1000$  & Primal step-size $\{0.01,0.1,1,10,100\}$ \\
& & Dual step-size $\{0.01,0.1,1,10,100\}$ \\
\bottomrule
\end{tabular}
}
\label{table:tabular_grid_search}
\end{table}

\begin{figure*}[h]
    \centering
    \includegraphics[scale=0.8]{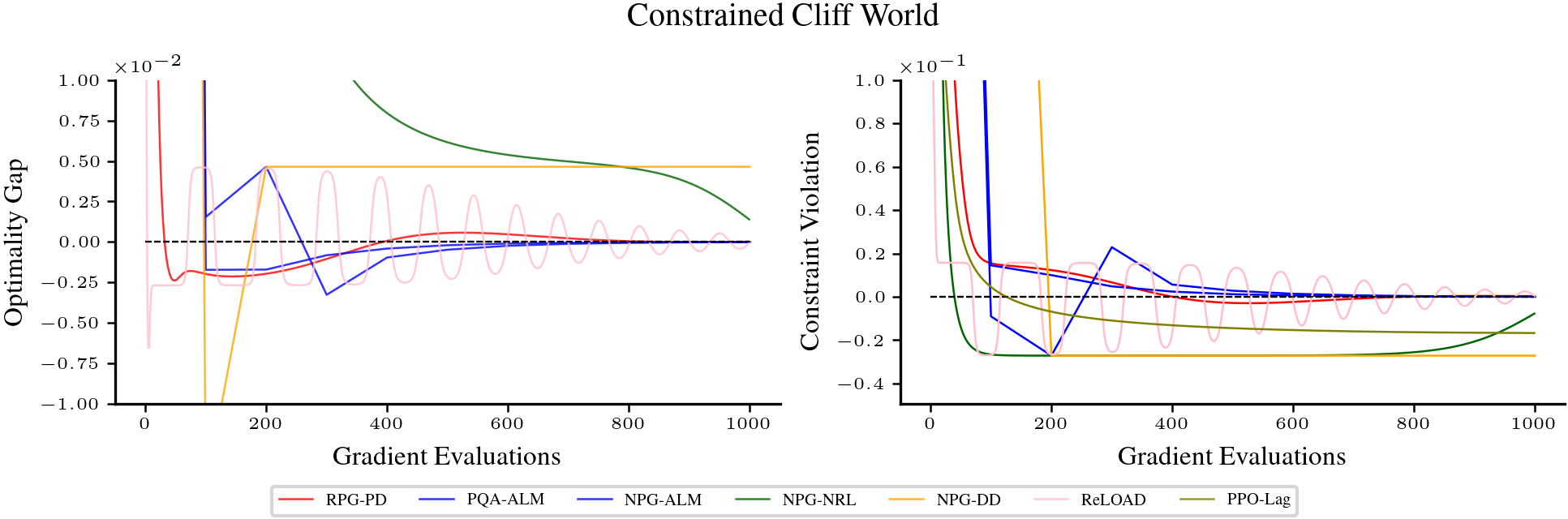}
    \includegraphics[scale=0.8]{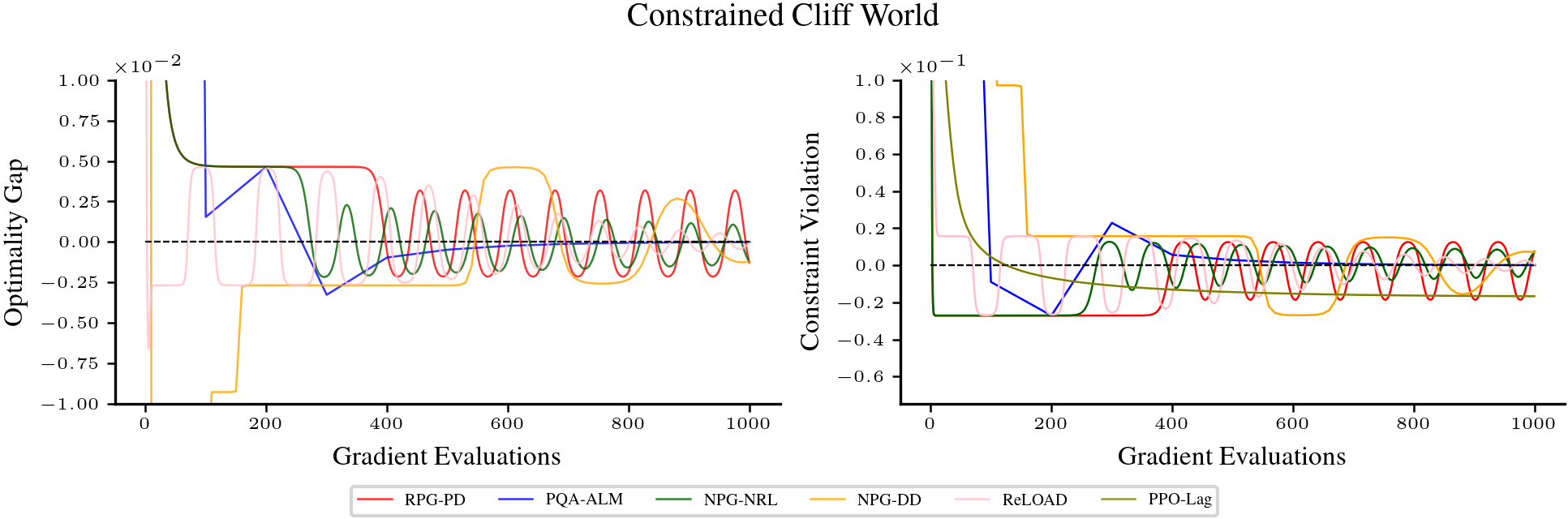}
    \caption{Constrained Cliff World with $\varepsilon = 0.001$ (top) and  $\varepsilon = 0.01$ (bottom)}
\end{figure*}

\begin{figure*}[h]
    \centering
    \includegraphics[scale=0.8]{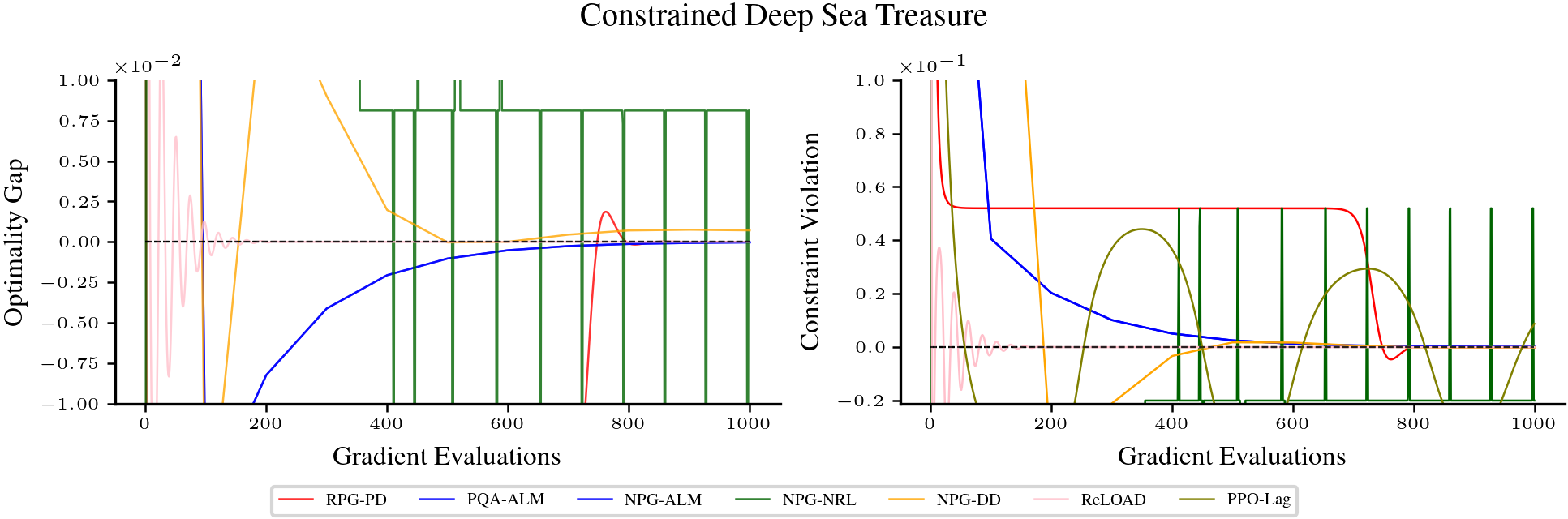}
    \includegraphics[scale=0.8]{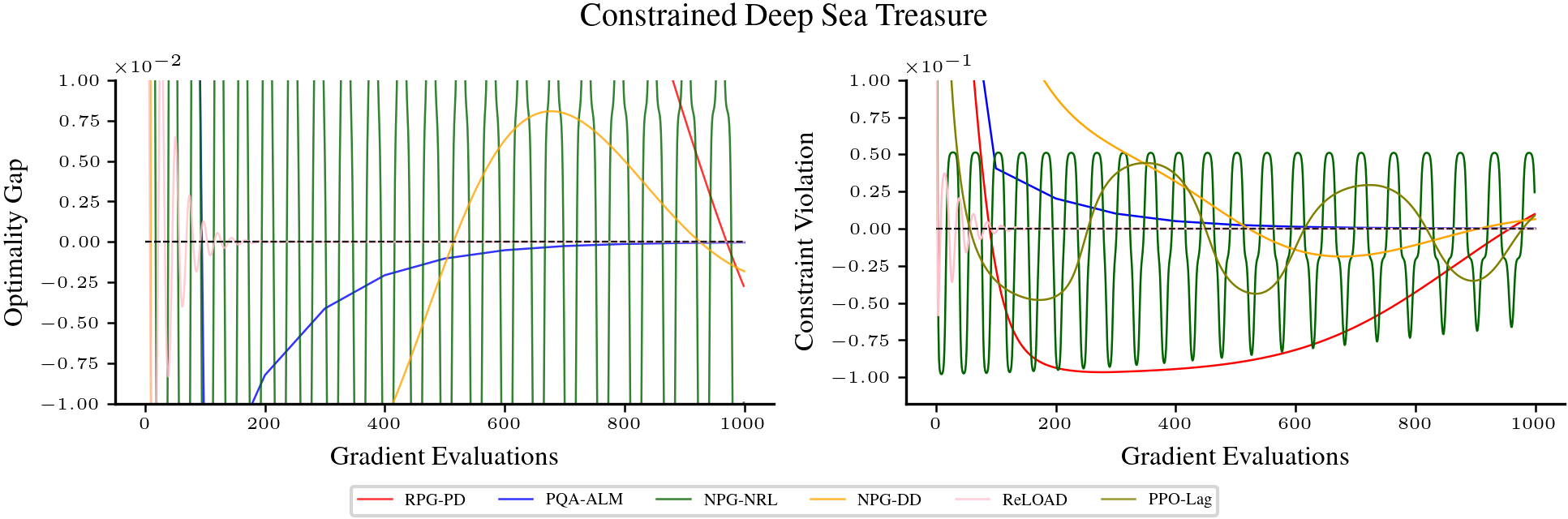}
    \caption{Constrained Deep Sea Treasure with $\varepsilon = 0.001$ (top) and  $\varepsilon = 0.01$ (bottom)}
\end{figure*}

\pagebreak

\section{Linear Function Approximation Experiments}\label{appendix:linear_fa_experiments}
To construct the features $\varphi$, following~\citet{vaswani2023decision,asad2024fast} we use tile-code features~\citep{sutton2018reinforcement}. Tilde-coded features require three parameters to be set: (i) hash table size (feature dimension $d$), (ii) number of tiles $N$, and (iii) size of tiles $s$. Increasing $N$ increases the express ability of the features Increasing $s$ controls how ``spread out'' the features are.
For both environments, we consider the following feature set: $\{(40, 4, 1), (60, 4, 3), (80, 6, 5) \}$

To extend \texttt{ReLOAD} to the linear function approximation setting, we follow the policy-based version of \texttt{ReLOAD}~\citep[Algorithm 6]{moskovitz2023reload}, using estimated $Q$-functions under the compatible function approximation theorem~\citep{agarwal2021theory}.
Specifically for the \texttt{Lin-NPG-PD} and \texttt{ReLOAD}, to compute the policy gradient estimate,  we solve the following regression problem 
\begin{equation*}
\min_{w \in \R^d} \E_{(s, a) \sim \saom}[(\dpd{\nabla \log \pitheta(a | s), w} - Q_{\diamond}^{\pi}(s, a))^2
\end{equation*}
with an appropriate pseudo-reward $\diamond$. A complete list of grid-search range is provided in~\cref{table:linear_FA_grid_search}.
\begin{table}[H]
\centering
\caption{Grid search ranges for each method. $n$ denotes the number of iterates used to update the corresponding surrogate.}
\resizebox{0.95\textwidth}{!}{
\begin{tabular}{l l l}
\toprule
\textbf{Method} & \textbf{Fixed Params.} & \textbf{Search Space} \\
\midrule
\texttt{PPQA-ALM} 
& $T=10$, $K=100$,  
& Primal step-size$\{0.0001, 0.001, 0.1, 1.0, 10\}$ \\
& $n=250$, $\beta=10$ & Surrogate step-size$\{0.0001, 0.001, 0.1, 1.0, 10\}$ \\
\midrule
\texttt{PPO-Lag}~\citep{ray2019benchmarking}
& $T=1000$, $n=250$,  
& Primal step-size: $\{0.01,0.1,1,10,100\}$  \\
& & Dual step-size: $\{0.01,0.1,1,10,100\}$ \\
& & Surrogate step-size$\{0.0001, 0.001, 0.1, 1.0, 10\}$ \\
& & Clip: $\{0.1, 0.2, 0.3\}$ \\
\midrule
\texttt{Lin-NPG-PD}~\citep{ding2023last}
& $T=1000$,  & Primal step-size $\{0.0001, 0.001, 0.1,1,10\}$; \\
& & Dual step-size $\{0.0001, 0.001, 0.1,1,10\}$ \\
& & Entropy Coefficient $\{0.001,0.01,0.1,1\}$ \\
\midrule
\texttt{ReLOAD}~\citep{moskovitz2023reload} 
& $T=1000$  & Primal step-size $\{0.0001, 0.01,0.1,1\}$ \\
& & Dual step-size $\{0.01,0.1,1,10,100\}$ \\
\bottomrule
\end{tabular}
}
\label{table:linear_FA_grid_search}
\end{table}
\newpage
\begin{figure*}[h]
    \centering
    \includegraphics[scale=0.8]{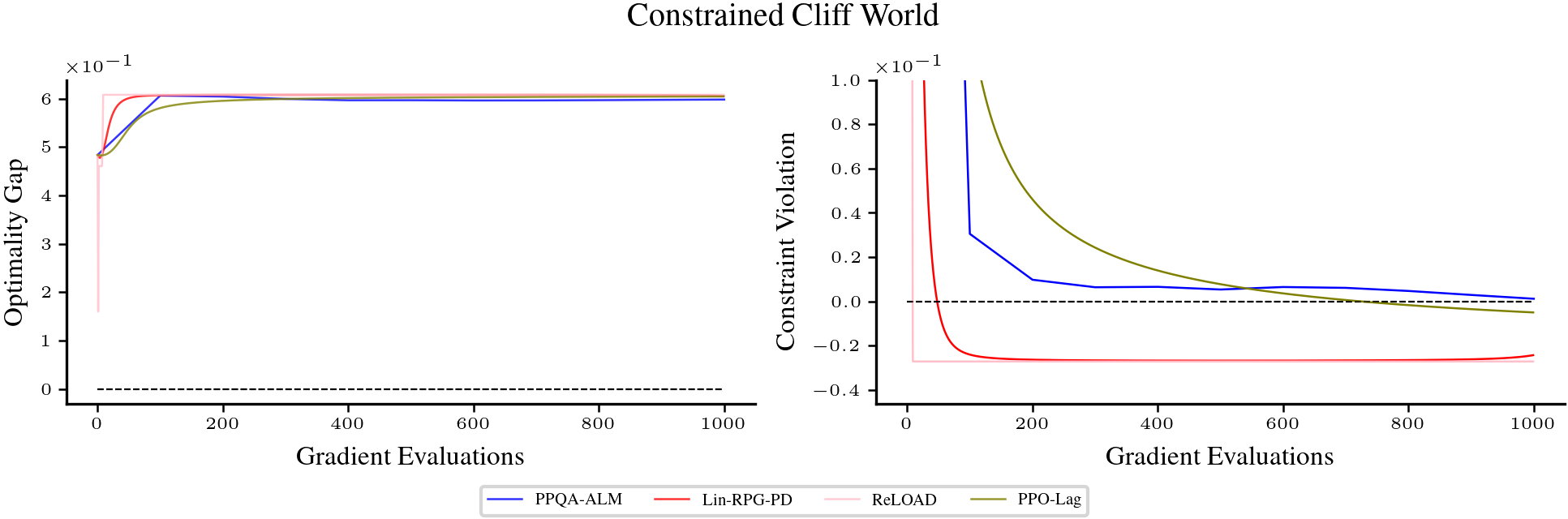}
    \includegraphics[scale=0.8]{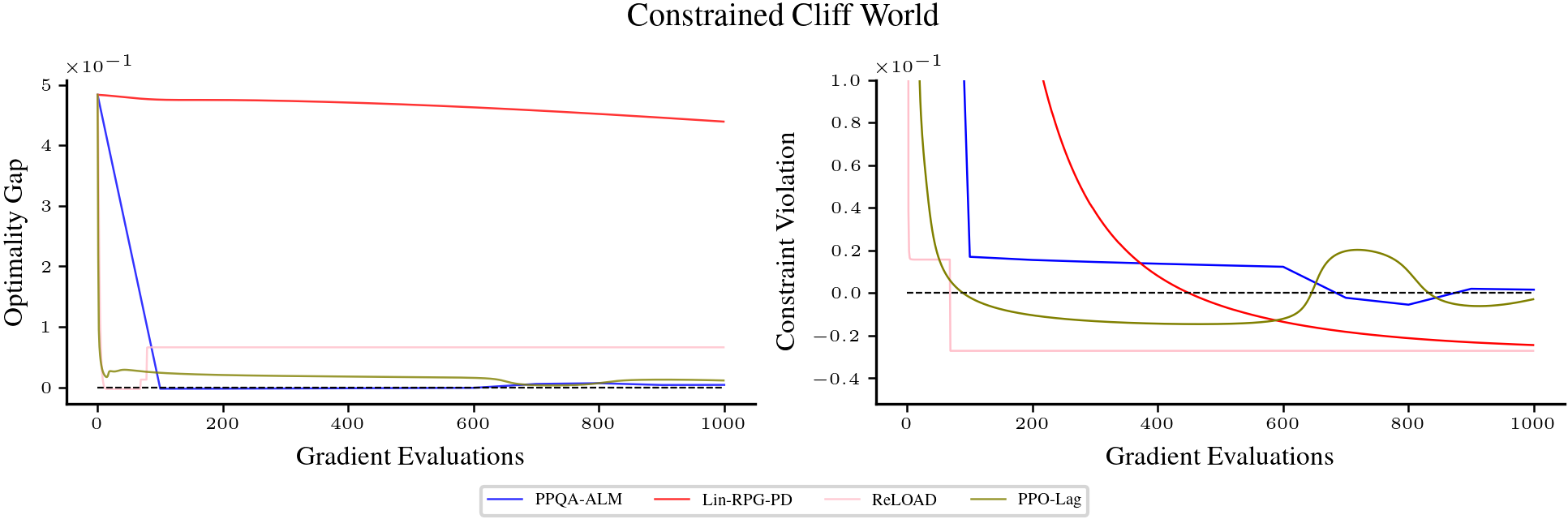}
    
    \includegraphics[scale=0.8]{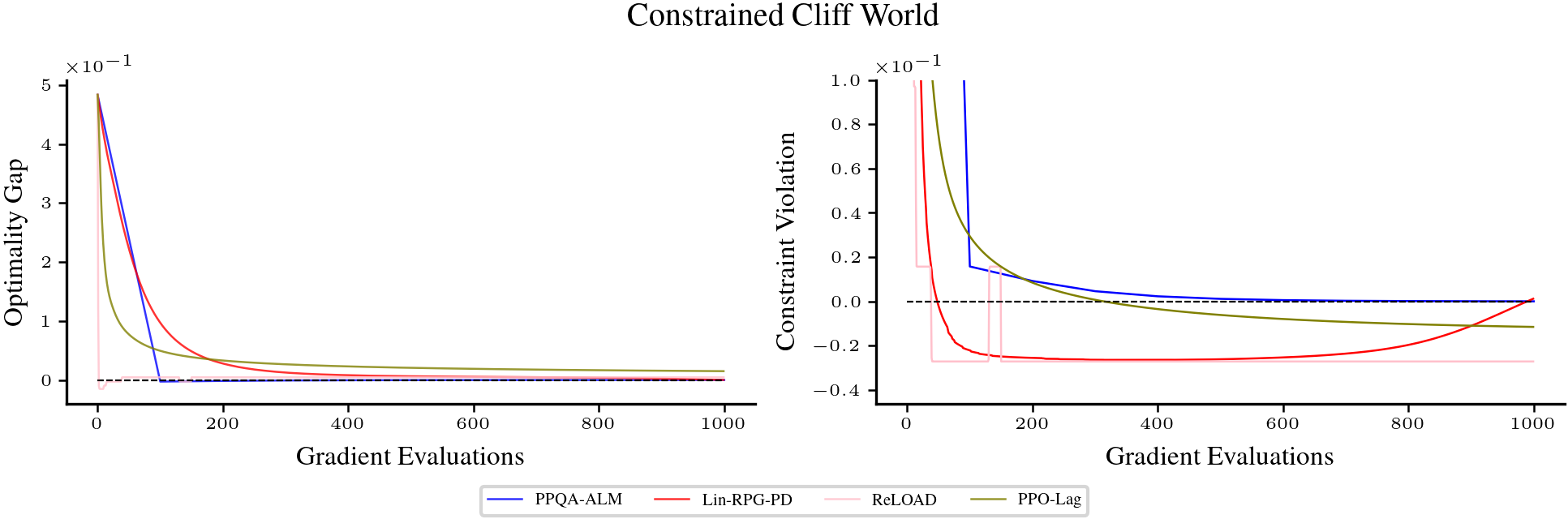}
    
    \caption{Constrained Cliff world with $\varepsilon = 0.001$ and $d=40$ (top), $d=60$ (middle), $d=80$ (bottom) features}
\end{figure*}
\newpage
\begin{figure*}[h]
    \centering
    \includegraphics[scale=0.8]{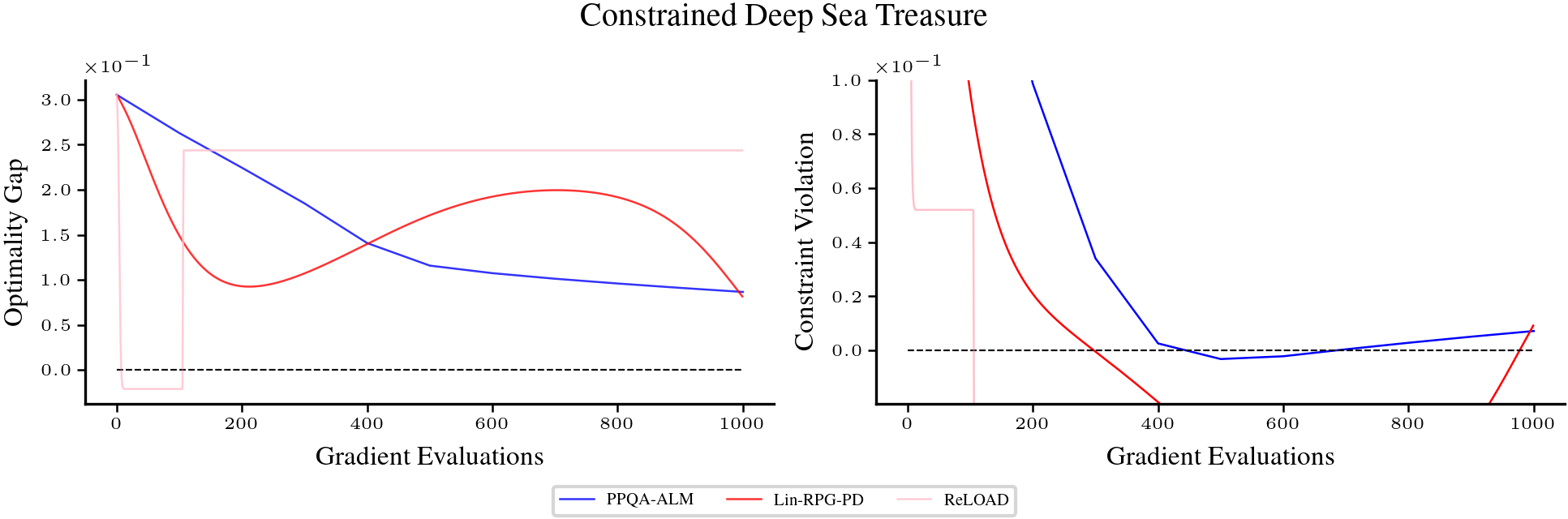}
    \includegraphics[scale=0.8]{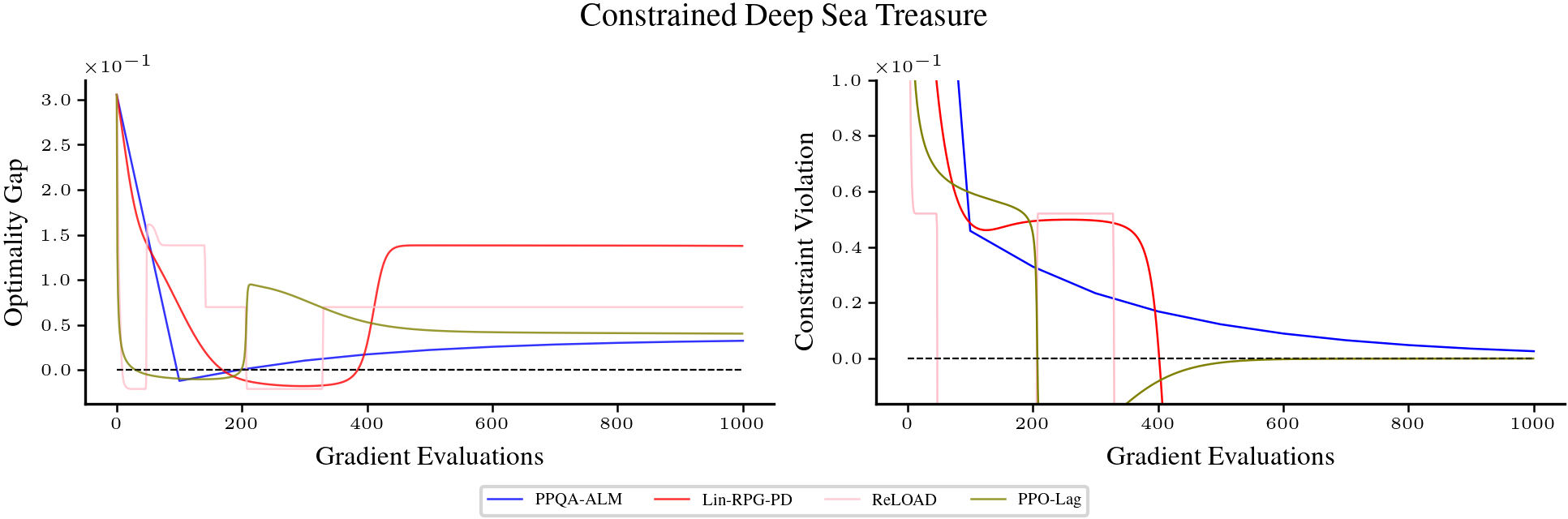}
    
    \includegraphics[scale=0.8]{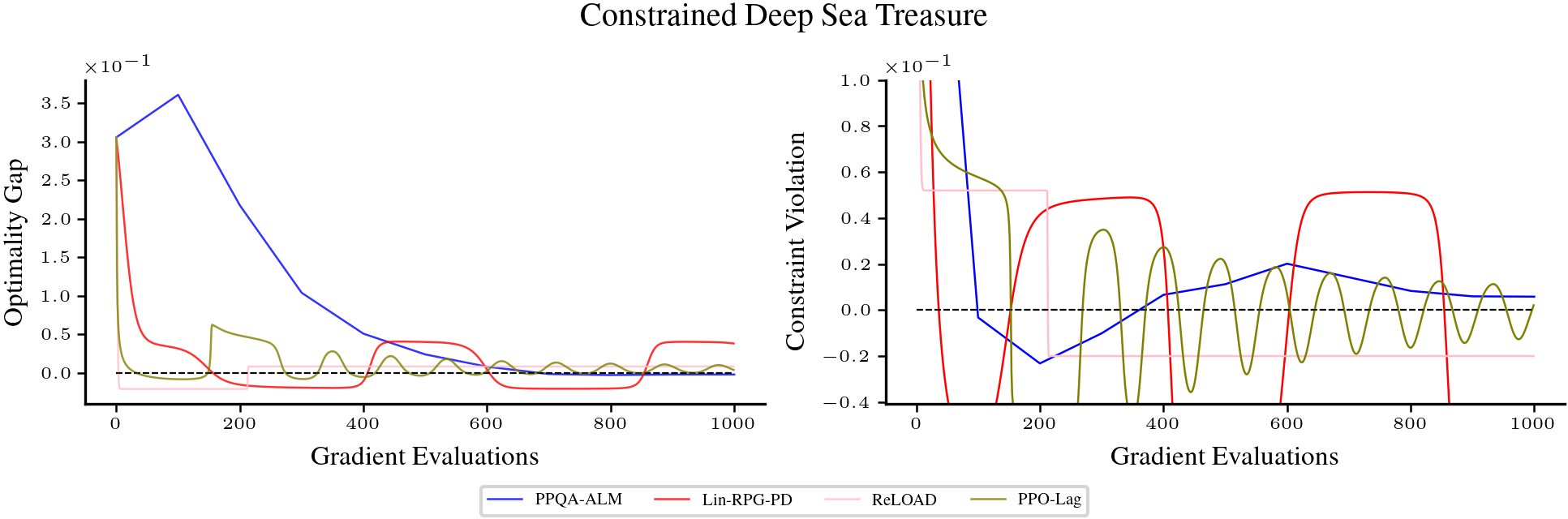}
    
    \caption{Constrained Deep Sea Treasure with $\varepsilon = 0.001$ and $d=40$ (top), $d=60$ (middle), $d=80$ (bottom) features}
\end{figure*}

\newpage

%% file: Appendix/A5_restarted_theorem_lemma.tex
\section{External Results}\label{appendix:extra_results}
For completeness, we append external Theorems and lemmas here.
\begin{restatable}[Theorem~1 in \citet{kumar2024policy}]{theorem}{GenUtilityPG}
\label{theorem:gen_utility_pg}
For any given general utility $G : \Theta \to \mathbb{R}$, the policy gradient of $G$ is
\begin{equation*}
    \nabla_\theta G(\pi) = \frac{1}{1- \gamma} \sum_{s \in \cS} d^\pi(s) \, \sum_{a \in \cA} Q^{\pi}_{\Gamma(\pi)}(s, a) \frac{d \pi(a \mid s)}{d \theta}
\end{equation*} 
where $\Gamma(\pi) \coloneq \left.\frac{d \cF(x)}{d x}\right|_{x=\saom} \in \R^{SA}$ is the general utility pseudo-reward function.
\end{restatable}

\begin{lemma}[Lemma 4.1 in~\citet{jain2022towards}]\label{lemma:dual_variable_ub}
The objective in~\cref{problem:cmdp_objective_pi} satisfies strong duality, and the optimal dual variable are bounded as $\lambda^*_i \leq \frac{1}{\zeta_i \, (1 - \gamma)}$ where $\zeta_i = \max_{\pi} V^{\pi}_{c_i}(\rho) - b_i > 0$ for each constraint $i \in [m]$.
\end{lemma}